\documentclass[11pt]{article}
\usepackage[margin=1.09in]{geometry}
\usepackage[utf8]{inputenc}
\usepackage[utf8]{inputenc} 
\usepackage[T1]{fontenc}    
\usepackage{url}            
\usepackage{booktabs}       
\usepackage{amsfonts}       
\usepackage{nicefrac}       
\usepackage{xcolor}
\usepackage[inline]{enumitem}
\usepackage{appendix}

\usepackage{microtype}
\usepackage{graphicx}
\usepackage{subfigure}

\usepackage{amsmath}

\usepackage{amsthm}
\usepackage{mathtools}
\usepackage{amssymb}
\usepackage{amsopn}
\usepackage{bbm}
\usepackage{color}
\usepackage{hyperref}
\usepackage{url}
\usepackage{tikz}
\usepackage{tabu}
\usepackage{placeins}
\usepackage{eso-pic}
\usepackage{caption}
\usepackage{eucal}
\usepackage{pifont}
\usepackage{mathtools}
\PassOptionsToPackage{mathcal}{euscript}

\newcommand{\abs}[1]{\left\vert#1\right\vert}
\newcommand{\norm}[1]{\left\lVert#1\right\rVert}
\newcommand{\tensornorm}[1]{{\left\vert\kern-0.25ex\left\vert\kern-0.25ex\left\vert #1 
    \right\vert\kern-0.25ex\right\vert\kern-0.25ex\right\vert}}
\newcommand{\floor}[1]{\left\lfloor#1\right\rfloor}
\newcommand{\ceil}[1]{\left\lceil#1\right\rceil}
\newcommand{\overbar}[1]{\mkern 1mu \overline{\mkern-1.5mu#1\mkern0mu}\mkern 0mu}
\newcommand{\expect}{\E\expectarg}
\DeclarePairedDelimiterX{\expectarg}[1]{[}{]}{%
  \ifnum\currentgrouptype=16 \else\begingroup\fi
  \activatebar#1
  \ifnum\currentgrouptype=16 \else\endgroup\fi
}
\newcommand{\innermid}{\nonscript\;\delimsize\vert\nonscript\;}
\newcommand{\activatebar}{%
  \begingroup\lccode`\~=`\|
  \lowercase{\endgroup\let~}\innermid 
  \mathcode`|=\string"8000
}
\newcommand*\tcircle[1]{%
  \raisebox{0.5pt}{%
    \textcircled{\fontsize{8.5pt}{0}\fontfamily{and}\selectfont #1}%
  }%
}

\newcommand{\E}{\mathbb{E}}

\newcommand{\N}{\mathbb{N}}
\newcommand{\Prob}{\mathbb{P}}
\newcommand{\R}{\mathbb{R}}

\newcommand{\CC}{\mathcal{C}}
\newcommand{\DD}{\mathcal{D}}
\newcommand{\EE}{\mathcal{E}}

\newcommand{\GG}{\mathcal{G}}
\newcommand{\HH}{\mathcal{H}}

\newcommand{\NN}{\mathcal{N}}
\newcommand{\MM}{\mathcal{M}}
\newcommand{\OO}{\mathcal{O}}

\newcommand{\dd}{\mathrm{d}}
\newcommand{\del}{\partial}
\newcommand{\diag}{\mathrm{diag}}

\newcommand{\tr}{\mathrm{tr}}
\newcommand{\one}{\mathbbm{1}}

\newcommand{\lr}{\eta}

\newcommand{\relu}{\mathrm{ReLU}}
\newcommand{\sign}{\mathrm{sign}}
\newcommand{\ftrue}{f_{\mathrm{true}}}

\newcommand{\iid}{\overset{{\tiny \text{i.i.d}}}{\sim}}
\newcommand{\tR}{t\wedge\theta_R}
\newcommand{\tauR}{\tau\wedge\theta_R}

\newcommand{\hbarr}{\mkern 1mu \overline{\mkern-1.5mu H \mkern0mu}\mkern -1mu}
\newcommand{\htilde}{\widetilde{H}}
\newcommand{\Xtilde}{\widetilde{X}}
\newcommand{\hhat}{\widehat{h}}

\theoremstyle{plain}
\newtheorem{thm}{Theorem}
\newtheorem{lem}{Lemma}

\theoremstyle{remark}
\newtheorem{defn}{Definition}
\newtheorem{ass}{Assumption}
\newtheorem{hyp}{Scaling regime}
\newtheorem{rmk}{Remark}

\newcommand{\alain}[1]{\color{blue} \textbf{Alain:} #1 \color{black}}

\title{Asymptotic Analysis of Deep Residual Networks\footnote{Alain Rossier's research was supported through EPSRC Centre for Doctoral Training in Mathematics of Random Systems: Analysis, Modelling and Simulation (EP/S023925/1).}}
\author{Rama~Cont$^1$, Alain Rossier$^{1,2}$ and Renyuan Xu$^3$\\
\ \\
   \small{ $^1$ Mathematical Institute, University of Oxford \hspace{0.2cm} $^2$ Instadeep Ltd}\\
   \small{ $^3$ Department of Industrial and Systems Engineering, University of Southern California} \\
   \small{Email addresses: Rama.Cont@maths.ox.ac.uk, rossier@maths.ox.ac.uk, renyuanx@usc.edu}
}

\begin{document}
\numberwithin{equation}{section}

\maketitle

\begin{abstract}
We investigate the asymptotic properties of deep Residual networks (ResNets) as the number of layers increases. We first show the existence of scaling regimes for trained weights markedly different from those implicitly assumed in the neural ODE literature. We study the convergence of the hidden state dynamics in these scaling regimes, showing that one may obtain an ODE, a stochastic differential equation (SDE) or neither of these. In particular, our findings point to the existence of a diffusive regime in which the deep network limit is described by a class of stochastic differential equations (SDEs). Finally, we derive the corresponding scaling limits for the backpropagation dynamics. 
\end{abstract}
\newpage
\tableofcontents
\newpage

\section{Introduction} \label{sec:intro}


Residual networks, or ResNets, are multilayer neural network architectures in which a {\it skip connection} is introduced at every layer~(\cite{HZRS2016}). This allows very deep networks to be trained by circumventing vanishing and exploding gradients, mentioned in~\cite{BSF1994}. The increased depth in ResNets has lead to commensurate performance gains in applications ranging from speech recognition~\cite{HDH2016, ZK2016} to computer vision~\cite{HZRS2016, HSLSW2016}.

A residual network with $L$ layers may be represented as
\begin{equation}\label{forward-map-resnet}
    h^{(L)}_{k+1} = h^{(L)}_{k} + \delta^{(L)}_k \sigma_d\left(A^{(L)}_{k}h^{(L)}_{k} + b^{(L)}_k \right),
\end{equation} 
where $h^{(L)}_k$ is the hidden state at layer $k=0,\ldots,L$, $h^{(L)}_0 = x \in \R^d$ the input, $h^{(L)}_L \in \R^d$ the output, $\sigma\colon\mathbb{R}\to\mathbb{R}$ a nonlinear activation function, $\sigma_d(x) = (\sigma(x_1),\ldots,\sigma(x_d))^{\top}$ its component-wise extension to $x\in \mathbb{R}^d$, and $A_k^{(L)}$, $b_k^{(L)}$, and $\delta_k^{(L)}$ trainable network weights for $k=0, \ldots, L-1$.


ResNets have been the focus of several theoretical studies due to a perceived link with a class of differential equations. The idea, put forth in~\cite{HR2018} and~\cite{CRBD2018}, is to view \eqref{forward-map-resnet} as a discretization of a system of ordinary differential equations 
\begin{equation}\label{eq:neural_ode_limit}
    \frac{\dd H_t}{\dd t}= \sigma_d\left(\overbar{A}_t H_t + \overline{b}_t \right),
\end{equation}
where $\overbar{A}\colon[0,1]\to\R^{d\times d}$ and $\overline{b}\colon[0,1]\to\R^d$ are appropriate smooth functions and $H(0)=x$. 
This may be justified~(\cite{TvG2018}) by assuming that 
\begin{equation}
    \delta^{(L)}\sim 1/L,\quad A_k^{(L)}\sim \overbar{A}_{k/L},\quad b_k^{(L)}\sim \overline{b}_{k/L} \label{eq.ThorpeScaling}
\end{equation}
as $L$ increases. Such models, named neural ordinary differential equations or neural ODEs~\cite{CRBD2018,dupont2019augmented}, have motivated the use of optimal control methods to train ResNets~\cite{ew2018}.


However, the precise link between deep ResNets and the neural ODE ~\eqref{eq:neural_ode_limit} is unclear: in practice, the weights $A^{(L)}$ and $b^{(L)}$ result from training, yet the validity of the scaling assumptions~\eqref{eq.ThorpeScaling} for trained weights  is far from obvious.
As a matter of fact, there is empirical evidence showing that using a scaling factor $\delta^{(L)}\sim 1/L$ can deteriorate the network accuracy~\cite{BMMCM2020}. Also, there is no guarantee that weights obtained through training have a non-zero limit which depends smoothly on the layer, as~\eqref{eq.ThorpeScaling} would require. In fact,  for ResNet architectures used in practice, empirical evidence points to the contrary \cite{cohen2021}. These observations motivate an in-depth examination of the actual scaling behavior of weights with network depth in ResNets and of its impact on the asymptotic behavior of those networks. 

{\bf Contributions.}
We systematically investigate the scaling behavior of trained networks weights and examine in detail the consequence of this scaling for the asymptotic properties of ResNets as the number of layers increases. 
We first show, through detailed numerical experiments,  
the existence of scaling regimes for trained weights markedly different from those implicitly assumed in the neural ODE literature. We study the convergence of the hidden state dynamics in these scaling regimes, showing that one may obtain an ODE, a stochastic differential equation (SDE) or neither of these. More precisely, we show strong convergence of the hidden state dynamics to a limiting ODE or SDE, by viewing the discrete hidden state dynamics as a ``nonlinear Euler scheme'' of the limiting equation.
At a mathematical level, we extend the convergence analysis of Higham et al. \cite{higham2002strong} for discretization schemes of time-homogeneous (Markov) diffusions to a class of nonlinear approximations  for It\^o processes with bounded coefficients. 

In particular, our findings  point to the existence of a ``diffusive regime'' in which the deep network limit is described by a class of stochastic differential equations (SDEs). These novel findings on the relation between ResNets and neural ODEs complement previous work~\cite{TvG2018,DBLP:journals/corr/abs-1904-05263,DBLP:journals/corr/abs-1910-02934,ott2021resnet, peyre2022}. 
Finally, we derive the corresponding scaling limit for the backpropagation dynamics. 
The results we obtain are different from previous ones on asymptotics of ResNets ~\cite{CRBD2018,HR2018, lu2020mean}, and correspond to a different scaling regime which is relevant for trained weights in practical settings.

In particular, in the diffusive regime we find a  limit different from the ``Neural SDE'' literature \cite{li2020a}. Indeed, we observe that the Jacobian of the output with respect to the hidden states depends on hidden states across all levels, so may {\it not} be directly expressed as the solution of a forward or backward stochastic differential equation, as proposed in \cite{li2020a}. 
However, in Section \ref{sec:backward} we obtain a representation for the asymptotics of the backpropagation dynamics in terms of an auxiliary forward SDE.

{\bf Outline.}
Section~\ref{sec:methodology} describes the various scaling regimes for trained weights evidenced in \cite{cohen2021} and the methodology for studying this scaling behaviour in the deep network limit.
In Section \ref{sec:experiments}, we report detailed numerical experiments on the scaling of trained network weights across a range of ResNet architectures and datasets, showing the existence of at least three different scaling regimes, none of which correspond to~\eqref{eq.ThorpeScaling}. In Section~\ref{sec:results}, we show that under these scaling regimes, the dynamics of the the hidden state may be described in terms of a class  of ordinary or stochastic differential equations, different from the neural ODEs studied in~\cite{CRBD2018,HR2018, lu2020mean}.   In Section~\ref{sec:backward}, we derive the large depth limit of the backpropagation dynamics under each scaling regime. 


{\bf Notations.} Let $\|v\| $ denote the Euclidean norm of a vector $v$. For a matrix $M$, denote $M^{\top}$ its transpose, $\diag(M)$ its diagonal vector, $\tr(M)$ its trace and $\norm{M}_F = \sqrt{\tr(M^{\top}M)}$ its Frobenius norm. Denote $\floor{x}$ the integer part of a real number $x$.  
 Let $\NN(m, \Sigma)$ denote the Gaussian distribution with mean $m$ and (co)variance $\Sigma$, $\otimes$ denote the tensor product, and $\R^{d, \, \otimes n} = \R^d \times \cdots \times \R^d$ ($n$ times). Define the vectorisation operator by ${\rm vec} \colon \R^{d_1{\times\cdots\times} d_n} \to \R^{d_1\cdots d_n}$, and let $\one_S$ be the indicator function of a set $S$. $\CC^0$ is the space of continuous functions, for $\nu\geq0$, $\CC^\nu$ is the space of $\nu$-Hölder continuous functions, and $\HH^1$ is the Sobolev space of order $1$.
\section{Scaling regimes} \label{sec:methodology}
We start by providing a framework for describing the scaling regimes for trained network weights, as identified in  numerical experiments  on deep ResNets \cite{cohen2021}. 

\subsection{Scaling regimes for trained network weights}

As described in Section~\ref{sec:intro}, the neural ODE limit assumes  
\begin{equation} \label{eq:thorpe}
\delta^{(L)} \sim  \frac{1}{L} \quad {\rm and} \quad A^{(L)}_{\floor{Lt} }\overset{L\to \infty}{\longrightarrow} \overbar{A}_t,\quad b^{(L)}_{\floor{Lt} }\overset{L\to \infty}{\longrightarrow} \overline{b}_t
\end{equation}
for $t\in[0,1]$, where $\overbar{A} \colon [0,1] \to \R^{d\times d}$ and $\overline{b} \colon [0,1] \to \R^{d}$ are smooth  functions~\cite{TvG2018}. 
Our numerical experiments, detailed in Section~\ref{sec:experiments}, show that the norm of the weights generally shrinks as $L$ increases (see for example Figures~\ref{fig:scaling_tanh_shared_delta} and~\ref{fig:scaling_relu_scalar_delta}), so one cannot expect the above assumption to hold, unless weights are renormalized in some way.
We consider here a more general assumption which includes~\eqref{eq:thorpe} but allows for shrinking weights.
\begin{hyp} \label{hypothesis.1}
There exist $\overbar{A} \in \CC^0\left( [0,1], \R^{d\times d} \right)$ and $\beta \in \left[0, 1\right]$ such that 
\begin{equation} \label{scaling-beta}
\forall s\in [0,1],  \qquad   \overbar{A}_s = \lim_{L\to\infty} L^{\beta}\,  A^{(L)}_{\floor{Ls}}.
\end{equation}
\end{hyp}

\begin{figure}[!tb]
    \centering
    \includegraphics[width=0.45\textwidth]{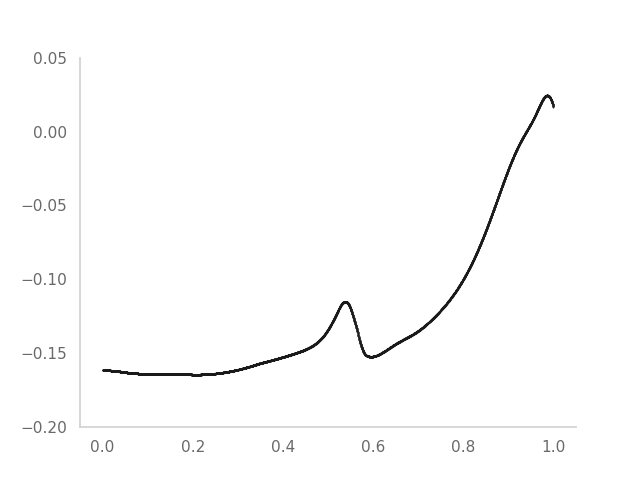}
    \includegraphics[width=0.45\textwidth]{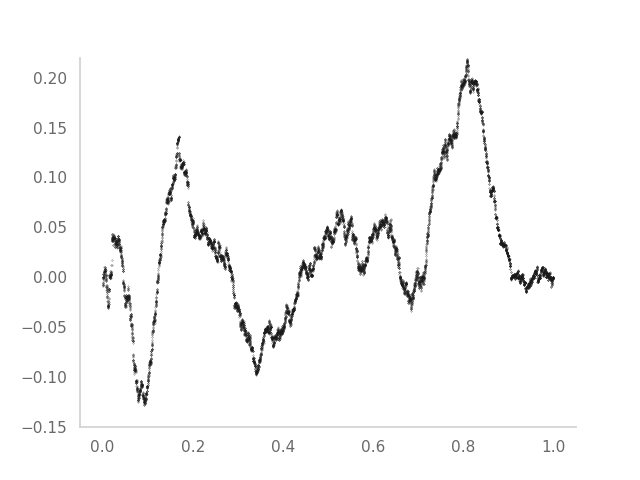}
    \caption{Trained weights as a function of $k/L$ for $k=0,\ldots,L$ and $L=9100$. Left: rescaled weights $L^\beta A^{(L)}_{k, (0, 0)}$ for a $\tanh$ network with $\beta=0.2$. Right: cumulative sum  $\sum_{j=0}^{k-1} A^{(L)}_{j, (0, 0)}$for a $\relu$ network. Note that each $A^{(L)}_{k, (0, 0)}\in\R$. \label{fig:trained_weights}}
\end{figure}
These renormalized weights do converge to a continuous function of the layer in some cases, as shown in Figure~\ref{fig:trained_weights} (top) which displays a ResNet~\eqref{forward-map-resnet} with fully connected layers and $\tanh$ activation function, without explicit regularization (see Section \ref{sec:fully_connected}).


Yet, it is not the case that network weights always converge to a smooth function of the layer, even after rescaling. Indeed, network weights $A_k^{(L)}$ are usually initialized to random, independent and identically distributed (i.i.d.) values, whose scaling limit would then correspond to a {\it white noise}, which cannot be represented as a function of the layer. In this case, the {\it cumulative sum} $\sum_{j=0}^{k-1} A^{(L)}_j$ of the weights behaves like a random walk, which does have a well-defined scaling limit $W \in \CC^0\left(\left[0, 1\right], \R^{d\times d} \right)$.
Figure \ref{fig:trained_weights} (bottom) shows that, for a $\relu$ ResNet with fully-connected layers, this cumulative sum of trained weights converges to an {\it irregular}, that is, non-smooth function of the layer. 

This observation motivates the consideration of a different scaling regime where the weights $A^{(L)}_k$ are represented as the {\it increments} of a continuous function $W^A$, i.e. the {\it cumulative sum} of the weights may converges to a limit but not the weight themselves. We also allow for a \textit{trend} term as in Scaling regime \ref{hypothesis.1}.  

\begin{hyp}\label{hypothesis.2}
There exist $\beta \in \left[0, 1\right)$, $\overbar{A} \in \CC^0\left( [0,1], \R^{d\times d} \right)$, and $W^A \in \CC^{0}([0,1], \mathbb{R}^{d\times d} )$ non-zero such that $W^A_0=0$ and
\begin{equation} \label{scaling-W}
A_k^{(L)} = L^{-\beta} \overbar{A}_{k/L} +  W^A_{(k+1)/L} - W^A_{k/L}.
\end{equation}
\end{hyp}
\noindent The above decomposition is unique. Indeed, for $s\in[0,1]$,
\begin{align}
L^{\beta -1} \sum_{k=0}^{\floor{Ls}-1} A_k^{(L)} &= L^{-1} \sum_{k=0}^{\floor{Ls}-1} \overbar{A}_{k/L} + L^{\beta - 1} W^A_{\floor{Ls}/L} \mathop{\to}^{L \to \infty} \int_{0}^s \overbar{A}_r \dd r. \label{cumul-decomp}
\end{align}
The integral of $\overbar{A}$ is thus uniquely determined by the weights $A_k^{(L)}$, so $\overbar{A}$ can be obtained by discretization and $W^A$ by fitting the residual error in \eqref{cumul-decomp}. In addition, Scaling Regimes~\ref{hypothesis.1} and~\ref{hypothesis.2} are mutually exclusive since Scaling regime~\ref{hypothesis.2} requires $W^A$ to be non-zero.

\begin{rmk} \label{rem:weight-init}
In the case of independent Gaussian weights
\begin{equation*}
    A_{k,mn}^{(L)} \iid \mathcal{N}\left(0, L^{-1} d^{-2} \right) \quad \text{and} \quad  b_{k,n}^{(L)} \iid \mathcal{N}\left(0, L^{-1}d^{-1}\right),
\end{equation*}
where $A_{k,mn}^{(L)}$ is the $(m,n)$-th entry of $A_k^{(L)}\in\R^{d\times d}$ and $b_{k,n}^{(L)}$ is the $n$-th entry of $b_k^{(L)}\in\R^d$, 
 we can represent the weights $\{ A^{(L)}, b^{(L)}\}$ as the increments of a matrix-valued Brownian motion
\begin{equation*}
    A_{k}^{(L)} = d^{-1} \left( W^A_{(k+1)/L} - W^A_{k/L} \right),
\end{equation*}
which is a special case of Scaling regime 2.
\end{rmk}
This remark shows that Scaling regime 2 corresponds to a 'diffusive' regime.



\subsection{Smoothness of weights with respect to the layer} \label{sec:tools}

A question related to the existence of a scaling limit is the degree of smoothness of the limits $\overbar{A}$ or $W^A$, if they exist. 
To quantify the smoothness of the function mapping the layer number to the corresponding network weight, we define in Table~\ref{tab:summary_results} several quantities which may be viewed as discrete versions of various (semi-)norms used to measure the smoothness of functions.

{
    \tabulinesep=1.3mm
        \captionof{table}{Quantities associated to a tensor $A^{(L)} \in \R^{L\times d \times d}$.}
    \begin{center}
    \begin{tabu}{p{5cm}|p{5cm}}
        \hline
        \textbf{Quantity} & \textbf{Definition}  \\ \hline
        Maximum norm & $\max_{k} \norm{A^{(L)}_k}_F$ \\
        Cumulative sum norm & $\norm{ \sum_{k=1}^L A_k^{(L)} }_F$ \\
        $\beta$-scaled norm of increments & $L^{\beta} \max_{k} \norm{A^{(L)}_{k+1} - A^{(L)}_{k}}_F$ \\
        Root sum of squares & $\left(\sum_k \norm{A^{(L)}_k}_F^2\right)^{1/2}$ \\
        \hline
    \end{tabu}
    \end{center}
    \label{tab:summary_results}
}

\section{Scaling behavior of trained weights: numerical experiments} \label{sec:experiments}

We  now report on detailed numerical experiments to investigate the scaling properties and asymptotic behavior of trained weights for residual networks as the number of layers increases. We focus on two types of architectures: fully-connected and convolutional networks.

\subsection{Methodology} \label{sec:procedure-experiments}

We underline that Scaling Regimes~\ref{hypothesis.1} and~\ref{hypothesis.2} are mutually exclusive since Scaling regime~\ref{hypothesis.2} requires $W^A$ to be non-zero. In order to examine whether one of these scaling regimes, or neither, holds for the trained weights $A^{(L)}$ and $b^{(L)}$, we proceed as follows.


\noindent{\bf Step 1:} First, to obtain the scaling exponent $\beta\in [0,1)$, note that under Scaling regime \ref{hypothesis.2}, \[
L^{\beta - 1}\sum_{k=1}^L A_k^{(L)} = \frac{1}{L}\sum_{k=1}^L \overbar{A}_{k/L} + L^{\beta-1} W^A_1\qquad  \mathop{\to}^{L \to \infty} \int_0^1 \overbar{A}_s \dd s.
\]
Hence, we perform a logarithmic regression of the cumulative sum norm of $A^{(L)}$ with respect to $L$, and the rate of increase of $\sum_{k=1}^L A_k^{(L)}$ as $L\to\infty$ is $1-\beta$.

\noindent{\bf Step 2:} After identifying the correct scale $L^{-\beta}$ for the weights, we compute the $\beta$-scaled norm of increments of $A^{(L)}$ to check whether they satisfy Scaling regime~\ref{hypothesis.1} and measure the smoothness of the trained weights. On one hand, if the $\beta$-scaled norm of increments of $A^{(L)}$ does not vanish as $L\to\infty$, it means that the rescaled weights cannot be represented as a continuous function of the layer, as in Scaling regime~\ref{hypothesis.1}. On the other hand, if the $\beta$-scaled norm of increments of $A^{(L)}$ vanishes (say, as $L^{-\nu}$) when $L$ increases, it supports Scaling regime~\ref{hypothesis.1} with a Hölder-continuous limit function $\overbar{A}\in \CC^\nu([0,1], \R^{d\times d})$.

\noindent{\bf Step 3:} To discriminate between Scaling regimes~\ref{hypothesis.1} and ~\ref{hypothesis.2}, we decompose the cumulative sum $\sum_{j=0}^{k-1} A_j^{(L)}$ of the trained weights into a {\it trend} component $\overbar{A}$ and a {\it noise} component $W^A$, as shown in~\eqref{cumul-decomp}. The presence of non-negligible noise term $W^A$ favors Scaling regime~\ref{hypothesis.2}.

\noindent{\bf Step 4:} Finally, we estimate the regularity of  the term $W^A$ under Scaling regime~\ref{hypothesis.2}. If $W^A$ has  {\it diffusive} behavior, as in the example of i.i.d. random weights, then its quadratic variation tensor defined  by
\begin{align*}
\left[ W^A \right]_s &= \lim_{L \to \infty} \sum_{k=0}^{\floor{Ls}-1} \left( W^A_{\frac{k+1}{L}} - W^A_{\frac{k}{L}} \right) \otimes  \left( W^A_{\frac{k+1}{L}} - W^A_{\frac{k}{L}} \right)^{\top}
\end{align*}
has a finite limit as $L\to\infty$.  Hence, using~\eqref{scaling-W} and Cauchy-Schwarz, we obtain 
\begin{equation} \label{qv-sqrtssq}
\tensornorm{\left[ W^A \right]_s} \leq 2 \cdot \lim_{L\to\infty} \sum_{k=0}^{\floor{Ls}-1} \norm{A^{(L)}_k}_F^2 + L^{1-2\beta} \norm{\overbar{A}}^2_{L^2}
\end{equation}
where $\tensornorm{\cdot}$ is the Hilbert-Schmidt norm. As $\overbar{A}$ is continuous on a compact domain, its $L^2$ norm is finite. Hence, if we have $\beta \geq \frac{1}{2}$, the fact that the root sum of squares of $A^{(L)}$ is upper bounded as $L\to\infty$ implies that the quadratic variation of $W^A$ is finite. \\

We follow all of the above steps for $b^{(L)}$ as well. Note that the scaling exponent $\beta$ may not be the same for $A^{(L)}$ and $b^{(L)}$.

\begin{rmk}\label{rmk:relu}
Note that $\sigma = \relu$ is homogeneous of degree $1$, so we can write
\begin{equation*}
\delta \cdot \sigma_d\left(Ah+b \right) = \sign(\delta) \cdot \sigma_d\left(\abs{\delta}A h + \abs{\delta}b \right).
\end{equation*}
Hence, when analyzing the scaling of trained weights in the case of a $\relu$ activation with fully-connected layers, we look at the quantities $\abs{\delta^{(L)}} A^{(L)}$ and $\abs{\delta^{(L)}} b^{(L)}$, as they represent the total scaling of the residual connection. \end{rmk}
\subsection{Results for fully-connected layers \label{sec:fully_connected}}
We first consider the case where the network layers are fully-connected. We consider the network architecture~\eqref{forward-map-resnet} for two different setups:
\begin{enumerate}[label=(\roman*)]
    \item $\sigma = \tanh$, $\delta^{(L)}_k = \delta^{(L)} \in \R_+$ trainable,
    \item $\sigma = \relu$, $\delta^{(L)}_k \in \R$ trainable.
\end{enumerate}
We choose to present these two cases for the following reasons. First, both $\tanh$ and $\relu$ are widely used in practice. Further, having $\delta^{(L)}$ scalar makes the derivation of the limiting behavior simpler. Also, since $\tanh$ is an odd function, the sign of $\delta^{(L)}$ can be absorbed into the activation. Therefore, we can assume that $\delta^{(L)}$ is non-negative for $\tanh$. Regarding $\relu$, having a shared $\delta^{(L)}$ would hinder the expressiveness of the network. Indeed, if for instance $\delta^{(L)} > 0$, we would get $h^{(L)}_{k+1} \geq  h^{(L)}_{k}$ element-wise since $\relu$ is non-negative. This would imply that $h^{(L)}_L \geq x$, which is not desirable. The same argument applies to the case $\delta^{(L)} < 0$. Thus, we let $\delta^{(L)}_k \in \R$ depend on the layer number for $\relu$ networks.

We consider two data sets. The first one is synthetic: fix $d=10$ and generate $N$ i.i.d samples $x_i$ coming from the $d-$dimensional uniform distribution in $[-1, 1]^d$. Let $K=100$ and simulate the following dynamical system:
\begin{equation*}
    \begin{cases}
    z_{0}^{x_i}&=x_i \\
    z_{k}^{x_i} &= z_{k-1}^{x_i}+ K^{-1/2}\tanh_d\left(g_d\left(z_{k-1}^{x_i},k, K\right) \right), \,\,\,  k = 1, \ldots, K,
    \end{cases}
\end{equation*}
where 
$g_d(z,k,K)\coloneqq\sin(5k\pi/K)z+\cos( 5k\pi/K )\one_d$. The targets $y_i$ are defined as $y_i = z_{K}^{x_i} / \norm{z_{K}^{x_i}}$. The motivation behind this low-dimensional dataset is to be able to train very deep residual networks and to be sure that there exists at least a (sparse) optimal solution.

The second dataset is a low-dimensional embedding of the MNIST handwritten digits dataset~\cite{MNIST}. Let $\left( \widetilde{x}, c \right) \in \R^{28\times 28} \times \left\{0, \ldots, 9 \right\}$ be an input image and its corresponding class. We transform $\widetilde{x}$ into a lower dimensional embedding $x \in \R^{d}$ using an untrained convolutional projection, where $d=25$. More precisely, we stack two convolutional layers initialized randomly, we apply them to the input and we flatten the downsized image into a $d-$dimensional vector. Doing so reduces the dimensionality of the problem while allowing very deep networks to reach at least $99\%$ training accuracy. The target $y\in \R^d$ is the one-hot encoding of the corresponding class. 

The weights are updated by stochastic gradient descent (SGD) using batches of size $B$ on the  mean-square loss and a constant learning rate $\lr$, until the loss falls below $\epsilon$, or when the maximum number of updates $T_{\max}$ is reached. We repeat the experiments for depths $L$ varying from $L_{\min}$ to $L_{\max}$. Details are given in Appendix~\ref{sec:hyperparameters}. \\

{\bf Results.} For the case of a $\tanh$ activation (i), we observe in Figure~\ref{fig:scaling_tanh_shared_delta} that for both datasets, $\delta^{(L)} \sim L^{-0.7}$ clearly decreases as $L$ increases, and $A^{(L)}$ decreases slightly when $L$ increases. We deduce that $\beta = 0.3$ for the MNIST dataset and $\beta = 0.2$ for the synthetic dataset.

\begin{figure}[tb!]
    \centering
    \includegraphics[width=0.45\textwidth]{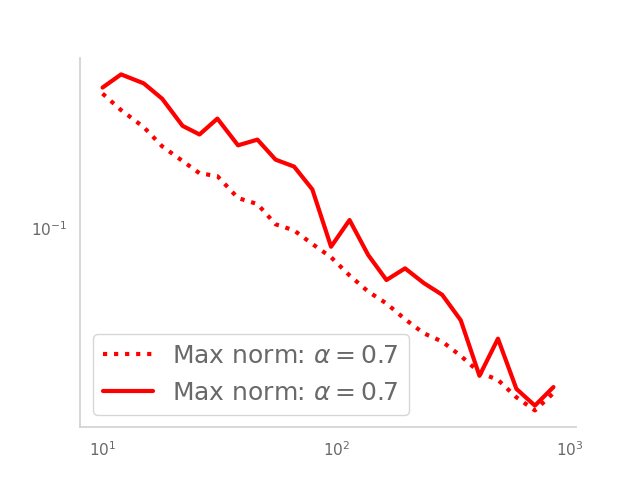}
    \includegraphics[width=0.45\textwidth]{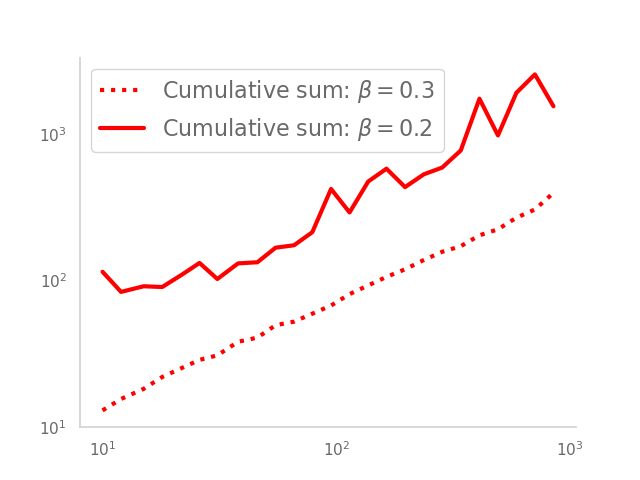}
    \caption{Scaling for $\tanh$ activation and $\delta^{(L)}\in\R$. Left: Maximum norm of $\delta^{(L)}$ with respect to $L$. Right: Cumulative sum norm of $A^{(L)}$ with respect to $L$. The dashed lines are for the synthetic data and the solid lines are for MNIST. The plots are in log-log scale.}
    \label{fig:scaling_tanh_shared_delta}
\end{figure}

We use these results to identify the scaling behavior of $A^{(L)}$. We observe in Figure~\ref{fig:hypothesis_tanh_shared_delta} (left) that the $\beta$-scaled norm of increments of $A^{(L)}$ decreases like $L^{-1/2}$, suggesting that Scaling regime~\ref{hypothesis.1} holds, with $\overbar{A}$ being $1/2-$Hölder continuous. This is confirmed in Figure~\ref{fig:hypothesis_tanh_shared_delta} (right), as the trend part $\overbar{A}$ is visibly continuous and even of class $\CC^1$. The noise part $W^A$ is negligible. This observation is even more striking given that the weights are trained \textbf{without explicit regularization}.

\begin{figure}[tb!]
    \centering
    \includegraphics[width=0.45\textwidth]{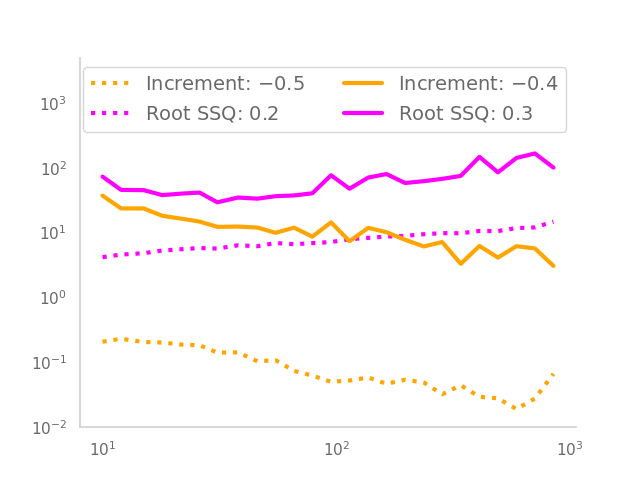}
    \includegraphics[width=0.45\textwidth]{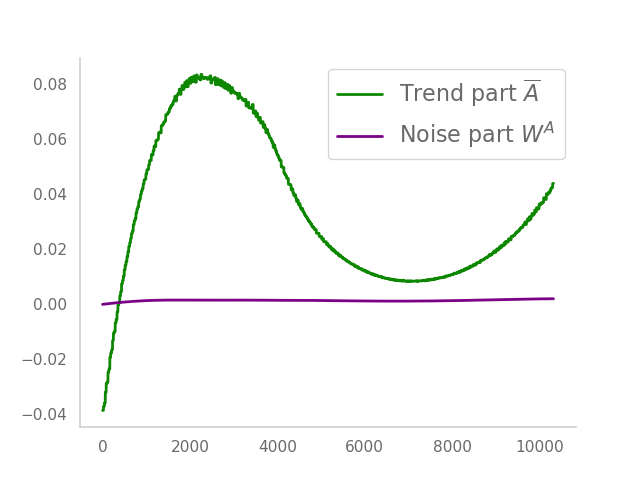}
    \caption{Identification of scaling behavior in the case of $\tanh$ activation and $\delta^{(L)}\in\R$. Left: log-log plot of root sum of squares of $A^{(L)}$ (pink) and the $\beta$-scaled norm of increments of $A^{(L)}$ (orange). Dashed lines are for the synthetic data and the solid lines are for MNIST. Right: Decomposition of the trained weights $A^{(L)}_{k, (9, 7)}$ with the trend part $\overbar{A}$ and the noise part $W^A$ for $L=10321$, as defined in~(\ref{scaling-W}), for the synthetic dataset.}
    \label{fig:hypothesis_tanh_shared_delta}
\end{figure}

Regarding the case of a $\relu$ activation function (ii), we observe in Figure~\ref{fig:scaling_relu_scalar_delta} (left) that the trend part of the residual connection $\abs{ \delta^{(L)}} A^{(L)}$ scales like $L^{-0.8}$ for the synthetic dataset and like $L^{-0.9}$ for the MNIST dataset. We see in Figure~\ref{fig:scaling_relu_scalar_delta} (right) that keeping the sign of $\delta^{(L)}_k$ is important, as the sign oscillates considerably throughout the network depth $k=0, \ldots, L-1$.

\begin{figure}[tb!]
    \centering
    \includegraphics[width=0.45\textwidth]{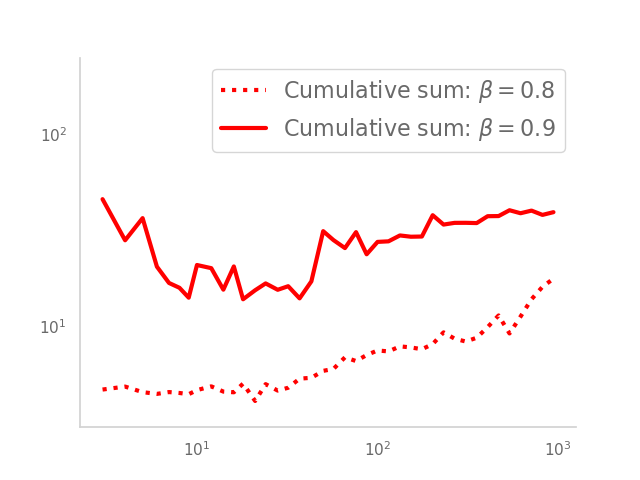}
    \includegraphics[width=0.45\textwidth]{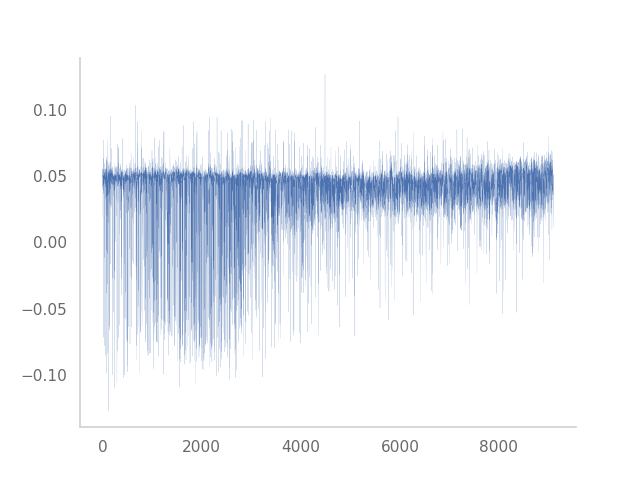}
    \caption{Scaling for $\relu$ activation and $\delta_k^{(L)}\in\R$. Left: Cumulative sum norm of $\vert \delta^{(L)}\vert A^{(L)}$ with respect to $L$, in log-log scale. Right: trained values of $\delta^{(L)}_k$ as a function of $k$, for $L=9100$ and for the synthetic dataset.}
    \label{fig:scaling_relu_scalar_delta}
\end{figure}

Figure~\ref{fig:hypothesis_relu_delta_scalar} (left) shows that the $\beta$-scaled norm of increments diverges as the depth increases. This suggests that there exists a noise part $W^A$. Following~\eqref{qv-sqrtssq}, the fact that the root sum of squares of $\abs{\delta^{(L)}} A^{(L)}$ is upper bounded as $L\to\infty$ and $\beta \geq 1/2$ implies that $W^A$ has finite quadratic variation. These claims are also supported by Figure~\ref{fig:hypothesis_relu_delta_scalar} (right): there is a non-zero trend part $\overbar{A}$, and a non-negligible noise part $W^A$.

\begin{figure}[tb!]
    \centering
    \includegraphics[width=0.45\textwidth]{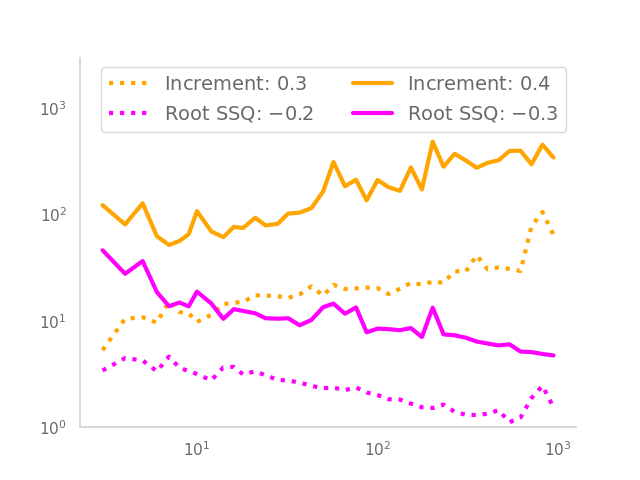}
    \includegraphics[width=0.45\textwidth]{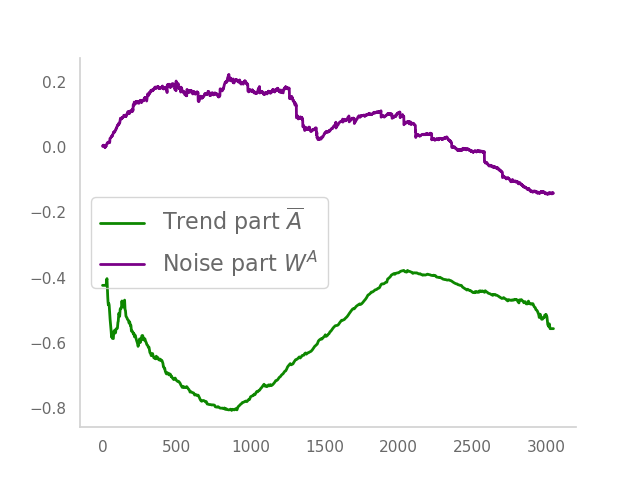}
    \caption{$\relu$ activation and scalar $\delta_k^{(L)}$. Left: in pink we plot in log-log scale the root sum of squares of $\vert \delta^{(L)}\vert A^{(L)}$, and in orange the $\beta$-scaled norm of increments of $\vert \delta^{(L)}\vert A^{(L)}$. The dashed lines are for the synthetic data and the solid lines for MNIST. Right: Decomposition of the trained weights $ \vert \delta^{(L)} \vert \, A^{(L)}_{k, (7, 7)}$ with the trend part $\overbar{A}$ and the noise part $W^A$ for $L=10321$, as defined in~(\ref{scaling-W}), for the synthetic dataset.
    \label{fig:hypothesis_relu_delta_scalar}}
\end{figure}

Given the scaling behavior of the trained weights, we conclude that Scaling regime~\ref{hypothesis.1} seems to be a plausible description for the $\tanh$ case (i), but Scaling regime~\ref{hypothesis.2} provides a better description for the $\relu$ case (ii).

Scaling behavior of
$b^{(L)}$ are shown for the $\tanh$ case in Figure~\ref{fig:analysis_b_tanh} and for the $\relu$ case in Figure~\ref{fig:analysis_b_relu}. We observe that the cumulative sum norm, the scaled norm of the increments and the root sum of squares of $b^{(L)}$ scales in the same way as $A^{(L)}$ as the depth $L$ increases. In particular, the scaling exponent $\beta$ for $b^{(L)}$ is equal to the scaling exponent of $A^{(L)}$, justifying the setup considered in Section \ref{sec:methodology}.  

\begin{figure}[tb!]
    \centering
    \includegraphics[width=0.32\textwidth]{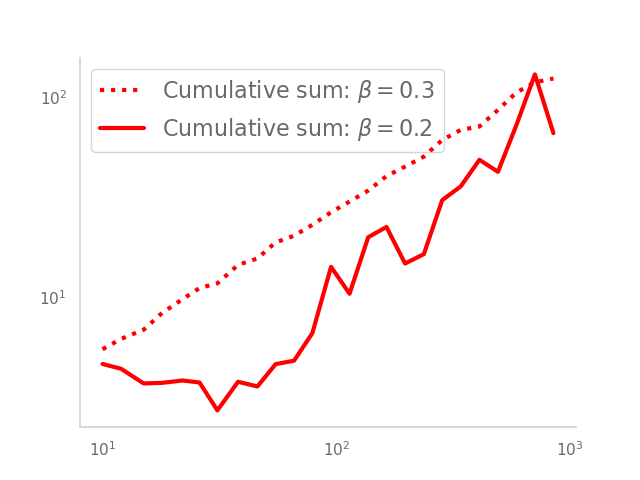}
    \includegraphics[width=0.32\textwidth]{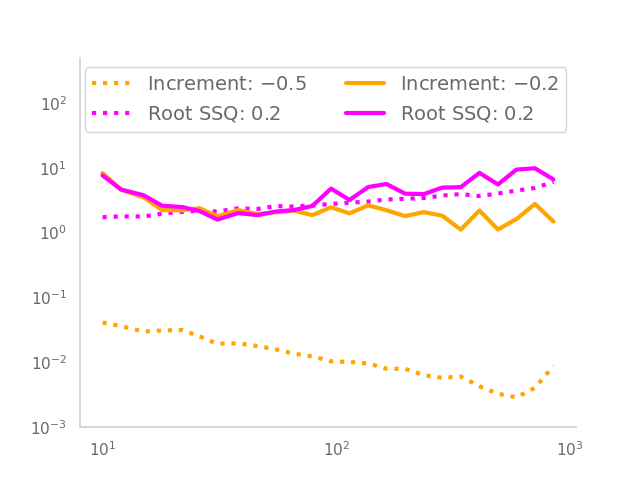}
    \includegraphics[width=0.32\textwidth]{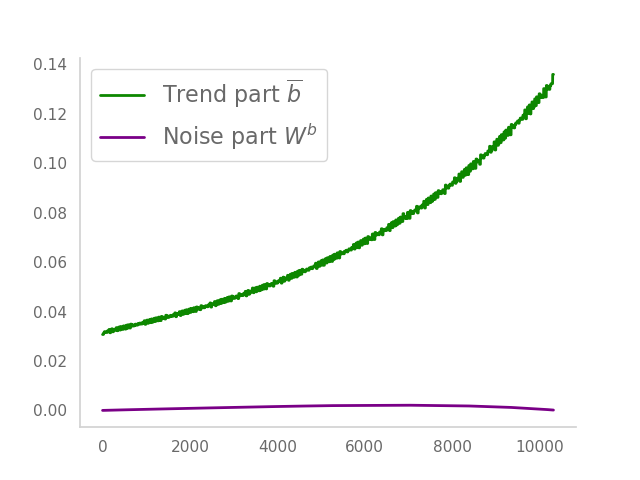}
    \caption{Scaling behavior for $b^{(L)}$ with $\tanh$ activation and scalar $\delta^{(L)}$. Left: cumulative sum norm of $b^{(L)}$ with respect to $L$, in log-log scale. Middle: the root sum of squares of $b^{(L)}$ in pink and the $\beta-$scaled norm of increments of $b^{(L)}$ in orange, in log-log scale. The dashed lines are for the synthetic data and the solid lines are for MNIST. Right: Decomposition of the trained weights $b^{(L)}_{k, 5}$ with the trend part $\overline{b}$ and the noise part $W^b$ for $L=10321$, as defined in~(\ref{scaling-W}), for the synthetic dataset.}
    \label{fig:analysis_b_tanh}
\end{figure}

\begin{figure}[tb!]
    \centering
    \includegraphics[width=0.32\textwidth]{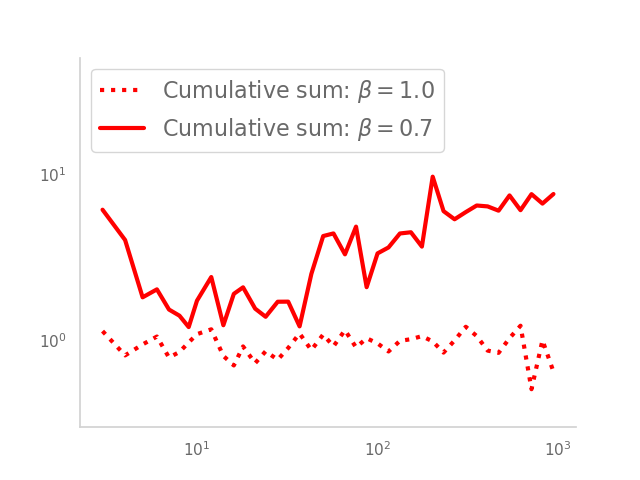}
    \includegraphics[width=0.32\textwidth]{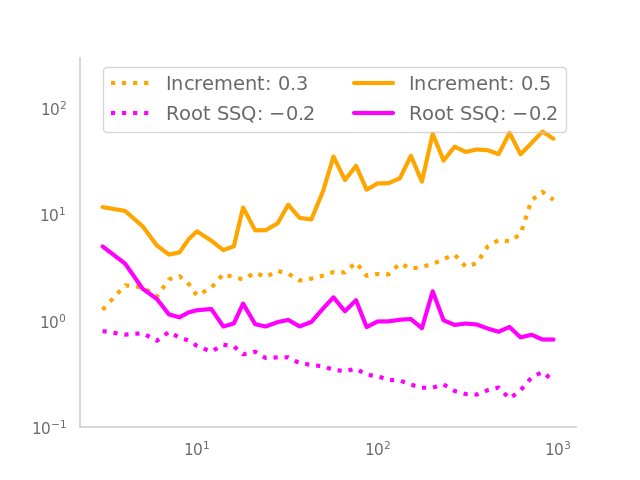}
    \includegraphics[width=0.32\textwidth]{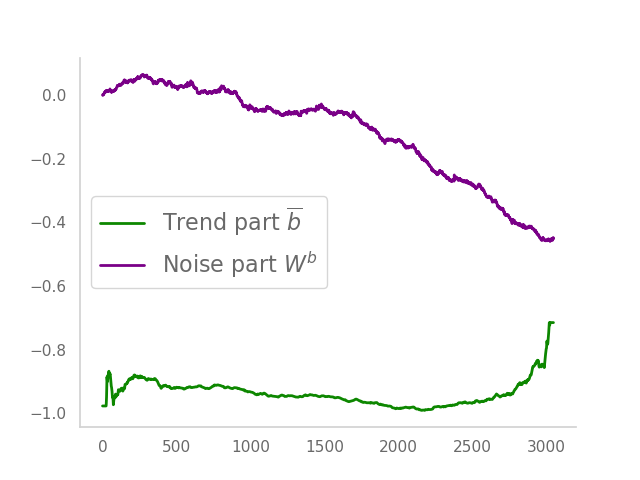}
    \caption{Scaling and hypothesis verification for $b^{(L)}$ with  $\relu$ activation and $\delta_k^{(L)}\in\R$. Left: cumulative sum norm of $\vert \delta^{(L)}\vert b^{(L)}$ with respect to $L$, in log-log scale. Middle: the root sum of squares of $\vert \delta^{(L)}\vert b^{(L)}$ in pink and the $\beta-$scaled norm of increments of $\vert \delta^{(L)}\vert b^{(L)}$ in orange, in log-log scale. The dashed lines are for the synthetic data and the solid lines for MNIST. Right: Decomposition of the trained weights $ \vert \delta^{(L)} \vert \, b^{(L)}_{k, 6}$ with the trend part $\overline{b}$ and the noise part $W^b$ for $L=10321$, as defined in~(\ref{scaling-W}), for the synthetic dataset.}
    \label{fig:analysis_b_relu}
\end{figure}

{\bf Importance of the stochastic term $W^A$.} It is legitimate  to ask whether the noise term $W^A$ plays a significant role in the output accuracy of the network. To test this, we create a residual network with denoised weights $\widetilde{A}^{(L)}_k \coloneqq  L^{-\beta} \overbar{A}_{k/L}$, compute its training error and we compare it to the original training error. We observe in Figure~\ref{fig:loss_denoised} (left) that for $\tanh$, the noise part $W^A$ is negligible and does not influence the loss. However, for $\relu$, the loss with denoised weights is one order of magnitude above the original training loss, meaning that the noise part $W^A$ plays a significant role in the accuracy of the trained network.

\begin{figure}[tb!]
    \centering
    \includegraphics[width=0.45\textwidth]{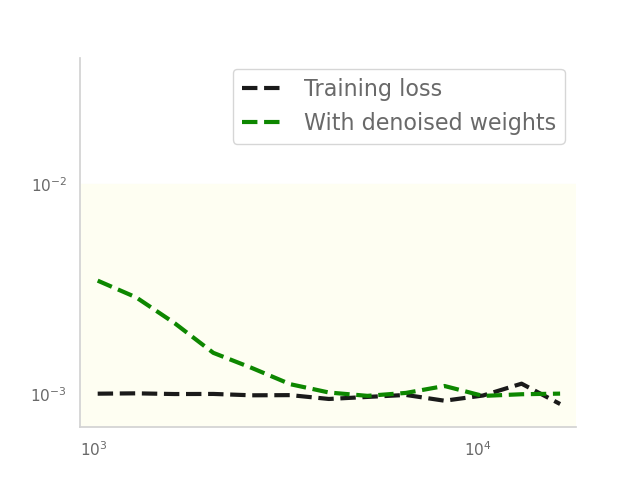}
    \includegraphics[width=0.45\textwidth]{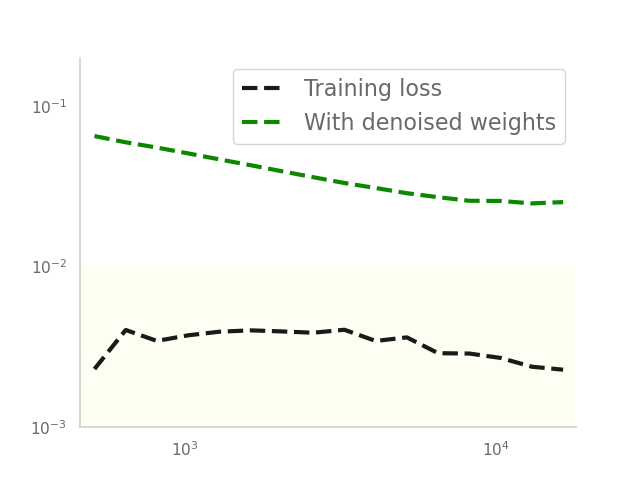}
    \caption{Loss value, as a function of $L$, in black for the trained weights $A_k^{(L)}$ and in green for the denoised weights $\widetilde{A}^{(L)}_k = L^{-\beta} \overbar{A}_{k/L}$. Left: $\tanh$ activation and $\delta^{(L)}\in\R$. Right: $\relu$ activation and $\delta_k^{(L)}\in\R$. Note that these curves are for the synthetic dataset and that we plot them in log-log scale. Also, we show in off-white the loss value range in which we consider our networks to have converged.}
    \label{fig:loss_denoised}
\end{figure}

{\bf Sensitivity of $\alpha$ and $\beta$ with respect to the hyperparameters.} The values of $\alpha$ and $\beta$ stem from the trained weights, which are themselves a function of the initialization and the training algorithm. We are using stochastic gradient descent, and the most significant hyperparameters of SGD are the learning rate $\eta$ and the batch size $B$. \\
Hence, we report the value $\alpha$ and $\beta$ found for the $\tanh$ and trainable $\delta$ architecture on the synthetic data with different batch sizes $B \in\left\{8, 32, 128\right\}$ and learning rates $\eta\in\left\{0.01, 0.003, 0.001 \right\}$, with $5$ different realizations for the initialization. We report the average values of $\alpha$ and $\beta$ for $5$ different seeds in Table \ref{table:hp-sensi} below.

\captionof{table}{Average value of $\alpha$ (left) and $\beta$ (centre) for the trained weights, over $5$ random initialization. $\eta$ is the learning rate, $B$ the batch size.}
\label{table:hp-sensi}
\begin{minipage}[tb!]{0.48\textwidth}
    \tabulinesep=1.6mm
    \setlength\tabcolsep{4pt}
    \begin{tabu}{|c|c|c|c|}
        \hline
        $\alpha$ & $B=8$ & $B=32$ & $B=128$  \\ \hline
        $\eta=.01$ & $.69 \pm .02$ & $.73\pm .02$ & $.67 \pm .02$  \\ \hline
        $\eta=.003$ & $.59\pm .05$ & $.60\pm .01$ & $.58\pm .01$ \\ \hline
        $\eta=.001$ & $.58\pm .01$ & $.55\pm .01$ & $.53 \pm .01$ \\ \hline
    \end{tabu}
\end{minipage}
\hspace{0.51cm}
\begin{minipage}[t]{0.48\textwidth}
    \tabulinesep=1.5mm
    \setlength\tabcolsep{4pt}
    \begin{tabu}{|c|c|c|c|}
        \hline
        $\beta$ & $B=8$ & $B=32$ & $B=128$  \\ \hline
        $\eta=.01$ & $.24\pm .02$ & $.29 \pm .05$ & $.22\pm .02$    \\ \hline
        $\eta=.003$ & $.33\pm .01$ & $.41\pm .06$ & $.40 \pm .02$ \\ \hline
        $\eta=.001$ & $.39\pm .02$ & $.43\pm .02$ & $.41 \pm .01$  \\ \hline
    \end{tabu}
\end{minipage}
\vspace{12pt}

We observe that the learning rate does affect $\alpha$ and $\beta$ while keeping $\alpha+\beta$ around $1$, and the batch size does not affect $\alpha$ or $\beta$. A plausible explanation for these observations is that a higher batch size means a more precise descent direction at the cost of efficiency, but the shape of the solution is not supposed to change. 

\subsection{Results for convolutional networks \label{sec:classification}}

We now consider the original ResNet with convolutional layers introduced in~\cite{HZRS2016}. This architecture is close to the state-of-the-art methods used for image recognition tasks. We do not include batch normalization~\cite{DBLP:journals/corr/IoffeS15} since it only slightly improves the performance of the network while making the analysis significantly more complicated.
The architecture is displayed in Figures \ref{fig:cifar_net_basic} and \ref{fig:cifar_net_block}. 

\begin{figure}[tb!]
    \centering
    \includegraphics[width=1.0\textwidth]{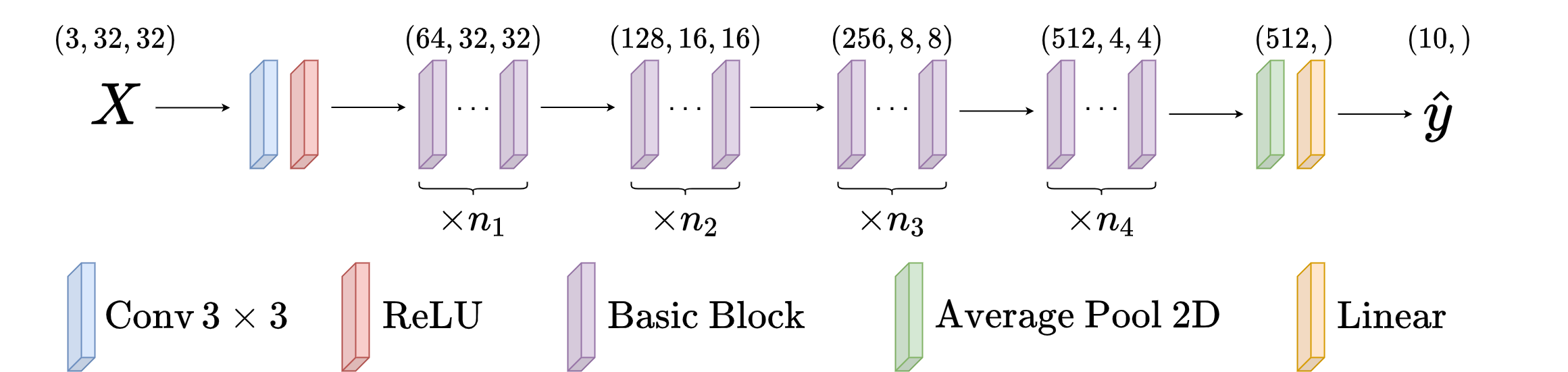}
    \vspace{0.1in}
    \caption{Residual architecture. There are 4 blocks that are respectively repeated $n_1$, $n_2$, $n_3$ and $n_4$ times. The network depth is $L=n_1+n_2+n_3+n_4$. The Basic Block architecture is detailed in Figure~\ref{fig:cifar_net_block}.}
    \label{fig:cifar_net_basic}
\end{figure}

\begin{figure}[tb!]
    \centering
    \includegraphics[width=1.0\textwidth]{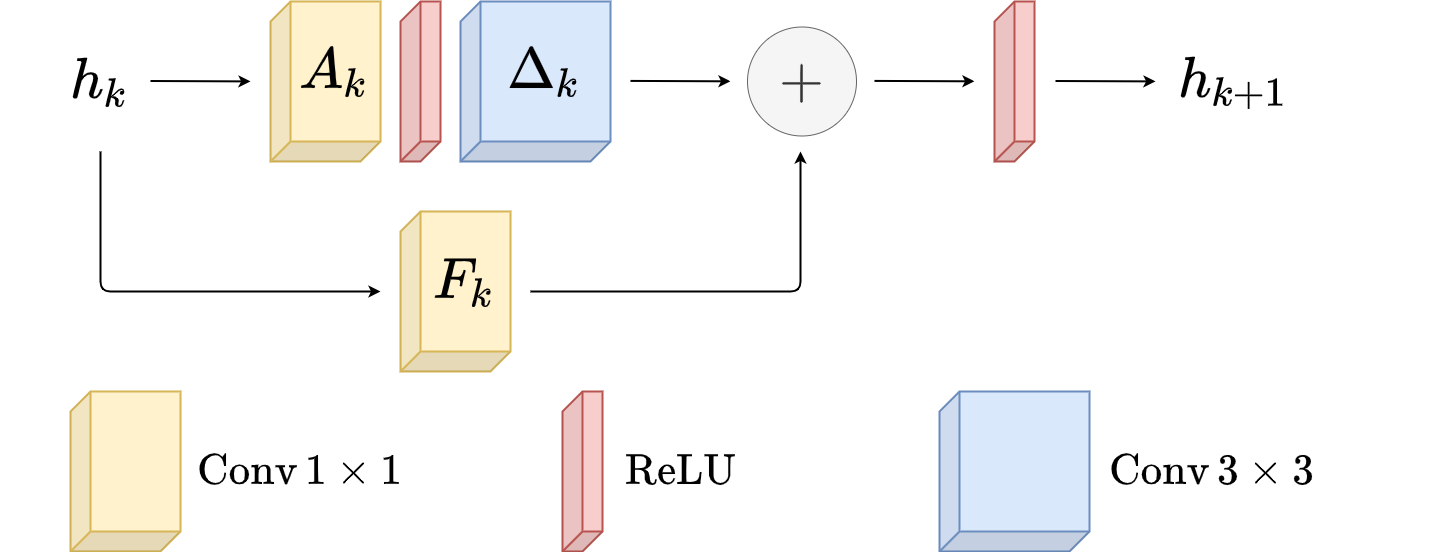}
    \vspace{0.1in}
    \caption{Basic Block from Figure~\ref{fig:cifar_net_basic}. See~\eqref{eq:2d-forward-map} for details.}
    \label{fig:cifar_net_block}
\end{figure}
Our network still possesses the skip connections from~(\ref{forward-map-resnet}): the dynamics of the hidden state reads
\begin{equation} \label{eq:2d-forward-map}
    h_{k+1} = \sigma\left(h_{k} +\Delta_k * \sigma\left(A_k*h_{k}\right)+F_k*h_{k}\right)
\end{equation}
for $k=0,\ldots,L-1$, where $\sigma=\relu$. Here, $\Delta_k$, $A_k$, and $F_k$ are kernels and $*$ denotes convolution. Note that $\Delta_k$ plays the same role as $\delta_k^{(L)}$ from~(\ref{forward-map-resnet}). To lighten the notation, we omit the superscripts $x$ (the input) and $L$ (the number of layers).

We train our residual networks at depths ranging from $L_{\min}=8$ to $L_{\max}=121$ on the CIFAR-10~\cite{CIFAR-10} dataset with the unregularized relative entropy loss. Here,  'depth' is the number of residual connections. We note that a network with $L_{\max}=121$ is already very deep. As a comparison, a standard ResNet-152~\cite{HZRS2016} has depth $L=50$ in our framework. \\


\begin{figure}[tb!]
    \includegraphics[width=0.45\textwidth]{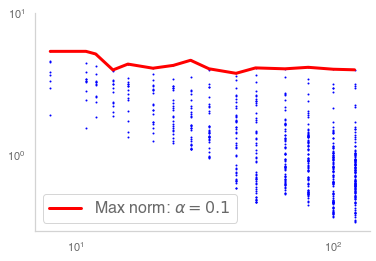}
    \includegraphics[width=0.45\textwidth]{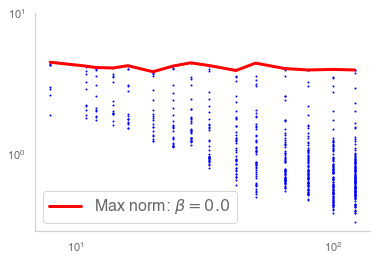}
    \caption{Scaling of $\Delta^{(L)}$ (left) and $A^{(L)}$ (right) against the network depth $L$ for convolutional architectures on CIFAR-10. In blue: spectral norm of the kernels $\Delta^{(L)}_k$, resp. $A^{(L)}_k$, for $k=0, \ldots, L-1$. In red: maximum norm, defined in Table~\ref{tab:summary_results}. The plots are in log-log scale.}
    \label{fig:cifar_scaling_with_depth}
\end{figure}

{\bf Results.} Table~\ref{tab:cifar_test_acc} shows the  accuracy of our convolutional residual networks trained on an NVIDIA GeForce RTX 2080 GPU on the CIFAR-10 dataset. The results are in line with those of traditional ResNet architectures~\cite{HZRS2016}, even though our networks do not have batch normalization layers~\cite{DBLP:journals/corr/IoffeS15}. It is also noteworthy to add that our concept of depth is not that of traditional ResNets. We define the number of layers $L$ as the number of skip connections in the network, that is the number of $\Delta_k$ kernels in~\eqref{eq:2d-forward-map}.

\newpage
{
    \tabulinesep=1.3mm
        \captionof{table}{Learning error in $\%$ on CIFAR-10 for each network depth $L$.}
    \begin{center}
    \begin{tabu}{c|c|c|c|c|c|c|c|c}
        \hline
        $L$ & 8 & 11 & 12 & 14 & 16 & 20 & 24 & 28 \\ 
        \hline
        Test error & 6.64 & 6.37 & 6.32 & 5.98 & 6.25 & 5.98 & 6.24 & 7.03 \\ 
        \hline \hline
        $L$ & 33 & 42 & 50 & 65 & 80 & 100 & 121 & \\ 
        \hline
        Test error & 6.13 & 6.21 & 6.32 & 6.19 & 6.30 & 6.20 & 6.37 & \\
        \hline
    \end{tabu}
    \end{center}
    \label{tab:cifar_test_acc}
    \vspace{0.2cm}
}

As in Section~\ref{sec:fully_connected}, we investigate how the weights scale with depth and whether Scaling regime~\ref{hypothesis.1} or Scaling regime~\ref{hypothesis.2} holds true for convolutional layers. To that end, we follow the steps of~\cite{DBLP:journals/corr/abs-1805-10408} to get the singular values, and therefore the spectral norms, of the linear operators defined by the convolutional kernels $\Delta_k^{(L)}$ and $A_k^{(L)}$. Figure~\ref{fig:cifar_scaling_with_depth} shows the maximum norm, and hence the scaling of $\Delta^{(L)}$ and $A^{(L)}$ against the network depth $L$. We observe that $\Delta^{(L)} \sim L^{-\alpha}$ and $A^{(L)} \sim L^{-\beta}$ with $\alpha=0.1$ and $\beta = 0$.



\begin{figure}[tb!]
    \includegraphics[width=0.45\textwidth]{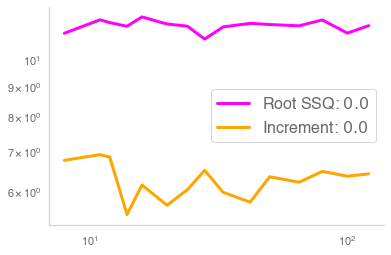}
    \includegraphics[width=0.45\textwidth]{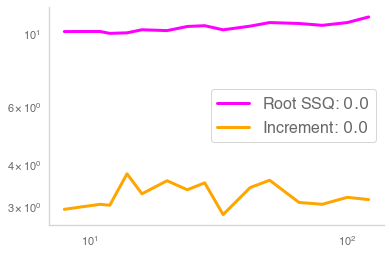}
    \caption{Scaling behavior of $\Delta^{(L)}$ (left) and $A^{(L)}$ (right). We plot in pink the root sum of squares and in orange the $\alpha$-scaled norm of increments of $\Delta^{(L)}$ (left) and the $\beta$-scaled norm of increments of $A^{(L)}$ (right). Plots are in log-log scale. The root sum of squares and the scaled norm of increments are defined in Table~\ref{tab:summary_results}. We obtain $\alpha$ and $\beta$ from Figure~\ref{fig:cifar_scaling_with_depth}.}
    \label{fig:cifar_hyp_with_depth}
\end{figure}

We then use the values obtained for $\alpha$ and $\beta$ to verify which Scaling regime holds. Figure~\ref{fig:cifar_hyp_with_depth} shows that both the $\alpha$-scaled norm of increments of $\Delta^{(L)}$ and the $\beta$-scaled norm of increments of $A^{(L)}$ seem to have lower bounds as the depth grows. This suggests that Scaling regime~\ref{hypothesis.1} does not hold for convolutional layers.

We also observe that the root sum of squares stays in the same order as the depth increases. Coupled with the fact that the maximum norms of $\Delta^{(L)}$ and $A^{(L)}$ are close to constant order as the depth increases, this suggests that the scaling limit is sparse with a finite number of weights being of constant order in $L$.

\subsection{Summary: three scaling regimes \label{sec:exp_summary}}

Our experiments show different scaling regimes for trained weights based on the network architecture.F
or fully-connected layers with $\tanh$ activation and a shared $\delta^{(L)}\in\R$, we observe a behavior consistent with Scaling regime~\ref{hypothesis.1} for both the synthetic dataset and MNIST.  
For fully-connected layers with $\relu$ activation and $\delta_k^{(L)}\in\R$, we observe that Scaling regime~\ref{hypothesis.2} holds for the synthetic dataset and MNIST. We deduce that the results for fully-connected layers are consistent with our findings in Figure~\ref{fig:trained_weights}.

In the case of convolutional architectures trained on CIFAR-10 and presented in Section~\ref{sec:classification}, we observe that the maximum norm of the trained weights does not decrease with the network depth and the trained weights display a sparse structure, indicating a third scaling regime corresponding to sparse scaling limits for both $\Delta^{(L)}$ and $A^{(L)}$. These results are consistent with previous evidence on the existence of sparse CNN representations for image recognition \cite{mallat2016}. We stress that the setup for our CIFAR-10 experiments has been chosen to approach state-of-the-art performance with our generic architecture, as shown in Figures \ref{fig:cifar_net_basic} and \ref{fig:cifar_net_block}.




\section{Deep network limit} \label{sec:results}

In this section, we study the scaling limit of the hiddent state dynamics \eqref{forward-map-resnet} under scaling regimes \ref{hypothesis.1} and  \ref{hypothesis.2}.

\subsection{Scaling regime \ref{hypothesis.1}: ODE limit}
First, we show that the scaling regime~\ref{hypothesis.1} together with a smooth and Lipschitz-continuous activation function lead to two ODE limits under different parameter regimes, including the neural ODE described in \cite{CRBD2018,TvG2018,HR2018} as a special case.

 We consider  a setup which is consistent with Scaling regime \ref{hypothesis.1} and $\delta^{(L)} =L^{-\alpha}$ for some $\alpha \ge 0$:
\begin{equation} \label{eq:resnet.v3}
\begin{aligned}
h_0^{(L)} &= x, \\
h^{(L)}_{k+1} &= h^{(L)}_{k} + L^{-\alpha}
\,\sigma_d\left(A^{(L)}_k h^{(L)}_{k}+ b^{(L)}_k\right),
\end{aligned}
\end{equation}
with
\begin{eqnarray*}
A^{(L)}_k = L^{-\beta}\overbar{A}_{k/L},\quad b^{(L)}_k = L^{-\beta}\overline{b}_{k/L}.
\end{eqnarray*}

\noindent We focus our analysis on smooth activation functions.
\begin{ass}[Activation function]\label{ass:activation}
The activation function $\sigma$ satisfies $\sigma \in \CC^3(\mathbb{R},\mathbb{R})$, $\sigma(0)=0$, $\sigma^{\prime}(0)=1$, and has a bounded third derivative  $\sigma^{\prime\prime\prime}$.
\end{ass}
Most smooth activation functions, including $\tanh$, satisfy this condition. 
The boundedness of the third derivative $\sigma^{\prime\prime\prime}$ may be relaxed to an exponential growth condition \cite{peluchetti2020}.

As observed in the numerical experiments, non-smooth activation functions such as $\relu$ lead to a different scaling regime to that of smooth functions.

We now describe ODE limits under Scaling regime~\ref{hypothesis.1}. Let $\hbarr^{(L)}:[0,1]\to \mathbb{R}^d$ be a continuous-time extension of the hidden states $h^{x,(L)}_k$:
\begin{eqnarray}\label{eq:interpolation3}
\hbarr^{(L)}_t \coloneqq h_k^{x,(L)} \one_{\frac{k}{L}\leq t <\frac{k+1}{L}}, \quad k =0,1,\ldots,L.
\end{eqnarray}
\begin{thm}[ODE limits under Scaling regime~\ref{hypothesis.1}]\label{thm:H1} Under Assumption~\ref{ass:activation} and $\sigma$ Lipschitz, 
\begin{itemize}
    \item \underline{Neural ODE limit \cite[Lemma 4.6]{TvG2018}:} If $\alpha=1$ and $\beta=0$ and we further assume that $\overbar{A} \in \HH^1\hspace{-3pt} \left( [0,1], \R^{d\times d} \right)$ and $\overline{b} \in \HH^1\hspace{-3pt}\left( [0,1], \R^{d} \right)$, then the interpolated hidden state dynamics~\eqref{eq:interpolation3} converge to the solution of the neural ODE
   \begin{eqnarray} \label{eq:limit-node}
  \frac{\dd H_t}{\dd t} = \sigma(\overbar{A}_t H_t+\overline{b}_t),\qquad H_0=x,
   \end{eqnarray}
   in the sense that  $\lim_{L \rightarrow \infty} \sup_{0 \leq t \leq 1}\|H_t-\hbarr_t\|=0$.
    \item \underline{A different ODE limit:}
    If $\alpha+\beta=1$ and $\beta>0$, and there exist $M>0$ and $\kappa>0$ such that $\forall s,t\in[0,1]$, $\norm{\overbar{A}_t-\overbar{A}_s} +\norm{\overline{b}_t-\overline{b}_s} \leq M |t-s|^{\kappa/2}$, then the interpolated hidden state dynamics~\eqref{eq:interpolation3} converge to the solution of the following ODE
   \begin{eqnarray}\label{eq:limit2}
  \frac{\dd H_t}{\dd t} = \overbar{A}_t H_t+\overline{b}_t,\qquad H_0=x,
   \end{eqnarray}
     in the sense that  $\lim_{L \rightarrow \infty} \sup_{0 \leq t \leq 1}\|H_t-\hbarr_t\|=0$.
\end{itemize}
\end{thm}

\subsection{Scaling regime \ref{hypothesis.2}}

Let $(\Omega, \mathcal{F}, \mathbb{F}, \mathbb{P})$ be a probability space with a $\mathbb{P}$-complete filtration $\mathbb{F}=(\mathcal{F}_t)_{t\geq 0}$. Let  $(B^A_t)_{t\ge 0}$, resp. $(B^b_t)_{t\ge 0}$, be $d\times d$-dimensional, resp. $d$-dimensional, $\mathbb{F}$-standard uncorrelated Brownian motions.  We consider 
a setup which is consistent with Section \ref{sec:experiments} where the noise part comes from the increment of some stochastic process and $\delta^{(L)} =L^{-\alpha}$ for some $\alpha \ge 0$:
\begin{equation} \label{eq:resnet.v2}
\begin{aligned}
h_0^{(L)} &= x, \\
h^{(L)}_{k+1} &= h^{(L)}_{k} + L^{-\alpha}
\,\sigma_d\left(A^{(L)}_k h^{(L)}_{k}+ b^{(L)}_k\right),
\end{aligned}
\end{equation}
with
\begin{eqnarray*}
A^{(L)}_k &=& L^{-\beta}\overbar{A}_{k/L} + \left(W^A_{(k+1)/L}-W^A_{k/L}\right),\\
b^{(L)}_k &=& L^{-\beta}\overline{b}_{k/L} + \left(W^b_{(k+1)/L}-W^b_{k/L}\right),
\end{eqnarray*}
where $(W^A_t)_{t\in [0,1]}$ and $(W^b_t)_{t\in[0,1]}$ are Itô processes \cite{revuz2013continuous}  adapted to $\mathbb{F}$ and can be written in the form:
\begin{equation}\label{eq:ito_process}
\begin{aligned}
\left(\dd W_t^A\right)_{ij} &= \left(U^A_t\right)_{ij} \dd t + \sum_{k,l=1}^d \left(q_t^A\right)_{ijkl} \left(\dd B^{A}_t\right)_{kl} \quad \mbox{for }\, i,j=1, \ldots, d ,\\
\dd W_t^b &= U^b_t \dd t +q^b_t \,\dd B^b_t,
\end{aligned}
\end{equation}
with $W_0^A = 0$, $W_0^b = 0$, $q_t^A\in \mathbb{R}^{d,\otimes4}$ and $q_t^b\in \mathbb{R}^{d\times d}$ for $t\in [0,1]$.
We use the following notation for the quadratic variation of $W^A$ and $W^b$:\begin{align}
\big[W^A\big]_t = \int_0^t \Sigma^A_u\, \dd u,\quad \big[W^b\big]_t = \int_0^t \Sigma^b_u\, \dd u, \label{eq.qv}
\end{align}
where $\Sigma^A$ and $\Sigma^b$ are bounded processes with values respectively in $\mathbb{R}^{d,\,\otimes 4}$ and $\mathbb{R}^{d \times d}$.
{ From \eqref{eq:ito_process} and \eqref{eq.qv}, we have the quadratic variation process as follows:
\begin{equation} \label{def:covariance-tensors}
\left(\Sigma_t^A\right)_{i_1 j_1 i_2 j_2} \coloneqq \sum_{k,l=1}^d \left(q_t^A\right)_{i_1 j_1 kl} \, \left(q_t^A\right)_{i_2 j_2 kl}, \quad \mbox{for } i_1, j_1, i_2, j_2 = 1, \ldots, d, \quad
\Sigma_t^b \coloneqq q_t^b \left(q_t^b\right)^{\top}.
\end{equation}
Here 
$(U_t^A)_{t\geq 0}$, $(U_t^b)_{t\geq 0}$, $(\Sigma_t^A)_{t\geq 0}$ and $(\Sigma_t^b)_{t\geq 0}$ are progressively measurable processes that satisfy the following conditions.
\begin{ass}[Regularity of the Ito processes $(W^A,W^b)$ and continuous functions $(\overbar{A},\overline{b})$] \label{ass:ito} We assume:
\begin{itemize}
    \item[(i)] There exists a constant $C_1>0$ such that almost surely
\begin{eqnarray}\label{eq:bound_ito}
\sup_{0 \leq t \leq 1}\norm{U_t^A} + \sup_{0 \leq t \leq 1} \norm{U_t^b} + \sup_{0 \leq t \leq 1} \norm{\Sigma_t^A} + \sup_{0 \leq t \leq 1} \norm{\Sigma_t^b} \leq C_1.
\end{eqnarray}
\item[(ii)]  There exist $M>0$ and $\kappa>0$ such that $\forall s,t\in[0,1]$ almost surely
 \begin{eqnarray}\label{eq:continuity_ito}
\norm{U_t^A-U^A_s}^2 + \norm{U_t^b-U^b_s}^2+ \norm{\Sigma_t^A-\Sigma^A_s}^2 +\norm{\Sigma_t^b-\Sigma^b_s}^2  \leq M |t-s|^{\kappa},
 \end{eqnarray}
 and  \begin{eqnarray}\label{eq:continuity_func}
\norm{\overbar{A}_t-\overbar{A}_s}^2 +\norm{\overline{b}_t-\overline{b}_s}^2 \leq M |t-s|^{\kappa}.
  \end{eqnarray}
\end{itemize}
\end{ass}
\noindent Note that~\eqref{eq:bound_ito} implies that $(U^A,U^B,\Sigma^A,\Sigma^B)$ are almost surely uniformly bounded and \eqref{eq:continuity_ito} implies that $(U^A,U^B,\Sigma^A,\Sigma^B)$ are almost surely H\"older continuous with exponent $\kappa/2$.

\begin{lem}[Uniform integrability]
Under Assumption \ref{ass:ito} (i), we have
\begin{equation}\label{eq:uniform_W}
\mathbb{E}\left[\sup_{0\leq s \leq 1}\norm{W^A_s}^{p_0}\right] \, \lor \, \mathbb{E}\left[\sup_{0\leq s \leq 1}\norm{W^b_s}^{p_0}\right]<\infty,
\end{equation}
for any $p_0> 1$.
\end{lem}
\begin{proof}
By Minkowski's inequality and Assumption \ref{ass:ito}-(i),
\begin{eqnarray*}
\mathbb{E}\left[\sup_{0\leq s \leq 1}\norm{W^A_s}^{p_0}\right] &\leq& 2^{p_0-1} \mathbb{E}\left[\sup_{0\leq s \leq 1}\norm{\int_0^s U_t^A \dd t}^{p_0} \right]  \\
&+& 2^{p_0-1} \mathbb{E}\left[\sup_{0\leq s \leq 1}\norm{\left(\int_0^s \sum_{k,l=1}^d \left(q_t^A\right)_{ijkl} \left(\dd B^{A}_t\right)_{kl}\right)_{i,j}}^{p_0} \right] \\
&\leq& 2^{p_0-1}C_1^{p_0} + 2^{p_0-1} \mathbb{E}\left[\sup_{0\leq s \leq 1}\norm{\left(\int_0^s \sum_{k,l=1}^d \left(q_t^A\right)_{ijkl} \left(\dd B^{A}_t\right)_{kl}\right)_{i,j}}^{p_0} \right] 
\end{eqnarray*}
By the Burkholder-Davis-Gundy inequality and Assumption \ref{ass:ito} $(i)$,
\begin{eqnarray*}
\mathbb{E}\left[\sup_{0\leq s \leq 1}\norm{\left(\int_0^s \sum_{k,l=1}^d \left(q_t^A\right)_{ijkl} \left(\dd B^{A}_t\right)_{kl}\right)_{i,j}}^{p_0} \right]
\leq C_{p_0}\mathbb{E}\left[\left(\int_0^1  \Sigma_u^A \dd u \right)^{p_0/2}\right] \leq C_{p_0}\,C_1^{p_0/2}.
\end{eqnarray*}
Combining the two inequalities above, we get
$\mathbb{E}\left[\sup_{0\leq s \leq 1}\norm{W^A_s}^{p_0}\right]<\infty$. Similarly $\mathbb{E}\left[\sup_{0\leq s \leq 1}\norm{W^b_s}^{p_0}\right]<\infty$ holds.
\end{proof}
}

\noindent Write $Q:[0,1]\times\mathbb{R}^{d} \rightarrow \mathbb{R}^d$, where each component $Q_i$ is defined, for $i=1,\ldots,d$, as
\begin{equation} \label{eq:Q_thm}
    Q_i(t,x)\coloneqq \sum_{j,k=1}^d x_j x_k \left(\Sigma^A_t\right)_ {ijik} + \Sigma^b_{t,ii}.
\end{equation}


\noindent Let $\hbarr^{(L)}:[0,1]\to \mathbb{R}^d$ be a continuous-time extension of the hidden states $h^{(L)}_k$:
\begin{eqnarray}\label{eq:interpolation2}
\hbarr^{(L)}_t \coloneqq h_k^{x,(L)} \one_{\frac{k}{L}\leq t <\frac{k+1}{L}}, \quad k =0,1,\ldots,L.
\end{eqnarray}

\begin{ass}[Uniform integrability]\label{ass:strong}
There exist $p_1 >4$ and a constant $C_0$ such that for all $L$,
 \begin{eqnarray}\label{ass:boundedness}
 \mathbb{E}\left[\sup_{0 \leq t \leq 1}\norm{\hbarr^{(L)}_t}^{p_1} \right]  \leq C_0.
 \end{eqnarray}
 \end{ass}
 Assumption \ref{ass:strong} is standard in the convergence of approximation schemes for SDEs \cite{higham2002strong}. 
In practice, condition~\eqref{ass:boundedness} is  guaranteed throughout the training as both the inputs and the outputs of the network are bounded. 

Let us now describe the intuition behind the deep network limit when $\beta>0$. Denote $t_k=k/L$ and define for $s\in [t_k,t_{k+1})$:
\begin{align*}
\widetilde{M}^{(L)}_{k, s} \coloneqq \left(W^A_s- W^A_{t_k}\right) h^{(L)}_{k}+\left(W^b_s- W^b_{t_k}\right) +L^{1-\beta} \overbar{A}_{t_k} h^{(L)}_{k}(s-t_k)+L^{1-\beta} \overline{b}_{t_k}(s-t_k).
\end{align*}
Using Itô's formula~\cite{ito1944} to $\sigma \big( \widetilde{M}^{(L)}_{k, s} \big)$ for $s\in [t_k,t_{k+1})$, we obtain the following approximation
\begin{equation} \label{eq:euler_maruyama3} 
    h^{(L)}_{k+1}-h^{(L)}_k = \delta^{(L)} \sigma \hspace{-1pt}\left( \widetilde{M}_{k, t_{k+1}}^{(L)}\right)  \simeq D_1 + D_2 + D_3,
\end{equation}
where
\begin{align*}
D_1 &\coloneqq L^{-\alpha}\left(\left(W^A_{t_{k+1}}-W^A_{t_{k}}\right)h_k^{(L)}+\left(W^b_{t_{k+1}}-W^b_{t_{k}}\right)\right), \\
D_2 &\coloneqq \frac{1}{2} L^{-\alpha} \sigma^{\prime \prime}(0)\,Q \hspace{-1pt} \left(t_k,{h^{(L)}_k}\right)({t_{k+1}-t_{k}}), \\
D_3 &\coloneqq L^{1-\beta-\alpha}\left(\overbar{A}_{t_k}h_{k}^{(L)}({t_{k+1}-t_{k}})+\overline{b}_{t_k}({t_{k+1}-t_{k}})\right).
\end{align*}
We observe from $D_1$ that~\eqref{eq:euler_maruyama3} admits a diffusive limit only when $\alpha=0$. In this case, we see that $D_2$ and $D_3$ do not explode only when  $\beta \geq 1$, corresponding to a stochastic differential equation (SDE) limit that is diffusive. 
Another case where we obtain a non-trivial limit is when $\alpha>0$ and $\alpha+\beta=1$, which leads to an ODE limit.

We now provide a precise mathematical description of the different scaling limits of $\hbarr^{(L)}$ for various values of $\alpha$ and $\beta$, using the concept of uniform convergence in $L^2$, also known as strong convergence. For a general exponent $p\geq 1$, we have the following definition.
\begin{defn}[Uniform convergence in $L^p$]\label{def:unif-conv-l2}
Let $p\geq 1$ and $\MM$ be the class of random functions $X:[0,1]\times \Omega\to\R^d$ such that 
$$ \mathbb{E}\left[\mathop{\sup}_{t\in [0,1]} \norm{ X(t) }^p\right] <\infty. $$
  We say that a sequence $(X^{(L)})_{L\in\N}\subset\MM$ \emph{converges uniformly in $L^p$} to $X^*\in\MM$ if
\begin{equation}\label{eq:formula-conv}
    \lim_{L \to \infty}\mathbb{E}\left[\sup_{0\leq t\leq 1}\norm{X^{(L)}_t - {X}^*_t}^p\right] = 0.
\end{equation}
\end{defn}

We now show that Scaling regime~\ref{hypothesis.2} together with a smooth activation function lead to an ODE limit (which is different from the neural ODE) or a stochastic differential equation (SDE) depending on the values of $\alpha$ and $\beta$.
\begin{thm}[ODE limit under Scaling regime~\ref{hypothesis.2}]\label{thm:H2-ode} Under Assumptions~\ref{ass:activation}, \ref{ass:ito}, and~\ref{ass:strong}, 
if $\alpha>0$, $\beta>0$ and $\alpha+\beta=1$, then the interpolated hidden state dynamics~\eqref{eq:interpolation2} converge uniformly in $L^2$ to the solution to the ODE
   \begin{eqnarray}\label{eq:limit5}
  \frac{\dd H_t}{\dd t} = \overbar{A}_t H_t+\overline{b}_t,\qquad H_0=x.
   \end{eqnarray}
\end{thm}
In particular, this implies the convergence of the hidden state process 
for any typical initialization (i.e almost surely with respect to the initialization). Note that in Theorem~\ref{thm:H2-ode}, the limit~\eqref{eq:limit5} defines a linear input-output map behaving like a linear network~\cite{ACGH2019}. This is different from the neural ODE~\eqref{eq:neural_ode_limit}, where the activation function $\sigma$ appears in the limit.



\begin{thm}[SDE limit under Scaling regime~\ref{hypothesis.2}]\label{thm:H2} Let Assumptions~\ref{ass:activation}, \ref{ass:ito} and~\ref{ass:strong} hold and let $\alpha=0$ and $\beta \ge 1$. Denote $H$ as the solution to the SDE
\begin{equation} \label{eq:limit1}
   \dd H_t = \dd W^A_t H_t +\dd W^b_t + \frac{1}{2}\sigma^{\prime \prime}(0)Q(t,H_t)\,\dd t +\one_{\beta=1}(\overbar{A}_t H_t+\overline{b}_t)\,\dd t,
\end{equation}
with initial condition $H_0=x$.
If there exist $p_2>2$  such that $\mathbb{E}\left[\sup_{0 \leq t \leq 1}\|H_t\|^{p_2} \right]  <\infty $, then the interpolated hidden state dynamics~\eqref{eq:interpolation2} converge uniformly in $L^2$ to the solution of~\eqref{eq:limit1}.
\end{thm}
The proofs of Theorem \ref{thm:H2} is given in Section \ref{sec:proofs}. And the proof of Theorem \ref{thm:H2-ode} follows similar ideas. In particular, $D_1$ and $D_2$ vanish in the limit when $\alpha>0$ in \eqref{eq:euler_maruyama3}.

Interestingly, when the activation function $\sigma$ is smooth, all limits in both Theorems~\ref{thm:H2-ode} and~\ref{thm:H2} depend on the activation only through $\sigma^{\prime}(0)$ (assumed to be $1$ for simplicity) and $\sigma^{\prime\prime}(0)$. In contrast to the behavior of the neural ODE limit~\eqref{eq:neural_ode_limit}, the characteristics of $\sigma$ away from $0$ are not relevant to the limit. In addition, our proofs rely on the smoothness of $\sigma$ at $0$. If the activation function is not differentiable at $0$, then a different limit should be expected.

The case $\overbar{A}\equiv 0$, $\overline{b}\equiv 0$, $\alpha=0$,  and $\beta=1$ in Theorem~\ref{thm:H2} is considered in~\cite{peluchetti2020}, under the additional assumption that $W^A$ and $W^b$ are Brownian motions with constant drift. We consider a more general setup, where we introduce nonzero terms $\overbar{A}$ and $\overline{b}$ and we allow $W^A$ and $W^b$ to be arbitrary Itô processes. Moreover,~\cite{peluchetti2020} prove weak convergence, which corresponds to convergence of quantities averaged across many trained networks with random independent initializations, whereas in practice, the training is done only once. Thus, the strong convergence, shown in Theorems~\ref{thm:H1} and~\ref{thm:H2}, is a more relevant notion for studying the asymptotic behavior of deep neural networks.

Although the ResNet dynamics \eqref{eq:resnet.v2} is {\it not} expressed as an Euler scheme of a (ordinary or stochastic) differential equation, we nevertheless show strong convergence to a limitng  ODE (in the case of Theorem \ref{thm:H2-ode}) or SDE (in the case of Theorem \ref{thm:H2}), using techniques  inspired by \cite{higham2002strong}. The  challenge is to bound the difference between the ResNet dynamics and the Euler scheme of the limiting SDE. It is worth mentioning that the results in \cite{higham2002strong} hold for a class of time-homogeneous (Markov) diffusion processes whereas our result holds for  It\^o processes with bounded coefficients. This distinction is important for training neural networks since the ``diffusion'' assumption involves the distribution of the hidden state dynamics which can never be tested in practice. We can only verify the smoothness of the hidden state dynamics as detailed in Section \ref{sec:experiments}. In addition, we also relax one technical condition assumed in \cite{higham2002strong}, which is difficult to verify in practice. See Remark \ref{rmk:uniform_integrability}.

Note that we assume that the Ito processes $W^A$ and $W^b$ are driven by \textbf{uncorrelated} Brownian motions $B^A$ and $B^b$. This assumption might look strong, but we pose it for ease of exposition: assuming a generic correlation structure between $B^A$ and $B^b$ would only a \textit{cross-term} in the definition of $Q$.

\subsection{Link with numerical experiments}

Let us now discuss how the analysis above sheds light on the numerical results in Section \ref{sec:fully_connected} and Section~\ref{sec:classification}. 


Figure~\ref{fig:scaling_tanh_shared_delta} shows that $\beta=0.2$ and $\alpha=0.7$ for the synthetic dataset with fully-connected layers and $\tanh$ activation function. 
This corresponds to the assumptions of Theorem~\ref{thm:H1} with the ODE limit~\eqref{eq:limit5}. This is also consistent with the estimated decomposition in Figure~\ref{fig:hypothesis_tanh_shared_delta} (right) where the noise part is negligible. 

Regarding $\relu$ activation with fully-connected layers, we observe that $\beta+\alpha = 0.9$ from Figure~\ref{fig:scaling_relu_scalar_delta} (left).
Since ReLU is homogeneous of degree 1 (see Remark~\ref{rmk:relu}), $|\delta^{(L)}|$ can be moved inside $\sigma$, so without loss of generality we can assume $\alpha=0$ and $\beta=0.9$. 
If we replace the $\relu$ function by a smooth version $\sigma^\epsilon$, then the limit is described by the stochastic differential equation~\eqref{eq:limit1}.
The $\relu$ case would then correspond to a limit of this equation as $\epsilon \to 0$. The existence of such a limit is, however, nontrivial and left for future work.

From the experiments with convolutional architectures, we observe that the maximum norm  (Figure~\ref{fig:cifar_scaling_with_depth}), the scaled norm of the increments, and the root sum of squares (Figure~\ref{fig:cifar_hyp_with_depth}) are upper bounded as the number of layers $L$ increases. This indicates that the weights fall into a sparse regime when $L$ is large. In this case, there is no continuous ODE or SDE limit and Scaling regimes~\ref{hypothesis.1} and~\ref{hypothesis.2} both fail.  

\subsection{Detailed proofs}\label{sec:proofs}

\subsubsection{Proof of Theorem \ref{thm:H1}}
It suffices to prove the second case with limit \eqref{eq:limit2}.  First we show that there exists $C_{\infty}>0$ such that
\begin{eqnarray}\label{ode_h_bound}
\sup_L \max_{0\leq k \leq L}\norm{h_k^{(L)}} \leq C_{\infty}.
\end{eqnarray}
Indeed, denote $C_{\sigma}$ as the Lipschitz constant of $\sigma$. Then
\begin{eqnarray*}
\norm{h_{k+1}^{(L)}-h_{k}^{(L)}}\leq \frac{C_{\sigma}}{L} \norm{\overbar{A}_{t_k}h_{k}^{(L)}+\overline{b}_{t_k}} \leq 
\frac{C_{\sigma}}{L} \left(A_{\max}+b_{\max}\right) \left(\norm{h_{k}^{(L)}} +1\right),
\end{eqnarray*}
where $A_{\max}:=\sup_{0 \leq t \leq 1}\norm{\overbar{A}_{t}}<\infty$, $b_{\max}:=\sup_{0 \leq t \leq 1}\norm{\overline{b}_{t}}<\infty$, and $C_{\max} \coloneqq A_{\max} + b_{\max}$. Hence 
\begin{eqnarray*}
\norm{h_{k+1}^{(L)}} \leq \left(\frac{ C_{\sigma}C_{\max} }{L} + 1 \right) \norm{h_{k}^{(L)}} +\frac{C_{\sigma}C_{\max} }{L}.
\end{eqnarray*}
By induction:
\begin{eqnarray*}
\norm{h_{j}^{(L)}} & \leq& \norm{x}  \left(\frac{ C_{\sigma}C_{\max} }{L} + 1 \right)^j + \frac{ C_{\sigma}C_{\max} }{L} \sum_{i=1}^j\left(\frac{ C_{\sigma}C_{\max} }{L} +1 \right)^{i-1}\\
&\leq& \left( \norm{x} + C_{\sigma}C_{\max} \right) \left(\frac{C_{\sigma}C_{\max} }{L} + 1 \right)^{L}\\
&\rightarrow & \left( \norm{x} + C_{\sigma}C_{\max}\right) \exp\left( C_{\sigma}C_{\max} \right) \quad {\rm as } \,\,L \rightarrow \infty.
\end{eqnarray*}
Hence \eqref{ode_h_bound} holds.

\noindent Denote $\Delta h_{k}^{(L)}:= h_{k+1}^{(L)}-h_{k}^{(L)}$ and $M^{(L)}_k(h) := \overbar{A}_{t_k}h+\overline{b}_{t_k}$. From~\eqref{eq:euler_maruyama3} we have \[
\Delta h_k^{(L)} := h_{k+1}^{(L)}-h_k^{(L)} = L^{-\alpha}\sigma \left(  L^{-\beta}M_k^{(L)} \left( h_k^{(L)} \right) \right).
\]
Denote as well $\Delta h_{k,i}^{(L)}$ and $M_{k, i}^{(L)}$ the $i$-th element of $\Delta h_k^{(L)}$ and $M_{k}^{(L)}$, respectively. 
Applying a third-order Taylor expansion of $\sigma$ around $0$ with the help of Assumption~\ref{ass:activation}, for $i=1,2,\ldots,d$, we get
\begin{align}
&\Delta h_{k,i}^{(L)} = L^{-\alpha} \sigma_d \left(L^{-\beta} M_{k,i}^{(L)}\left( h_k^{(L)} \right) \right) \nonumber \\
&= L^{-\beta-\alpha} M_{k,i}^{(L)}\left( h_k^{(L)} \right) + \frac{1}{2}\sigma^{\prime \prime}(0)L^{-2\beta-\alpha} \left( M_{k,i}^{(L)} \left( h_k^{(L)} \right) \right)^2  + \frac{1}{6}\sigma^{\prime\prime\prime}\left( \xi^{(L)}_{k,i} \right) L^{-3\beta-\alpha}\left(M_{k,i}^{(L)} \left( h_k^{(L)} \right) \right)^3 \nonumber\\
&= L^{-1} M_{k,i}^{(L)} \left( h_k^{(L)} \right) + \frac{1}{2}\sigma^{\prime \prime}(0)L^{-\beta-1} \left(M_{k,i}^{(L)}(h_k^{(L)}) \right)^2  + \frac{1}{6}\sigma^{\prime\prime\prime} \left( \xi^{(L)}_{k,i} \right) L^{-2\beta-1}\left(M_{k,i}^{(L)} \left( h_k^{(L)} \right) \right)^3 \label{eq:delta_h}
\end{align}
with $\abs{ \xi^{(L)}_{k,i} } \leq L^{-\beta}\abs{ \overbar{A}_{t_k}h_k^{(L)}+\overline{b}_{t_k} }_i$. The last equation holds since $\alpha+\beta=1$. Denote $t_k=k/L$ for $k=0,1,\ldots,L$ as the uniform partition of the interval $[0,1]$. For $t\in (t_k,t_{k+1}]$, define $\widetilde{H}_0^{(L)} \coloneqq x = h_0^{(L)}$ and 
\begin{align*}
\widetilde{H}_t^{(L)} &\coloneqq h_k^{(L)} + \left(t-t_k\right) M_{k,i}^{(L)}\left( h_k^{(L)} \right) + \frac{1}{2}\sigma^{\prime \prime}(0)L^{-\beta-1} \left(M_{k,i}^{(L)} \left( h_k^{(L)} \right) \right)^2 \\
&\quad + \frac{1}{6}\sigma^{\prime\prime\prime} \left( \xi^{(L)}_{k,i} \right) L^{-2\beta-1} \left( M_{k,i}^{(L)} \left(h_k^{(L)}\right) \right)^3.
\end{align*}
Then we have $\widetilde{H}_{t_{k+1}}^{(L)}=h_k^{(L)} + \Delta h_k^{(L)} = h_{k+1}^{(L)}$ for all $k=0,1,\ldots,L-1$. Recall $(H_t)_{t\in\left[0,1\right]}$ the solution to the ODE \eqref{eq:limit2}. Denote $d^{(L)}_{k}(t) \coloneqq H_t - \widetilde{H}_t^{(L)}$ for $t\in(t_k,t_{k+1}]$ and define the errors \[
e^{(L),1}_{k} \coloneqq \sup_{t_k < t \leq t_{k+1}}\norm{ \widetilde{H}_t^{(L)} - h_k^{(L)}} \quad \mbox{and} \quad e^{(L),2}_{k} \coloneqq \sup_{t_k < t \leq t_{k+1}}\norm{d^{(L)}_{k}(t)}.
\]
We first bound $e^{(L),1}_{k}$. Note that by definition:
\begin{align}
e^{(L),1}_{k} &\leq \norm{ M_k^{(L),i}\left( h_k^{(L)} \right) } L^{-1} + \frac{1}{2}\sigma^{\prime \prime}(0)L^{-\beta-1} \left(M_k^{(L),i} \left( h_k^{(L)} \right) \right)^2  + \frac{1}{6} c_0 L^{-2\beta-1} \abs{ M_k^{(L),i} \left( h_k^{(L)} \right) }^3\nonumber\\
&\leq D_{\infty} L^{-1}, \label{eq:e_1}
\end{align}
where $D_{\infty} := A_{\max} C_{\infty} +b_{\max} +\frac{1}{2}\sigma^{\prime \prime}(0)(A_{\max} C_{\infty} +b_{\max} )^2 +\frac{1}{6}c_0(A_{\max} C_{\infty} +b_{\max} )^2$.
Therefore we have
\[
\lim_{L \rightarrow \infty}\sup_{0 \leq k < L} e^{(L),1}_{k} = 0.
\]
Next, we bound $e^{(L),2}_{k}$. For $t\in (t_{k+1},t_{k+2}]$,
\begin{align}
d^{(L)}_{k+1}(t) &= d^{(L)}_{k}(t_{k+1}) - \left( t-t_{k+1} \right) M_{k+1}^{(L),i} \left( h_{k+1}^{(L)} \right) + \int_{t_{k+1}}^{t} \left( \overbar{A}_s H_s + \overline{b}_s \right) \dd s \nonumber \\
&\quad - \frac{1}{2}\sigma^{\prime \prime}(0) L^{-\beta-1} \left( M_{k+1}^{(L),i} \left( h_{k+1}^{(L)} \right) \right)^2 - \frac{1}{6}\sigma^{\prime\prime\prime}\left( \xi_{k+1}^i \right) L^{-2\beta-1}\left(M_{k+1}^{(L),i} \left( h_{k+1}^{(L)} \right) \right)^3. \label{eq:d_k}
\end{align}
Denote $c_0 := \sup_{x\in \mathbb{R}}|\sigma^{\prime\prime\prime}(x)|<\infty$, hence from \eqref{eq:delta_h} and  \eqref{eq:d_k},
\begin{eqnarray*}
e^{(L),2}_{k+1} &\leq& e^{(L),2}_{k} +\sup_{t_{k+1} < t \leq t_{k+2}} \norm{ \int_{t_{k+1}}^{t}\left(\left(\overbar{A}_sH_s+\overline{b}_s \right)- \left( \overbar{A}_{t_{k+1}}h_{k+1}^{(L)}+\overline{b}_{t_{k+1}}\right) \right) \dd s}\\
&&+  \frac{1}{2}\left|\sigma^{\prime \prime}(0)\right|L^{-\beta-1} \norm{M_{k+1}^{(L)}(h_{k+1}^{(L)}) }^2  + \frac{1}{6}c_0 L^{-2\beta-1}\norm{M_{k+1}^{(L)}(h_{k+1}^{(L)})}^3.
\end{eqnarray*}
Denote $H_{\max} \coloneqq \sup_{0\leq t \leq 1}\norm{H_t}<\infty$. Then,
\begin{eqnarray*}
 \EE_k^{(L)} &\coloneqq& \sup_{t_{k+1} < t \leq t_{k+2}} \norm{ \int_{t_{k+1}}^{t}\left(\left(\overbar{A}_sH_s+\overline{b}_s \right)- \left( \overbar{A}_{t_{k+1}}h_{k+1}^{(L)}+\overline{b}_{t_{k+1}}\right) \right) \dd s}\\
 &\leq& \sup_{t_{k+1} < t \leq t_{k+2}}  \norm{ \int_{t_{k+1}}^{t}\left(\overline{b}_s - \overline{b}_{t_{k+1}} \right) \dd s} +\sup_{t_{k+1} < t \leq t_{k+2}}  \norm{ \int_{t_{k+1}}^{t}\left(\overbar{A}_s - \overbar{A}_{t_{k+1}} \right)H_s \dd s} \\
 && + \sup_{t_{k+1} < t \leq t_{k+2}} \norm{ \overbar{A}_{t_{k+1}}\int_{t_{k+1}}^{t}\left(H_s - h_{k+1}^{(L)} \right) \dd s}.
\end{eqnarray*}
Hence, we deduce
\begin{eqnarray*}
 \EE_k^{(L)} &\leq& \int_{t_{k+1}}^{t_{k+2}}\norm{\overline{b}_s - \overline{b}_{t_{k+1}} } \dd s  + H_{\max}  \int_{t_{k+1}}^{t_{k+2}}\norm{\overbar{A}_s - \overbar{A}_{t_{k+1}} } \dd s  + A_{\max} \int_{t_{k+1}}^{t_{k+2}} \norm{H_s - h_{k+1}^{(L)}}\dd s\\
 &\leq & M(1+H_{\max}) \int_{t_{k+1}}^{t_{k+2}}|s-t_{t_{k+1}}|^{\kappa/2} \dd s + A_{\max} \int_{t_{k+1}}^{t_{k+2}} \norm{H_s - h_{k+1}^{(L)}}\dd s\\
 &\leq &\frac{M}{1+\kappa/2}\left(1+H_{\max}\right)L^{-(1+\kappa/2)} + \sup_{0 \leq t \leq 1}\|\overbar{A}_{t}\| L^{-1}  \left(D_{\infty}L^{-1}+e_{k+1}^{(L),2}\right).
\end{eqnarray*}
The last equation holds by \eqref{eq:e_1}. Then, we have for $L>A_{\max}$,
\begin{align}
\left(1- A_{\max} L^{-1}\right)e^{(L),2}_{k+1} &\leq e^{(L),2}_{k}+ \frac{M}{1+\kappa}\left(1+H_{\max} \right)L^{-(1+\kappa)} +  \frac{1}{2}\sigma^{\prime \prime}(0)L^{-(\beta+1)}\left(A_{\max}C_{\infty}+b_{\max}\right)^2 \nonumber\\
&\quad + \frac{1}{6}c_0 L^{-(2\beta+1)}\left(A_{\max}C_{\infty}+b_{\max}\right)^3+A_{\max}D_{\infty}L^{-2}\nonumber\\
&\leq e_k^{(L),2} + L^{-(1+\nu)} C_2,\label{eq:e_k}
\end{align}
with $\nu := \min\{\kappa,\beta,1\}>0$ and $C_2$ a constant independent of $k$ and $L$. 
Finally, when $L \ge G^{1/\gamma} + 2 A_{\max}$ we have from \eqref{eq:e_k}:
\begin{eqnarray}
 e_0^{(L),2} &\leq& \frac{L^{-(1+\gamma)}G}{1-A_{\max}L^{-1}} \leq \frac{1}{L-A_{\max}}, \label{eq:e_0}
\end{eqnarray}
and for $k=0,\ldots,L-1$,
\begin{eqnarray}
e^{(L),2}_{k+1} &\leq&  \frac{1}{1- A_{\max} L^{-1}} \left(e^{(L),2}_{k} +L^{-(1+\gamma)}G\right)\nonumber\\
&\leq& \left(\frac{1}{1- A_{\max} L^{-1}} \right)^{k+1}e^{(L),2}_{0} +L^{-(1+\gamma)}G \frac{\left(\frac{1}{1- A_{\max} L^{-1}}\right)^{k+2}-1}{\left(\frac{1}{1- A_{\max}L^{-1}}\right)-1}\nonumber\\
&\leq& \exp\left(2A_{\max}\frac{k+1}{L}\right)\frac{1}{L-A_{\max}} +L^{-\gamma} \frac{G}{A_{\max}} \exp\left(2A_{\max}\frac{k+2}{L}\right) \label{eq:e_k6}.
\end{eqnarray}
\eqref{eq:e_k6} holds since $e_0^{(L),2}\leq \frac{1}{L-A_{\max}}$ and $\frac{1}{1-A_{\max}L^{-1}}< 1+2A_{\max}L^{-1} \leq \exp(2A_{\max}L^{-1})$ when $L>2A_{\max}$. Therefore, we conclude
\[
\lim_{L \rightarrow \infty}\sup_{0 \leq k < L} e^{(L),2}_{k} = 0.
\]

\subsubsection{Proof of Theorem \ref{thm:H2-ode}}
We provide a complete proof of Theorem \ref{thm:H2-ode} for the case $\alpha=0$  and $\beta=1$. Other cases follow similarly. When $\alpha=0$ and $\beta=1$, we define the targeted SDE limit for the discrete scheme \eqref{eq:resnet.v2} as follows:
\begin{eqnarray}\label{eq:limit_sde}
\dd H_t = \mu(t,H_t)\dd t + \dd V_t^A \, H_t + \dd V_t^b \,\,\,\,\, \mbox{for} \,\, t\in\left[0,1 \right], \,\,\,\, H_0=x,
\end{eqnarray}
in which
\begin{equation} \label{eq:case1_mean}
\begin{aligned}
\mu\left(t,h\right) &\coloneqq U_t^{A} \, h+\,U_t^{b}+ \overbar{A}_{t}  h + \overline{b}_{t}+\frac{1}{2}\sigma^{\prime \prime}(0)Q(t,h),\\
\dd V_t^A &\coloneqq \sum_{k,l=1}^d \left(q_t^A\right)_{ijkl} \left(\dd B_t^A\right)_{kl},\quad \dd V_t^b \coloneqq q_t^b \,\dd B_t^b,
\end{aligned}
\end{equation}
with $V_0^A=0$ and $V_0^b=0$. Here the quadratic variation process $\frac{1}{2}\sigma^{\prime\prime}(0)Q(t,h)$ is the \textit{Itô correction} term for the drift. On the one hand this correction term introduces non-linearity into the drift and makes the proof challenging. On the other hand, this term is the key for the convergence analysis. See \eqref{eq:bound14} and \eqref{eq:bound15}.

{\bf Euler-Maruyama scheme of the limiting SDE.} Denote $\Delta_L = 1/L$ as the sub-interval length and $t_k=k/L$, $k=0,1,\ldots,L$ as the uniform partition of the interval $[0,1]$. Further denote  $\Delta V^A_k = V^A_{t_{k+1}}-V^A_{t_k}$ and $\Delta V^b_k = V^b_{t_{k+1}}-V^b_{t_k}$  as the increment of the stochastic processes.  Define the Euler-Maruyama discretization scheme of the SDE \eqref{eq:limit_sde} as:
\begin{eqnarray}\label{eq:euler}
\hhat_{k+1}^{(L)}-\hhat_{k}^{(L)} \coloneqq \mu \hspace{-1pt} \left(t_k,\hhat_{k}^{(L)}\right)\Delta_L +\Delta V_k^A  \, \hhat_{k}^{(L)} + \Delta V_k^b,
\end{eqnarray}
and the one-step forward increment follows:
\begin{eqnarray}
f^{(L)}(k,h) \coloneqq \mu \left(t_k,h\right)\Delta_L +\Delta V_k^A  \, h + \Delta V_k^b.
\end{eqnarray}
Therefore \eqref{eq:euler} can be rewritten as $\hhat_{k+1}^{(L)} =\hhat_{k}^{(L)} + f^{(L)} \hspace{-2pt} \left(k,\hhat_{k}^{(L)}\right)$.

{\bf Continuous-time extension.}
Recall that we extend the scheme $\big\{ h^{(L)}_k : k=0, \ldots, L \big\}$ to a continuous-time process $\hbarr^{(L)}_t$ on $t\in[0,1]$ by a piecewise constant and right-continuous interpolation of $\{ h^{(L)}_k : k=0, \ldots, L-1 \}$:
\begin{eqnarray}\label{eq:cte}
\hbarr^{(L)}_t \coloneqq \sum_{k=0}^{L} h_k^{(L)} {\bf 1}_{t_{k}\leq t <t_{k+1}}.
\end{eqnarray}
We call $\hbarr^{(L)}_t$ the {\it continuous-time extension} (CTE) of $\{ h^{(L)}_k : k=0, \ldots, L-1 \}$. 

{\bf Continuous-time approximation.}
Denote 
\begin{eqnarray}
M^{(L)}_k(h) &\coloneqq & \left(\mu\left(t_k,h\right) -\frac{1}{2}\sigma^{\prime \prime}(0)Q(t_k,h)\right)\Delta_L+ \Delta V_k^A  \, h + \Delta V_k^b \nonumber \\
&=&\left(U^{A}_{t_k} \, h+\,U^{b}_{t_k}+ \overbar{A}_{t_k}  h + \overline{b}_{t_k}\right)\Delta_L+ \Delta V_k^A  \, h + \Delta V_k^b \nonumber \\
&=:&\widetilde{\mu}\left(t_k,h\right)\Delta_L+ \Delta V_k^A  \, h + \Delta V_k^b, \label{eq:def:MkL}
\end{eqnarray}
and from~\eqref{eq:euler_maruyama3} we thus have
$$\Delta h_k^{(L)} := h_{k+1}^{(L)}-h_k^{(L)} = \sigma \hspace{-1pt} \left(  M_k^{(L)} \hspace{-2pt} \left( h_k^{(L)} \right) \right).$$
Denote $\Delta h_{k,i}^{(L)}$ and $M_{k,i}^{(L)}$ the $i$-th element of $\Delta h_k^{(L)}$ and $M_{k}^{(L)}$, respectively. 
Applying a third-order Taylor expansion of $\sigma$ around $0$ with the help of Assumption~\ref{ass:activation}, for $i=1,2,\ldots,d$, we get
\begin{align*}
&\Delta h_{k,i}^{(L)} = \sigma \hspace{-1pt} \left(  M^{(L)}_{k,i} \hspace{-2pt} \left( h_k^{(L)} \right) \right) \\
&= M^{(L)}_{k,i} \hspace{-2pt} \left( h_k^{(L)} \right) + \frac{1}{2}\sigma^{\prime \prime}(0) M^{(L)}_{k,i} \hspace{-2pt} \left( h_k^{(L)} \right)^2 + \frac{1}{6} \sigma^{\prime \prime \prime} \hspace{-2pt} \left( \xi^{(L)}_{k,i} \right) M^{(L)}_{k,i} \hspace{-2pt} \left( h_k^{(L)} \right)^3 \\
&= \underbrace{{\mu}_i \hspace{-1pt} \left(t_{k},h_k^{(L)}\right)\Delta_L + \left( \Delta V_k^A  \, h_k^{(L)} \right)_i+ \left( \Delta V_k^b \right)_i}_{f_i^{(L)}\left( k, h_k^{(L)} \right)} + \underbrace{\frac{1}{2}\sigma^{\prime \prime} (0) \left( M_{k,i}^{(L)} \hspace{-2pt} \left( h_k^{(L)} \right)^2 - Q_i \hspace{-1pt} \left( t_k, h_k^{(L)} \right) \right) }_{N_{k,i}^{(L)}\left( h_k^{(L)} \right)} \\
&\quad \, + \frac{1}{6} \sigma^{\prime \prime \prime} \hspace{-2pt} \left( \xi^{(L)}_{k,i} \right) M^{(L)}_{k,i} \hspace{-2pt} \left( h_k^{(L)} \right)^3 \\
&= f^{(L)}_i \hspace{-2pt} \left( k,h_k^{(L)} \right) + N_{k,i}^{(L)} \hspace{-2pt} \left( h_k^{(L)} \right) 
+ \frac{1}{6}\sigma^{\prime \prime \prime} \hspace{-2pt} \left( \xi^{(L)}_{k,i} \right) M^{(L)}_{k,i} \hspace{-2pt} \left( h_k^{(L)} \right)^3 ,
\end{align*}
with $\abs{\xi^{(L)}_{k,i}} < \abs{ M^{(L)}_{k,i} \hspace{-2pt} \left( h_k^{(L)} \right) }$. The increment of the ResNet $\Delta h_{k,i}^{(L)}$ has two parts: the increment of the Euler-Maruyama scheme $ f^{(L)}_i\hspace{-2pt}\left( k,h_k^{(L)} \right)$ and the residual  
\begin{equation} \label{eq:residual-d}
D_{k,i}^{(L)} \hspace{-2pt} \left( h_k^{(L)} \right) \coloneqq \frac{1}{6}\sigma^{\prime \prime \prime} \hspace{-2pt} \left( \xi^{(L)}_{k,i} \right) M^{(L)}_{k,i} \hspace{-2pt} \left( h_k^{(L)} \right)^3 + N_{k,i}^{(L)} \hspace{-2pt} \left( h_k^{(L)} \right).
\end{equation} 
It is clear from here that the Euler-Maruyama scheme of the limiting SDE is different from the ResNet dynamics. Hence classical results on the convergence of discrete SDE schemes cannot be applied directly.

In our analysis it will be more natural to work with the following {\it continuous-time approximation} (CTA), defined as
\begin{eqnarray}\label{eq:cta}
\htilde^{(L)}_t \coloneqq h_0^{(L)} +\int_0^t \mu \hspace{-1pt} \left({ {t_{k_s}}}, \hbarr^{(L)}_s \right) \dd s +\int_0^t \left(\dd V^A_s \, \hbarr^{(L)}_s + \dd V^b_s \right) +\sum_{k < Lt} D_k^{(L)} \hspace{-2pt} \left( h_k^{(L)} \right),
\end{eqnarray}
where $D_k^{(L)}(h) = \left(D_{k,1}^{(L)}(h), \ldots, D_{k,d}^{(L)}(h) \right)^{\top}$ and $k_s$ is the integer for which $s\in[t_{k_s},t_{k_s+1})$ for a given $s\in[0,1)$.

Here $\htilde^{(L)}_t$ approximates the CTE~\eqref{eq:cte} with a continuous version, with interpolations both in time and in space, of the $f^{(L)}(k,h)$ part while the residual term $D_k^{(L)}(h)$ remains the same.
By design we have $\htilde^{(L)}_{t_k}=\hbarr^{L}_{t_k} = h^{(L)}_k$, that is, $\htilde^{(L)}_t $ and $\hbarr^{(L)}_t $ coincide with the discrete solution at grid points $t_k$, $k=0,1,\ldots,L-1$. This relationship is instrumental in order to control the error.
 
We will first study the difference between $\htilde$ and $h^{(L)}$, and then the difference between $\hbarr$ and $h^{(L)}$, in the supremum norm. The sum of the two will give a bound for the error of the discrete approximation.
 
 \subsubsection{Preliminary result}\label{sec:local_lip}


 \begin{lem}[Local Lipschitz condition and uniform integrability]\label{lemma:main} Under the assumptions from Theorem~\ref{thm:H2}, we have the folllowing results:
 \begin{enumerate}[label=(\roman*)]
     \item  For each $R>0$, there exists a constant $C_R$, depending only on $R$, such that almost surely we have
 \begin{eqnarray}\label{ass:local_lipschitz}
 \norm{\mu(t,x)-\mu(t,y)}^2  \leq C_R \norm{x-y}^2 , \,\, \forall t\in[0,1] \,\, \forall x,y \in \mathbb{R}^d \textit{ with } \norm{x} \lor \norm{y} \leq R,
 \end{eqnarray}
 where $\mu$  is defined in~\eqref{eq:case1_mean}.
 \item 
 There exist some constants $p>2$ and $C>0$ such that
 \begin{eqnarray}\label{ass:boundedness.v2}
 \mathbb{E}\left[\sup_{0 \leq t \leq 1} \norm{\htilde^{(L)}_t}^p \right] \lor \mathbb{E}\left[\sup_{0 \leq t \leq 1} \norm{H_t}^p \right] \leq C.
 \end{eqnarray}
 \end{enumerate}
 \end{lem}
 \begin{rmk}\label{rmk:uniform_integrability}
 Note that~\cite{higham2002strong} assumes the uniform integrability condition for $\htilde_t^{(L)}$ which is difficult to verify in practice.  Here we relax this condition by only assuming the uniform integrability condition for the ResNet dynamics $\{h_k^{(L)} : k=0,\ldots, L \}$, see Assumption~\ref{ass:strong}.  We can then prove~\eqref{ass:boundedness.v2} under Assumption~\ref{ass:strong} and some properties of the It\^o processes.
 \end{rmk}
 \begin{proof}[Proof of Lemma \ref{lemma:main}]
 First, there exists $C_2>0$ such that 
 \begin{eqnarray}
 \|Q(t,x)-Q(t,y)\| \leq C_2\norm{x-y}  \, \norm{x+y} \leq 2C_2 R \norm{x-y}, 
 \end{eqnarray}
 since $Q(t,x)$ is quadratic in $x$ and $\sup_{0 \leq t \leq 1}\|\Sigma_t^A\| \leq C_1$. Then,
 \begin{eqnarray*}
 \|\mu\left(t,x\right)-\mu\left(t,y\right) \|^2 &=& \left\|U_t^{A} \, (x-y)+ \overbar{A}_{t}  (x-y) + \frac{1}{2}\sigma^{\prime \prime}(0)\left( Q(t,x)-Q(t,y) \right) \right\|^2\\
 &\leq& \left(3  \, \max_{t\in[0,1]}\norm{U^A_t} +3 \, \max_{t\in[0,1]} \norm{\overbar{A}_t} + 3 \abs{\sigma^{\prime\prime}(0)} C_2^2R^2 \right) \norm{x-y}^2.
 \end{eqnarray*}
 Note that $\max_{t\in[0,1]} \norm{\overbar{A}_t} < \infty$ since $\overbar{A}\in \CC^0\left([0,1],\mathbb{R}^d\right)$ and $\max_{t\in[0,1]}\norm{U^A_t} < C_1$ almost surely according to \eqref{eq:bound_ito}, respectively. 
 Therefore \eqref{ass:local_lipschitz} holds by taking $C_R =3  \, \max_{t\in[0,1]}\norm{U^A_t} +3 \, \max_{t\in[0,1]} \norm{\overbar{A}_t} + 3 \abs{\sigma^{\prime\prime}(0)} C_2^2R^2$.
 
Thanks to the assumption in Theorem~\ref{thm:H2}, there exists a constant  $C_3>0$ such that  $\mathbb{E}\left[\sup_{0 \leq t \leq 1}\|H_t\|^{p_1} \right]  \leq C_3,$ then
 we only need to show that~\eqref{ass:boundedness.v2} holds for $\htilde$ for some $p>2$. To see this, let $k_s$ be the integer for which $s\in[t_{k_s},t_{k_s+1})$ for a given $s\in[0,1)$. Then
\begin{eqnarray*}
\hbarr^{(L)}_s-\htilde^{(L)}_s &=& h_{k_s}^{(L)} - \left( h_{k_s}^{(L)}+\int_{t_{k_s}}^s \mu \hspace{-1pt} \left( t_{k_r}, \hbarr^{(L)}_r \right) \dd r + \int_{t_{k_s}}^s \left(\dd V_r^A\hbarr_r^{(L)} + \dd V_r^b\right)\right)\\
&=&-\mu\hspace{-1pt} \left({ {t_{k_s}}},h_{k_s}^{(L)} \right) (s-t_{k_s}) - \left(V^A_s-V^A_{t_{k_s}}\right) \, h_{k_s}^{(L)}-\left(V^b_s-V^b_{t_{k_s}}\right).
\end{eqnarray*}
Hence, by the Minkowski inequality, 
\begin{align}
\norm{\hbarr^{(L)}_s-\htilde^{(L)}_s}^{p} &\leq 3^{p-1}\left( \norm{\mu\hspace{-1pt}\left({ {t_{k_s}}},h_{k_s}^{(L)} \right)}^{p} \left(\Delta_L\right)^{p} + \norm{h_{k_s}^{(L)}}^{p}  \, \norm{V^A_s-V^A_{t_{k_s}}}^{p} + \norm{V^b_s-V^b_{t_{k_s}}}^{p}\right)\nonumber\\
&\leq C_4 \left(\norm{h_{k_s}^{(L)}}^{2p}+\norm{h_{k_s}^{(L)}}^{p}+1+\norm{h_{k_s}^{(L)}}^{p}  \, \norm{V^A_s-V^A_{t_{k_s}}}^{p} + \norm{V^b_s-V^b_{t_{k_s}}}^{p} \right)\label{eq:bound13}
\end{align}
for some $C_4>0$, as $\mu(t,h)$ is quadratic in $h$.  The value of $p>2$ will be determined later. From~\eqref{eq:bound13}, we get
\begin{eqnarray}
&&\mathbb{E}\left[\sup_{0 \leq s \leq 1} \norm{\hbarr^{(L)}_s-\htilde^{(L)}_s}^{p} \right] \nonumber\\
&\leq& C_4 \left(\mathbb{E}\left[\sup_{0 \leq s \leq 1} \norm{\hbarr_{s}^{(L)}}^{2p}\right] + \mathbb{E}\left[\sup_{0 \leq s \leq 1}\norm{\hbarr_{s}^{(L)}}^{p}\right]+1\right)\nonumber \\ 
&&+ \,\, C_5\left(\left(\mathbb{E}\left[\sup_{0 \leq s \leq 1} \norm{\hbarr_{s}^{(L)}}^{2p}\right] \mathbb{E}\left[\sup_{0\leq s \leq 1} \norm{V^A_s-V^A_{t_{k_s}}}^{2p}\right]\right)^{1/2}+\mathbb{E}\left[\sup_{0\leq s \leq 1} \norm{V^b_s-V^b_{t_{k_s}}}^{p}\right]\right) \nonumber \\
&\leq& C_4 \left(\mathbb{E}\left[\sup_{0 \leq s \leq 1} \norm{\hbarr_{s}^{(L)}}^{2p}\right]+\mathbb{E}\left[\sup_{0 \leq s \leq 1} \norm{\hbarr_{s}^{(L)}}^{p}\right]+1\right)\nonumber\\
&&+ \,\, C_6 \left(\left(\mathbb{E}\left[\sup_{0 \leq s \leq 1} \norm{\hbarr_{s}^{(L)}}^{2p}\right]\mathbb{E}\left[\sup_{0\leq s \leq 1} \norm{V^A_s}^{2p}\right]\right)^{1/2}+\mathbb{E}\left[\sup_{0\leq s \leq 1} \norm{V^b_s}^{p}\right]\right),\label{eq:bound8}
\end{eqnarray}
for some constants $C_4,C_5,C_6>0$ independent of $L$, $R$ and $\delta$. The first inequality holds by the H\"older and \eqref{eq:bound8} holds by the Minkowski inequality. Take $p=\frac{1}{2}p_1 > 2$. Then,~\eqref{eq:bound8} is bounded thanks to Assumption~\ref{ass:strong} for $\mathbb{E}\left[\sup_{0 \leq t \leq 1}\|\hbarr_{t}^{(L)}\|^{2p}\right]<\infty$, and we have $\mathbb{E}\left[\sup_{0\leq s \leq 1}\norm{W^A_s}^{p}\right]<\infty $ and $\mathbb{E}\left[\sup_{0\leq s \leq 1}\norm{W^b_s}^{p}\right]<\infty$ by \eqref{eq:uniform_W}.
 Hence, by the Minkowski inequality, we have
\begin{eqnarray*}
\mathbb{E}\left[\sup_{0 \leq t \leq 1} \norm{\htilde^{(L)}_t}^{p}\right] \leq 2^{p-1}\mathbb{E}\left[\sup_{0 \leq s \leq 1} \norm{\hbarr^{(L)}_s-\htilde^{(L)}_s}^{p}\right]  +2^{p-1}\mathbb{E}\left[\sup_{0 \leq t \leq 1} \norm{\hbarr^{(L)}_t}^{p}\right] < \infty.
\end{eqnarray*} 
 \end{proof}

\subsubsection{Proof of Theorem~\ref{thm:H2}}\label{sec:main_pf}
We are now ready to show the proof of Theorem~\ref{thm:H2}.
\begin{proof}
Let us define two stopping times to utilize the local Lipschitz property of $\mu$:
 \begin{eqnarray}
 \tau_R := \inf \left\{t \geq 0: \norm{\htilde^{(L)}_t} \geq R\right\}, \quad \rho_R := \inf \left\{t \ge 0: \norm{H_t} \geq R\right\},\quad \theta_R := \tau_R \land \rho_R,
 \end{eqnarray}
 and define the approximation errors
 \begin{eqnarray}
 e_1(t):= \htilde^{(L)}_t-H_t, \,\,\text{ and }\,\,
 e_2(t):= \htilde^{(L)}_t-\hbarr^{(L)}_t.
 \end{eqnarray}
 The proof contains two steps. The first step is to show $\lim_{L\rightarrow\infty}\mathbb{E}\left[ \sup_{0 \leq t \leq 1} \norm{e_1(t)}^2\right]=0$ and the second step is to show $\lim_{L\rightarrow\infty}\mathbb{E}\left[ \sup_{0 \leq t \leq 1} \norm{e_2(t)}^2\right]=0$. 

Following the idea in \cite{higham2002strong}, we first show that for any $\delta>0$ (to be determined later):
\begin{eqnarray}\label{eq:bound1}
\mathbb{E}\left[ \sup_{0 \leq t \leq 1} \norm{e_1(t)}^2\right] \leq \mathbb{E}\left[ \sup_{0 \leq t \leq 1} \norm{\htilde^{(L)}_{t\wedge \theta_R}-H_{t\wedge \theta_R}}^2\right] +\frac{2^{p+1}\delta C}{p} + \frac{2(p-2)C}{p\delta^{2/(p-2)}R^p},
 \end{eqnarray}
where $C$ and $p$ are defined in~\eqref{ass:boundedness.v2}.
To see this, recall that by Young's inequality, for $r^{-1}+q^{-1} =1$, we have 
 \begin{equation} \label{eq:young}
 ab = \delta^{1/r}a \cdot \delta^{1/q-1}b \leq \frac{\delta}{r} a^{r} +\frac{1}{q\delta^{q/r}}b^q, \quad \forall a,b,\delta>0.
 \end{equation}
 First decompose the left-hand side of~\eqref{eq:bound1} to obtain, for all $\delta>0$,
 \begin{eqnarray}\label{eq:bound6}
 \mathbb{E}\left[\sup_{0 \leq t \leq 1}\|e_1(t)\|^2\right] &=& \mathbb{E}\left[\sup_{0 \leq t \le  1} \|e_1(t)\|^2 \one_{\{\tau_R>1,\rho_R>1\}}\right]+\mathbb{E}\left[\sup_{0 \leq t \leq 1} \|e_1(t)\|^2 \one_{\{\tau_R\le 1 \text{ or }\rho_R\le 1\}}\right]\nonumber\\
 &\leq& \mathbb{E}\left[ \sup_{0\leq t\leq 1} \|e_1(\tR)\|^2 \one_{\{\theta_R>1\}} \right] +\frac{2\delta}{p} \mathbb{E}\left[ \sup_{0\leq t \leq 1} \|e_1(t)\|^p\right]\nonumber\\ &&\quad+\,\frac{1-2/p}{\delta^{2/(p-2)}} \mathbb{P}\left(\tau_R \leq 1 \text{ or } \rho_R \leq 1\right).
 \end{eqnarray}
 where we apply \eqref{eq:young} with $r=p/2$ to the second term. Now
 \begin{eqnarray}\label{eq:bound4}
 \mathbb{P}(\tau_R \leq 1) =\mathbb{E}\left[ \one_{\{\tau_R \leq 1\}} \frac{\| \htilde^{(L)}_{\tau_R} \|^p}{R^p}\right] \leq \frac{1}{R^p}\mathbb{E}\left[\sup_{0\leq t \leq 1} \norm{ \htilde^{(L)}_t}^p\right] \leq \frac{C}{R^p}.
 \end{eqnarray}
 A similar result can be derived for $\rho_R$, so that we have
 \begin{eqnarray}\label{eq:bound5}
 \mathbb{P}(\tau_R \leq 1 \text{ or } \rho_R \leq 1) \leq \mathbb{P}(\tau_R \leq 1) + \mathbb{P}(\rho_R \leq 1) \leq \frac{2C}{R^p}.
 \end{eqnarray}
 Using the inequalities in \eqref{eq:bound4}--\eqref{eq:bound5}, along with 
 \begin{eqnarray}
 \mathbb{E}\left[\sup_{0 \leq t \leq 1}\norm{e_1(t)}^p \right] \leq 2^{p-1} \mathbb{E}\left[ \sup_{0\leq t \leq 1} \left( \norm{\htilde^{(L)}_t}^p+\norm{H_t}^p \right) \right] \leq 2^p\,C
 \end{eqnarray}
 in \eqref{eq:bound6}, we show the desired result \eqref{eq:bound1}.

To obtain a uniform bound on $\widetilde{H}-H$ , we bound the first term on the right-hand side of \eqref{eq:bound1}. Using the definition of the targeted SDE limit in~\eqref{eq:limit_sde}:
\[
    H_{\tR}:= H_0 +\int_0^{\tR} \mu(s,H_s)\dd s +\int_0^{\tR} \left(\dd W_s^A \, H_s + \dd W_s^b\right),
\]
and the continuous-time approximation~\eqref{eq:cta}, we get
\begin{eqnarray*}
\norm{\htilde^{(L)}_{\tR}-H_{\tR}}^2 &=& \Bigg\lVert \int_0^{\tR}\left(\mu \hspace{-1pt} \left({ {t}_{k_s}},\hbarr^{(L)}_s \right) \dd s - \mu\left(s,H_s \right)\right)\dd s \\
&& \quad + \int_0^{\tR}\dd W_s^A \, \left(\hbarr^{(L)}_s-H_s\right)+ \sum_{k < L(\tR)} D_k^{(L)}\hspace{-2pt}\left( {h}^{(L)}_k \right) \Bigg\rVert^2 \\
&=&\Bigg\lVert\int_0^{\tR}\left(\mu\hspace{-1pt}\left(s,\hbarr^{(L)}_s\right)  - \mu(s,H_s) + \mu\hspace{-1pt}\left({ {t}_{k_s}},\hbarr^{(L)}_s\right) - \mu\hspace{-1pt}\left(s,\hbarr^{(L)}_s\right)\right)\dd s \\
&& \quad + \int_0^{\tR}\dd W_s^A \, \left(\hbarr^{(L)}_s-H_s\right)+ \sum_{k < L(\tR)} D_k^{(L)}\hspace{-2pt}\left({h}^{(L)}_k\right) \Bigg\lVert^2.
\end{eqnarray*}
We first bound the above using Cauchy-Schwarz inequality:
\begin{eqnarray*}
&&\norm{\htilde^{(L)}_{\tR}-H_{\tR}}^2 \\
&\leq & 4 \left[ \int_0^{\tR}\norm{ \mu\hspace{-1pt}\left(s,\hbarr^{(L)}_s\right) \dd s-\mu\left(s,H_s\right)}^2 \dd s  \right] + 4 \, \norm{ \int_0^{\tR}\dd W_s^A \, \left(\hbarr^{(L)}_s-H_s\right) }^2 \\
&&+ \,\,  4 \, \left [ \int_0^{\tR} \norm{\mu\hspace{-1pt}\left({ {t}_{k_s}},\hbarr^{(L)}_s\right) \dd s-\mu\hspace{-1pt}\left(s,\hbarr^{(L)}_s\right)}^2 \dd s  \right] + 4 \, \norm{ \sum_{k < L(\tR)} D_k^{(L)}\hspace{-2pt}\left( {h}^{(L)}_k \right) }^2.
\end{eqnarray*}
Now, from the local Lipschitz condition~\eqref{ass:local_lipschitz} and Doob's martingale inequality~\cite{revuz2013continuous}, we have for any $\tau \leq 1$,
\begin{align}
 &\mathbb{E}\left[\sup_{0 \leq t \leq \tau} \norm{\htilde^{(L)}_{\tR}-H_{\tR}}^2 \right] \nonumber \\
 &\leq 4\left(C_R+4C^2_1\right) \mathbb{E}\int_0^{\tauR} \norm{\hbarr^{(L)}_s-H_s}^2 \dd s \nonumber \\
 &\,\, + \, 4 \,  \mathbb{E}\left[ \int_0^{\tR}\norm{ \mu\hspace{-1pt}\left( t_{k_s}, \hbarr^{(L)}_s \right) \dd s - \mu\hspace{-1pt}\left(s, \hbarr^{(L)}_s \right)}^2 \dd s\right] \,+\, 4\sum_{k\leq L\tau} \mathbb{E}\norm{D_k^{(L)} \hspace{-2pt}\left( {h}^{(L)}_k \right) \one_{\|{h}^{(L)}_k\|\leq R}}^2 \nonumber\\
 &\leq C_R' \int_0^{\tau}  \mathbb{E} \left[ \sup_{0 \leq r \leq s}\norm{\htilde^{(L)}_{r\wedge\theta_R}-H_{r\wedge\theta_R}}^2\right] \dd s +  C_R' \underbrace{\mathbb{E}\int_0^{\tauR} \norm{\hbarr^{(L)}_s-\htilde^{(L)}_s}^2 \dd s}_{\tcircle{1}} \nonumber\\
 &\,\, + \,  4  \, \underbrace{\mathbb{E}\left[ \int_0^{\tR}\norm{ \mu\hspace{-1pt}\left({ {t}_{k_s}}, \hbarr^{(L)}_s\right) - \mu\hspace{-1pt}\left(s,\hbarr^{(L)}_s\right)}^2 \dd s \right]}_{\tcircle{2}}+ \, 4\underbrace{\mathbb{E}\left[\sup_{0\leq t \leq \tau} \norm{\sum_{k < L(\tR)} D_k^{(L)} \hspace{-2pt} \left({h}^{(L)}_k\right)}^2\right]}_{\tcircle{3}} \label{eq:key_bound}
 \end{align}
where $C_R' \coloneqq 8 \left(C_R+4C^2_1\right)$. First, we give an upper bound for~$\tcircle{2}$. By the Cauchy–Schwarz inequality,
 \begin{align*}
 \norm{\mu(t,h)-\mu(s,h)}^2  \leq 5 \,\, \Big( &\norm{U_t^{A}- U_s^{A}}^2 \, \norm{h}^2 + \, \norm{U_t^{b}-U_s^{b}}^2 + \norm{\overbar{A}_{s}- \overbar{A}_{t}} \norm{h}^2 + \norm{\overline{b}_{t}-\overline{b}_{s}}^2 \\
 &+\frac{1}{2}\sigma^{\prime \prime}(0)\norm{Q(t,h)-Q(s,h)}^2\Big).
 \end{align*}
Hence, for $h\in \mathbb{R}^d$, the following holds almost surely by \eqref{eq:continuity_ito}:
  \begin{eqnarray}
\norm{\mu(t,h)-\mu(s,h)}^2  \leq C_M \abs{t-s}^{\kappa}\left( 1+ \norm{h}^2 + \norm{h}^4 \right).
 \end{eqnarray}
 Under Assumption \ref{ass:strong}, there exists a constant $\widetilde{C}_0>0$ such that
\begin{eqnarray*}
\mathbb{E}\left[\sup_{0 \leq t \leq 1}\left(\norm{\hbarr^{(L)}_t}^{4}+\norm{\hbarr^{(L)}_t}^{2} \right)\right]  \leq \widetilde{C}_0.
\end{eqnarray*} 
 Hence by Tonelli's theorem,
\begin{align}
\mathbb{E}\left [ \int_0^{\tR}\norm{\mu\hspace{-1pt} \left({ {t}_{k_s}},\hbarr^{(L)}_s \right) - \mu\hspace{-1pt}\left(s,\hbarr^{(L)}_s\right)}^2 \dd s\right] \nonumber &\leq \int_0^{1}\mathbb{E}\left[ \norm{\mu\hspace{-1pt}\left({ {t}_{k_s}},\hbarr^{(L)}_s\right) - \mu\hspace{-1pt}\left(s,\hbarr^{(L)}_s\right)}^2\right]\dd s\nonumber \\
&\leq (\widetilde{C}_0+1)C_M L \left(\int_0^{1/L} r^{\kappa} \dd r\right) \nonumber \\
&= \frac{(\widetilde{C}_0+1)C_M }{1+\kappa}L^{-\kappa}. \label{eq:tonelli}
\end{align}

{\bf Upper bound on $\tcircle{3}$.}
Define the following discrete filtration 
\begin{equation} \label{def:filtration-G}
\GG_{k} \coloneqq \sigma\Big(U_s^A, U_s^A,q_s^A,q_s^b, B_s^{A}, B_s^{b} \,:\,\, s \leq t_{k+1} \Big).
\end{equation}
Note that $h_k^{(L)}$ is $\GG_{k-1}$-measurable but not $\GG_{k}$-measurable. Define for $k=0, \ldots, L-1$ and for $i=1, \ldots, d$:
\begin{eqnarray*}
 X^{(L)}_{k,i}   &:=& \left( \left(\Delta V_k^A  \, h_{k}^{(L)}\right)_i + \left(\Delta V_k^b \right)_i\right)^2-\mathbb{E} \left[\left. \left( \left(\Delta V_k^A  \, h_{k}^{(L)}\right)_i + \left(\Delta V_k^b\right)_i\right)^2\right\vert \GG_{k-1}\right]\label{eq:x_def}\\
 Y^{(L)}_{k,i} &:=& \mathbb{E} \left[\left. \left( \left(\Delta V_k^A  \, h_{k}^{(L)}\right)_i + \left(\Delta V_k^b\right)_i\right)^2\right\vert \GG_{k-1}\right] - Q_i\hspace{-1pt} \left(t_k, h_k^{(L)} \right)\Delta_L\\
 J^{(L)}_{k,i} &:=& \widetilde{\mu}_i(t,h)^2 (\Delta_L)^2 + 2 \widetilde{\mu}_i(t,h) \Delta_L \left( \left(\Delta V_k^A  \, h\right)_i + \left(\Delta V_k^b\right)_i\right).
\end{eqnarray*}
We can then decompose
\begin{align*}
D_{k,i}^{(L)}\hspace{-2pt}\left(h_k^{(L)}\right) &= \frac{1}{6}\sigma^{\prime \prime \prime} \hspace{-2pt} \left( \xi^{(L)}_{k,i} \right) M^{(L)}_{k,i} \hspace{-2pt} \left(h_k^{(L)}\right)^3 + N_{k,i}^{(L)}\hspace{-2pt}\left(h_k^{(L)}\right) \\
&= \frac{1}{6}\sigma^{\prime \prime \prime} \hspace{-2pt} \left(\xi^{(L)}_{k,i}\right) M^{(L)}_{k,i}\hspace{-2pt}\left(h_k^{(L)}\right)^3 + \frac{1}{2} \sigma^{\prime\prime}(0) \left( X^{(L)}_{k,i} + Y^{(L)}_{k,i} + J^{(L)}_{k,i} \right).
\end{align*}
Hence, we deduce the following bound on $\tcircle{3}$ by Cauchy-Schwarz.
\begin{align}
&\mathbb{E}\left[\sup_{0\leq t \leq \tau} \abs{\sum_{k < L(\tR)} D_{k,i}^{(L)} \hspace{-2pt} \left({h}^{(L)}_k\right)}^2\right]\nonumber\\
&\leq \sigma''(0)^2 \,\, \mathbb{E}\left[\sup_{0\leq t \leq \tau} \abs{\sum_{k < L(\tR)} X^{(L)}_{k,i} \,}^2 + \abs{\sum_{k < L(\tR)} Y^{(L)}_{k,i} \,}^2 + L\sum_{k < L(\tR)} \abs{J^{(L)}_{k,i}}^2 \right] \nonumber \\
&\quad + \, 4 \,\, \mathbb{E}\left[\sup_{0\leq t \leq \tau} \abs{\sum_{k < L(\tR)} \frac{1}{6} \sigma^{\prime\prime\prime}\hspace{-2pt} \left(\xi^{(L)}_{k,i}\right) M^{(L)}_{k,i} \hspace{-2pt} \left( h_k^{(L)}\right)^3 }^2\right]\nonumber\\
&\leq \sigma''(0)^2 \,\, \mathbb{E}\left[\sup_{0\leq t \leq \tau} \abs{\sum_{k < Lt} X^{(L)}_{k,i} \one_{\norm{{h}^{(L)}_k}\leq R}}^2\right] +  \sigma''(0)^2 \,\,\mathbb{E}\left[\sup_{0\leq t \leq \tau} \abs{\sum_{k < Lt} Y^{(L)}_{k,i} \one_{\norm{{h}^{(L)}_k}\leq R}}^2\right] \nonumber \\
&\quad + \, \sigma''(0)^2 L \sum_{k=0}^{L-1} \mathbb{E}\left[\abs{ J^{(L)}_{k,i}}^2 \one_{\norm{{h}^{(L)}_k}\leq R}\right]+ \frac{1}{9}\sigma'''\hspace{-2pt}\left(\xi^{(L)}_{k,i} \right)^2 L \sum_{k=0}^{L-1} \mathbb{E}\left[ M^{(L)}_{k,i}\hspace{-2pt} \left(h_k^{(L)}\right)^6 \one_{\norm{{h}^{(L)}_k}\leq R}\right]. \label{eq:bound2}
\end{align}
We provide an upper bound for each of the four terms in \eqref{eq:bound2}. For the first term, denote $\Xtilde^{(L)}_{k,i} \coloneqq X^{(L)}_{k,i} \one\left(\big\lVert {h}^{(L)}_k \big\rVert\leq R\right)$ and $S^{(L)}_{k,i} \coloneqq \sum_{k'=0}^k \Xtilde^{(L)}_{k',i}$ so that $\big\{ S^{(L)}_{k,i}: k=-1, 0,\ldots, L-1\big\}$ is a $\left( \GG_k \right)$--martingale.  Hence, by Doob's martingale inequality, we have
  \begin{equation} \label{eq:doobs-xtilde}
    \E\left[ \sup_{0 \leq t \leq \tau } \abs{ \sum_{k < Lt} X^{(L)}_{k,i} \one_{\norm{{h}^{(L)}_k}\leq R} }^2 \right] = \E\left[ \sup_{0 \leq t \leq \tau } \abs{ S^{(L)}_{\floor{Lt},i} }^2 \right] \leq 4 \, \E \left[ \abs{ S^{(L)}_{\floor{L \tau},i} }^2 \right].
  \end{equation}
Fix $k = 0, \ldots, L-1$. For $i=1, \ldots, d$, we compute the following conditional expectation.
\begin{align} 
\E \left[ \left(S^{(L)}_{k,i} \right)^2 \, \Big\vert \,\, \GG_{k-1} \right]  &= \E \left[ \left( S^{(L)}_{k-1,i} \right)^2 + 2 \Xtilde^{(L)}_{k,i} \sum_{k'=0}^{k-1} \Xtilde^{(L)}_{k',i} + \left(\Xtilde^{(L)}_{k, i} \right)^2 \, \Bigg\vert \,\, \GG_{k-1} \right] \nonumber \\
&= \left( S^{(L)}_{k-1,i} \right)^2  + \E\left[ \left(\Xtilde^{(L)}_{k,i} \right)^2 \, \Big\vert \,\, \GG_{k-1} \right]. \label{eq:tower-law}
\end{align}
The cross-term disappear as $\E\left[\left. \Xtilde^{(L)}_{k,i} \, \right\vert \, \GG_{k-1} \right] = \E\left[\left. X^{(L)}_{k,i} \, \right\vert \, \GG_{k-1} \right]\one\left(\big\lVert {h}^{(L)}_k \big\rVert\leq R\right) = 0$ by definition of $X^{(L)}_{k,i}$. Furthermore, conditionally on $\GG_{k-1}$ and on $\big\{ \big\lVert {h}^{(L)}_k \big\rVert\leq R \big\}$,  observe that $X^{(L)}_{k,i}$ is the centered square of a normal random variable whose variance is $\OO(L^{-1})$ uniformly in $k$ by \eqref{eq:bound_ito}, so there exist $C_{R,1} > 0$ depending only on $R$ such that  
\begin{equation*}
\sup_{0\leq k < L} \E\left[ \left(\Xtilde^{(L)}_{k,i} \right)^2 \, \Big\vert \,\, \GG_{k-1} \right] \leq C_{R,1} L^{-2}.
\end{equation*}
Hence, 
plugging  back into \eqref{eq:doobs-xtilde}, 
we obtain
\begin{equation}\label{eq:bound14}
\E\left[ \sup_{0 \leq t \leq \tau } \abs{ \sum_{k\leq Lt} X_k^i\one_{\norm{{h}^{(L)}_k}\leq R} }^2 \right]  \leq 4 C_{R,1} L^{-1}.
\end{equation}
For the second term involving $Y^{(L)}_{k,i}$, we explicitly compute the conditional expectation using the definition of $V$ in~\eqref{eq:case1_mean} and the definition of $Q$ in~\eqref{eq:Q_thm}.

\begin{align*}
    Y^{(L)}_{k,i} &= \E \left.\left[\left( \Delta V^b_k \right)_i^2 + \sum_{j,l=1}^d h^{(L)}_{k, j} h^{(L)}_{k, l} \left( \Delta V_k^A \right)_{ij} \left( \Delta V_k^A \right)_{il}\, \right\vert \, \GG_{k-1}  \right] - Q_i \hspace{-1pt} \left(t_k,h_k^{(L)}\right) \Delta_L \\
    &= \int_{t_k}^{t_{k+1}} \left( \E\left[\left. \Sigma^b_{s, ii} \, \right\vert \, \GG_{k-1} \right] + \sum_{j,l=1}^d h^{(L)}_{k, j} h^{(L)}_{k, l}  \E\left[\left. \Sigma^A_{s, ijil}  \, \right\vert \, \GG_{k-1} \right] \right) \dd s - Q_i\hspace{-1pt} \left(t_k, h_k^{(L)}\right) \Delta_L \\
    &= \int_{t_k}^{t_{k+1}} \left( \E\left[\left. \Sigma^b_{s, ii} - \Sigma^b_{t_k, ii} \, \right\vert \, \GG_{k-1} \right] + \sum_{j,l=1}^d h^{(L)}_{k, j} h^{(L)}_{k, l}  \E \left[\left. \Sigma^A_{s, ijil} - \Sigma^A_{t_k, ijil} \, \right\vert \, \GG_{k-1} \right] \right) \dd s. \\
\end{align*}
By Cauchy-Schwarz, Tonelli and \eqref{eq:continuity_ito} in Assumption \ref{ass:ito} $(ii)$ we obtain:
\begin{align}
&\mathbb{E}\left[\sup_{0\leq t \leq \tau} \abs{\sum_{k\leq L(\tR)} Y^{(L)}_{k,i} \one_{\norm{{h}^{(L)}_k}\leq R}}^2\right] \leq \E \left[ \left( \sum_{k=0}^{L-1} \abs{Y^{(L)}_{k,i}} \one_{\norm{{h}^{(L)}_k}\leq R} \right)^2 \right]\nonumber \\
&\leq \E\left[ \left( \sum_{k=0}^{L-1} \int_{t_k}^{t_{k+1}} \hspace{-2pt} \left( \expect*{ \abs{\Sigma^b_{s, ii} - \Sigma^b_{t_k, ii} }  | \GG_{k-1}} + R^2 \sum_{j,l=1}^d  \expect*{ \abs{\Sigma^A_{s, ijil} - \Sigma^A_{t_k, ijil}} | \GG_{k-1}} \right) \dd s \right)^2 \right] \nonumber \\
&\leq \left( \sum_{k=0}^{L-1} \int_{t_k}^{t_{k+1}} (1+R^2) M^{1/2} \abs{s-t_{k_s}}^{\kappa/2} \dd s \right)^2 \nonumber \\
&= M(1+R^2)^2 \left( L\int_{0}^{1/L} r^{\kappa/2} \dd r \right)^2 = \frac{M(1+R^2)^2}{(1+\kappa/2)^2} L^{-\kappa} \eqqcolon C_{R,2} L^{-\kappa},\label{eq:bound15}
\end{align}
where $C_{R, 2} > 0$ depends only on $R$.
Moving to the third term of \eqref{eq:bound2} involving $J^{(L)}_{k,i}$, observe that
\begin{align*}
\sup_{\norm{h}\leq R}\mathbb{E}\left[ \abs{\sum_{j=1}^d \int_{t_k}^{t_{k+1}}(\dd V_{t}^A)_{ij}h_{j}}^2\right] &\leq R^2 d \sum_{j=1}^d\mathbb{E}\left[\abs{ \sum_{l,m=1}^d \int_{t_{k}}^{t_{k+1}}  \left(q_s^A\right)_{ijlm} \left(\dd B_s^A\right)_{lm} }^2\right] \leq C_{7}R^2 \Delta_L,\\
\mathbb{E}\left[ \abs{ \int_{t_k}^{t_{k+1}}(\dd V_{t}^b)_{i}}^2\right] &= \mathbb{E}\left[\abs{ \sum_{l=1}^d \int_{t_{k}}^{t_{k+1}}  \left(q_r^b\right)_{il} \left(\dd B_s^b\right)_{l} }^2\right] \leq C_{8}\Delta_L,
\end{align*}
for some $C_{7},C_{8}>0$ independent of $R$ since $\sup_{0 \leq t\leq 1}\norm{\Sigma^A_t}\leq C_1$ and $\sup_{0 \leq t\leq 1}\norm{\Sigma^b_t}\leq C_1$ almost surely.
Then there exists $C_{R,3}>0$ depending only on $R$ such that
\begin{eqnarray}\label{eq:bound10}
\sup_{\|h\|\leq R}\mathbb{E}\left[\abs{ J^{(L)}_{k,i} }^2 \one_{\norm{h_k^{(L)}}\leq R}\right] \leq C_{R,3} L^{-3}.
\end{eqnarray}
Finally, we bound the fourth term of \eqref{eq:bound2} using Cauchy-Schwarz, Assumption \ref{ass:activation} and property \eqref{eq:bound_ito} of the It\^o processes:
\begin{eqnarray}\label{eq:bound9}
\sigma'''(\xi_i)^2 \sup_{\|h\|\leq R}  \mathbb{E}\left[ M^{(L)}_{k,i}  \big(h \big)^6 \right] \leq  m^2\,C_{R,4} L^{-3},
\end{eqnarray}
for some constant $C_{R, 4}>0$ depending only on $R$. Combining the results in  \eqref{eq:bound14}, \eqref{eq:bound15}, \eqref{eq:bound10} and \eqref{eq:bound9}, there exists constants $C_{R,5}, C_{R,6} > 0$ depending only on $R$ such that
\begin{align} 
\E\left[\sup_{0\leq t \leq \tau} \abs{\sum_{k\leq L(\tR)} D_k^{(L),i} \hspace{-2pt} \left({h}^{(L)}_k\right)}^2\right] \leq \frac{C_{R, 5}}{4d}L^{-\kappa} + \frac{C_{R, 6}}{4d} L^{-1}.  \label{eq:bound11}
\end{align}

{\bf Upper bound on $\tcircle{1}$.} Given $s\in[0,T\wedge \theta_R)$, we have
\begin{eqnarray}
\hbarr^{(L)}_s-\htilde^{(L)}_s &=& h_{k_s}^{(L)} - \left( h_{k_s}^{(L)}+\int_{t_{k_s}}^s \mu(s,\hbarr^{(L)}_s)\dd s + \int_{t_{k_s}}^s \left(\dd V_s^A\hbarr_s^{(L)} + \dd V_s^b\right)\right)\nonumber\\
&=&-\mu\hspace{-1pt}\left(t_{k_s},h_{k_s}^{(L)} \right) (s-t_{k_s})-\left(V^A_s-V^A_{t_{k_s}}\right) \, h_{k_s}^{(L)}-\left(V^b_s-V^b_{t_{k_s}}\right) \label{eq:Htilde_Hbar}
\end{eqnarray}
by continuity of $\mu$. Hence
\begin{eqnarray}
\norm{\hbarr^{(L)}_s-\htilde^{(L)}_s}^2 \leq  3 \norm{\mu\hspace{-1pt} \left(t_{k_s},h_{k_s}^{(L)} \right)}^2 (\Delta_L)^2 
+ 3\norm{h_{k_s}^{(L)}}^2 \norm{V^A_s-V^A_{t_{k_s}}}^2 + 3\norm{V^b_s-V^b_{t_{k_s}}}^2.
\end{eqnarray}
Now, from the local Lipschitz condition \eqref{ass:local_lipschitz}, for $\norm{h} \leq R$ we have almost surely
\begin{eqnarray*}
\norm{\mu(s,h)}^2 \leq 2\left( \norm{\mu(s,h)-\mu(s,0)}^2 + \norm{\mu(s,0)}^2 \right) \leq 2\left(C_R\norm{h}^2 + \norm{\mu(s,0)}^2\right).
\end{eqnarray*}
Combining the two previous inequalities we obtain
\begin{eqnarray*}
\norm{\hbarr^{(L)}_s-\htilde^{(L)}_s}^2 \leq 4 \left(C_R\norm{h_{k_s}^{(L)}}^2 + \norm{\mu(s,0)}^2 + 1 \right) \left(\Delta_L^2 + \norm{V^A_s-V^A_{t_{k_s}}}^2 + \norm{V^b_s-V^b_{t_{k_s}}}^2 \right).
\end{eqnarray*}
Hence, using \eqref{ass:boundedness.v2} and the Lyapunov inequality \cite{platen2010numerical}, we get
\begin{eqnarray}
&&\mathbb{E}\int_0^{\tauR}\norm{\hbarr^{(L)}_s-\htilde^{(L)}_s}^2 \dd s\nonumber\\
&\leq& \mathbb{E}\int_0^{\tauR} 4 \left(C_R\norm{h_{k_s}^{(L)}}^2 + \norm{\mu(s,0)}^2+1\right)\left(\Delta_L^2+\norm{V^A_s-V^A_{t_{k_s}}}^2 + \norm{V^b_s-V^b_{t_{k_s}}}^2\right)\,\dd s \nonumber\\
&\leq& \int_0^{\tau} 4 \,\, \mathbb{E} \left[\left(C_R\norm{h_{k_s}^{(L)}}^2 + \norm{\mu(s,0)}^2+1\right)\left(\Delta_L^2 + \norm{V^A_s-V^A_{t_{k_s}}}^2 + \norm{V^b_s-V^b_{t_{k_s}}}^2\right)\right]\,\dd s\nonumber\\
&\leq& \int_0^{1} 4  \left(C_R \, \mathbb{E}\left[\norm{h_{k_s}^{(L)}}^2 \right] + \norm{\mu(s,0)}^2+1\right)\left(\Delta_L^2+2C_1\Delta_L+2C_1\Delta_L\right)\,\dd s\nonumber\\
&\leq&  4  \left(C_R\,C_0^{2/p}+1+\int_0^1 \norm{\mu(s,0)}^2 \dd s\right)\left(\Delta_L^2+4C_1\Delta_L\right). \label{eq:bound12}
\end{eqnarray}
Combining the results in \eqref{eq:tonelli}, \eqref{eq:bound11} and \eqref{eq:bound12}, we have in \eqref{eq:key_bound} that
\begin{align*}
&\mathbb{E}\left[\sup_{0 \leq t \leq \tau} \norm{\htilde^{(L)}_{\tauR}-H_{\tR}}^2 \right] \leq C_R' \left(C_R\,C_0^{2/p}+1+\int_0^1 \norm{\mu(s,0)}^2 \dd s\right)\left( L^{-2} + 4C_1 L^{-1} \right) \\ 
&+ \, \frac{(\widetilde{C}_0+1)C_M }{1+\kappa}L^{-\kappa} + \left(C_{R, 5} L^{-\kappa} + C_{R, 6} L^{-1} \right) +  C_R' \int_0^{\tau} \mathbb{E} \left[ \sup_{0 \leq r \leq s}\norm{\htilde^{(L)}_{r\wedge\theta_R}-H_{r\wedge\theta_R}}^2\right]\dd s.
\end{align*}
Applying the Grönwall inequality,
\begin{eqnarray}\label{grownwall}
\mathbb{E}\left[\sup_{0 \leq t \leq \tau} \norm{\htilde^{(L)}_{\tauR}-H_{\tR}}^2 \right] \leq C_9 C_{R, 7} { L^{- \min\{1,\kappa\}}} \exp( C_R'),
\end{eqnarray}
where $C_9$ is a universal constant independent of $L$, $R$ and $\delta$ and $C_{R,7}$ is a constant only depending on $R$. Combining \eqref{grownwall} with \eqref{eq:bound1}, we have
\begin{eqnarray}
\mathbb{E}\left[ \sup_{0 \leq t \leq 1} \norm{e_1(t)}^2 \right] \leq C_9 C_{R,7} L^{- \min\{1,\kappa\}} \exp(C_R') + \frac{2^{p+1}\delta C}{p} + \frac{2(p-2)C}{p\delta^{2/(p-2)}R^p}.\label{eq:bound_final}
\end{eqnarray}
Given any $\epsilon>0$, we can choose $\delta>0$ so that $\frac{2^{p+1}\delta C}{p}<\frac{\epsilon}{3}$, then choose $R$ so that $ \frac{2(p-2)C}{p\delta^{2/(p-2)}R^p}<\frac{\epsilon}{3}$, and finally choose $L$ sufficiently large so that 
\begin{equation*}
C_9 C_{R,7} L^{- \min\{1,\kappa\}} \exp(C_R') \leq \frac{\epsilon}{3}.
\end{equation*}
Therefore in \eqref{eq:bound_final}, we have, 
\begin{eqnarray}\label{eq:bound_e}
\mathbb{E}\left[ \sup_{0 \leq t \leq 1}\norm{e_1(t)}^2\right]  \leq \epsilon.
\end{eqnarray}

{\bf It remains to provide a uniform bound for $\hbarr-\widetilde{H}$.}
Recall the relationship between $\htilde$ and $\hbarr$ defined in~\eqref{eq:Htilde_Hbar}: by~\eqref{eq:bound_ito} we have almost surely that
\begin{eqnarray*}
\norm{\hbarr^{(L)}_s-\htilde^{(L)}_s}^{2} &\leq& 3\left( \norm{\mu\hspace{-1pt}\left({ {t_{k_s}}},h_{k_s}^{(L)} \right)}^{2} \left(\Delta_L\right)^{2} + \norm{h_{k_s}^{(L)}}^{2}  \, \norm{V^A_s-V^A_{t_{k_s}}}^{2} + \norm{V^b_s-V^b_{t_{k_s}}}^{2}\right)\\
&\leq & C_{10}\left( \norm{h_{k_s}^{(L)}}^4+\norm{h_{k_s}^{(L)}}^2+1 \right)\,(\Delta_L)^2 \\
&&+ \,\, 3 \left( \norm{h_{k_s}^{(L)}}^{2}  \, \norm{V^A_s-V^A_{t_{k_s}}}^{2} + \norm{V^b_s-V^b_{t_{k_s}}}^{2}\right).
\end{eqnarray*}
Therefore,
\begin{align}
&\mathbb{E} \left[\sup_{0\leq s  \leq 1}\norm{\hbarr^{(L)}_s-\htilde^{(L)}_s}^{2}\right] 
\leq  C_{10}\left( \mathbb{E}\left[\sup_{0\leq s \leq 1}\norm{h_{k_s}^{(L)}}^4\right]+\mathbb{E}\left[\sup_{0\leq s \leq 1}\norm{h_{k_s}^{(L)}}^2\right]+1 \right)\,(\Delta_L)^2\nonumber\\
&+ 3\, \left( \left(\mathbb{E}\left[\sup_{0\leq s \leq 1}\norm{h_{k_s}^{(L)}}^{4}\right]  \, \mathbb{E}\left[\sup_{0\leq s \leq 1}\norm{V^A_s-V^A_{t_{k_s}}}^{4}\right]\right)^{1/2} +\mathbb{E}\left[\sup_{0\leq s \leq 1} \norm{V^b_s-V^b_{t_{k_s}}}^{2}\right] \right)\label{eq:sup_2h}.
\end{align}
First, by Assumption~\ref{ass:strong},
\begin{equation}\label{eq:h_max}
\mathbb{E}\left[ \sup_{0 \leq s \leq 1 }\norm{h_{k_s}^{(L)}}^n \right] = \mathbb{E}\left[ \sup_{0 \leq s \leq 1 }\norm{\hbarr_{s}^{(L)}}^{n} \right] <\infty, \quad n \in\left\{2,4\right\}.
\end{equation}
Second, by the Power Mean inequality and Doob's martingale inequality,
\begin{align*}
\mathbb{E}\left[\sup_{t_{k}\leq s < t_{k+1}}\norm{V_s^A-V_{t_{k}}^A}^4\right] &= \mathbb{E}\left[\sup_{t_{k}\leq s < t_{k+1}} \left( \sum_{i,j=1}^d \abs{ \sum_{k,l=1}^d \int_{t_{k}}^s  \left(q_r^A\right)_{ijkl} \left(\dd B_r^A\right)_{kl} }^2 \right)^2 \right] \\ 
&\leq d^8 \sum_{i,j,k,l=1}^d \E \left[\sup_{t_{k}\leq s < t_{k+1}} \abs{ \int_{t_{k}}^s \left(q_r^A\right)_{ijkl} \left(\dd B_r^A\right)_{kl} }^4  \right] \\
&\leq \Big(\frac{4}{3}\Big)^4 d^8 \sum_{i,j,k,l=1}^d \mathbb{E}\left[\abs{\int_{t_{k}}^{t_{k+1}} \left(q_s^A\right)_{ijkl} \left(\dd B_s^A\right)_{kl}}^4\right] \leq C_{11} \Delta^2_L,
\end{align*}
Hence
\begin{eqnarray}\label{eq:bound16}
\mathbb{E}\left[\sup_{0\leq s \leq  1}\norm{V_s^A-V_{t_{k_s}}^A}^4\right] \leq \mathbb{E}\left[\sum_{k=0}^{L-1}\left(\sup_{t_{k}\leq s < t_{k+1}}\norm{V_s^A-V_{t_{k_s}}^A}^4\right)\right] \leq C_{11} \Delta_L.
\end{eqnarray}
By H\"older inequality, 
\begin{eqnarray}\label{eq:bound17}
\mathbb{E}\left[\sup_{0\leq s \leq  1}\norm{V_s^A-V_{t_{k_s}}^A}^2\right] \leq \left(\mathbb{E}\left[\sup_{0\leq s \leq  1}\norm{V_s^A-V_{t_{k_s}}^A}^4\right]\right)^{1/2} \leq \sqrt{C_{11}}\Delta^{1/2}_L.
\end{eqnarray}
Combining \eqref{eq:h_max}, \eqref{eq:bound16}, and \eqref{eq:bound17} in \eqref{eq:sup_2h}, we obtain
\begin{eqnarray*}
\mathbb{E}\left[\sup_{0\leq   t \leq 1}\norm{e_2(t)}^{2}\right] = \mathbb{E}\left[\sup_{0\leq t  \leq 1}\norm{\hbarr^{(L)}_t-\htilde^{(L)}_t}^{2}\right] \leq C_{12}\Delta_L^{1/2},
\end{eqnarray*}
for some constant $C_{12}>0$. By choosing $L>(C_{12}/\epsilon)^2$,  we have
\begin{eqnarray}\label{eq:bound_h}
\mathbb{E}\left[\sup_{0\leq t  \leq 1}\norm{e_2(t)}^{2}\right] \leq {\epsilon}.
\end{eqnarray}
Finally, combining \eqref{eq:bound_e} and \eqref{eq:bound_h} leads to the desired result.

\end{proof}

\section{Asymptotic analysis of the backpropagation dynamics} \label{sec:backward}

The most widely used method to train neural networks is the pairing of
\begin{itemize}
    \item the backpropagation algorithm to find the exact gradient (or a stochastic approximation) of the loss function with respect to the network weights, and
    \item a variant of the gradient descent algorithm to iteratively update the network weights.
\end{itemize}
We are interested to study the behaviour of the former in residual networks, under our Scaling regimes \ref{hypothesis.1} and \ref{hypothesis.2}. To do so, we will first formalize the objective function and the discrete backward equation linking the gradient of the loss function across layers.

\subsection{Backpropagation in supervised learning}

Suppose we want to learn the mapping $\ftrue \in \CC^1(\R^d, \R^d)$ through a dataset of input-target pairs $\DD \coloneqq \left\{ (x_i, y_i) : i=1, \ldots, N \right\} \subset \R^d \times \R^d$, where $x_i \in B$ for some $B\subset \R^d$ compact and $y_i = \ftrue(x_i)$. The goal of any parametric supervised learning is to find, given a class of mappings $\phi_{\theta}: \R^d \to\R^d$, the parameter $\theta \in \Theta$ that minimizes the average training error: 
\begin{equation} \label{mean-loss-fct}
J_{\DD}(\theta) \coloneqq \frac{1}{N} \sum_{i=1}^N \ell( \phi_{\theta}(x_i), y_i) = \frac{1}{N} \sum_{i=1}^N \ell( \phi_{\theta}(x_i), \ftrue(x_i)) .
\end{equation}
Here, $\ell: \R^d \times \R^d \to \R_+$ is a loss function, for example the squared error $\ell(\widehat{y}, y) = \norm{y - \widehat{y}}^2$. In the following, we omit the dependence in $\DD$. Fix $L\in\N$ and define \[
\theta^{(L)} \coloneqq \left( A^{(L)}_k, b^{(L)}_k \right)_{k = 1}^L \in \left( \R^{d \times d} \times \R^d \right)^L.
\]
For an input $x\in\R^d$, recall the following forward dynamics for the residual network
\begin{equation} \label{eq:resnet.v4}
\begin{aligned}
h_0^{(L),x} &= x, \\
h^{(L),x}_{k+1} &= h^{(L),x}_{k} + L^{-\alpha}
\,\sigma_d \hspace{-1pt} \left(A^{(L)}_k h^{(L),x}_{k}+ b^{(L)}_k\right).
\end{aligned}
\end{equation}
We define $\phi_{\theta^{(L)}}(x) \coloneqq h_L^{(L),x}$. Our goal is to compute $\nabla_{\theta^{(L)}} J \hspace{-1pt} \left(\theta^{(L)}\right)$. Observe from the definition \eqref{mean-loss-fct} and the chain rule that \[
\nabla_{\theta_k^{(L)}} J \hspace{-1pt} \left(\theta^{(L)}\right) = \frac{1}{N} \sum_{i=1}^N \nabla_{\theta_k^{(L)}} h^{(L),\, x_i}_{k+1} \left( \frac{\del h_L^{(L),\, x_i} }{\del h^{(L),\, x_i}_{k+1}} \frac{\del \ell}{\del \widehat{y}}\left(h_L^{(L),\, x_i}, y_i \right) \right).
\]
The terms $\del \ell / \del \widehat{y}$ and $\nabla_{\theta_k} h_{k+1}$ are straightforward to obtain, so the crux of the challenge lies in computing $\del h_L / \del h_{k+1}$. Using \eqref{eq:resnet.v4}, for $x\in\R^d$, we get 
\begin{align}
    g_k^{(L), x} \coloneqq \frac{\del h_L^{(L),\, x} }{\del h^{(L),x}_k} &= \frac{\del h_L^{(L),\, x} }{\del h^{(L),x}_{k+1}} \frac{\del h^{(L), x}_{k+1}}{\del h^{(L), x}_k}  \nonumber \\
    &= g_{k+1}^{(L),x} \left( I_d + L^{-\alpha} \, \diag\left( \sigma'_d \hspace{-1pt} \left( A_k^{(L)} h_{k}^{(L), x} + b_k^{(L)} \right) \right) A^{(L)}_k \right) , \label{eq:backprop}
\end{align}
where $\sigma'_d(z) = \left( \sigma'(z_i) \right)_{i=1}^d \in \R^d$ for $z\in\R^d$. The terminal condition is given by $g_L^{(L), x} = I_d$. We now obtain the asymptotic dynamics of $g$ under three different cases. In particular, we derive (backward) ODE limits for any set of weights under Scaling regime \ref{hypothesis.1} and the asymptotic limit derived from an SDE under Scaling regime \ref{hypothesis.2}. For clarity, we omit the dependence in the input $x$ for $g_k^{(L)}$.

\subsection{Backward equation for the Jacobian under Scaling regime~\ref{hypothesis.1}}
Let $\overbar{G}^{(L)}:[0,1]\to \mathbb{R}^{d \times d}$ be a continuous-time extension of the Jacobians $g^{(L)}_k$ defined in~\eqref{eq:backprop}:
\begin{eqnarray}\label{eq:jacobian-interpolatiom}
\overbar{G}^{(L)}_t = g_{k+1}^{(L)} \one_{\frac{k}{L} < t \leq \frac{k+1}{L}}, \quad k =0,1,\ldots,L-1.
\end{eqnarray}

\begin{thm}[Backpropagation limits under Scaling regime~\ref{hypothesis.1}]\label{thm:backward-H1} 
Under the same assumptions as Theorem~\ref{thm:H1},
\begin{itemize}
    \item \underline{Neural ODE regime:} If $\alpha=1$, $\beta=0$, and $(H_t)_{t\in\left[0, 1\right]}$ is the solution to the neural ODE \eqref{eq:limit-node}, then the backpropagation dynamics converge uniformly to the solution to the linear (backward) ODE
   \begin{eqnarray} \label{eq:backward-1}
    \frac{\dd G_t}{\dd t} = - G_t \diag \left( \sigma'_d \hspace{-1pt} \left( \overbar{A}_t H_t + \overline{b}_t \right) \right) \overbar{A}_t, \quad G_1 = I_d
   \end{eqnarray}
   in the sense that  $\lim_{L \rightarrow \infty} \sup_{0 \leq t \leq 1}\|G_t-\overbar{G}^{(L)}_t\|=0$.
    \item \underline{Linear ODE regime:}
    If $\alpha+\beta=1$, $\beta>0$, and $(H_t)_{t\in\left[0, 1\right]}$ is the solution to the linear ODE \eqref{eq:limit2}, then the backpropagation dynamics converge uniformly to the solution to the linear (backward) ODE
   \begin{eqnarray}\label{eq:backward-2}
       \frac{\dd G_t}{\dd t} = -G_t \overbar{A}_t, \quad G_1 = I_d
   \end{eqnarray}
    in the sense that  $\lim_{L \rightarrow \infty} \sup_{0 \leq t \leq 1}\|G_t-\overbar{G}^{(L)}_t\|=0$.
\end{itemize}
\end{thm}

The ideas of the proof follow closely those of Theorem~\ref{thm:H1} and the complete proof is given in Section~\ref{proof:backward-H1}. We readily see that under Scaling regime~\ref{hypothesis.1}, the backward dynamics of the gradient become linear. When $\beta > 0$, which is the case observed in practice, the dependence on the activation function disappears in the large depth limit, exactly as for the forward dynamics. 

\subsection{Backward equation for the Jacobian under Scaling regime~\ref{hypothesis.2}}

Recall the set-up of Theorem~\ref{thm:H2}. Let $(\Omega, \mathcal{F}, \mathbb{F}, \mathbb{P})$ be a probability space with a $\mathbb{P}$-complete filtration $\mathbb{F}=(\mathcal{F}_t)_{t\geq 0}$. Let  $(B^A_t)_{t\ge 0}$, resp. $(B^b_t)_{t\ge 0}$, be $d\times d$-dimensional, resp. $d$-dimensional, independent $\mathbb{F}$-Brownian motions.
Recall that for Scaling regime~\ref{hypothesis.2},
\begin{equation} \label{eq:hypothesis-2-recall}
    A_k^{(L)} = \overbar{A}_{k/L} L^{-1} + W_{(k+1)/L}^A - W_{k/L}^A \qquad b^{(L)}_k = \overline{b}_{k/L} L^{-1} + W^b_{(k+1)/L} - W^b_{k/L},
\end{equation}
where $(W^A_t)_{t\in [0,1]}$ and $(W^b_t)_{t\in[0,1]}$ are Itô processes \cite{revuz2013continuous}  adapted to $\mathbb{F}$ and can be written in the form:
\begin{equation}
\begin{aligned}
\left(\dd W_t^A\right)_{ij} &= \sum_{k,l=1}^d \left(q_t^A\right)_{ijkl} \left(\dd B^{A}_t\right)_{kl} \quad \mbox{for }\, i,j=1, \ldots, d ,\\
\dd W_t^b &= q^b_t \dd B^b_t,
\end{aligned}
\end{equation}
with $W_0^A = 0$, $W_0^b = 0$, $q_t^A\in \mathbb{R}^{d,\otimes4}$ and $q_t^b\in \mathbb{R}^{d\times d}$ for $t\in [0,1]$. We use the notation in \eqref{eq.qv} and \eqref{def:covariance-tensors} for the quadratic variation of $W^A$ and $W^b$.

Define 
\begin{equation} \label{eq:def-nu}
     \nu(t,h) \coloneqq \overbar{A}_t \one_{\beta=1} + \frac{1}{2} \sigma''(0) \nabla_h Q(t, h).
\end{equation}
We will use  the following assumption for the results in this section:
\begin{ass} \label{ass:backward}
\[
\sup_{L} \E\left[ \sup_{0 \leq t \leq 1} \norm{\overbar{G}_t^{(L)}}^{4} \right] < \infty, \qquad \E\left[\exp\left( 8 \int_0^1 \abs{ \tr \left( \nu(s, H_s) \right)} \dd s \right) \right] < \infty.
\]
\end{ass}

\noindent The boundedness of the fourth moment of the Jacobians $g_k^{(L)}$  in $L$ is similar to Assumption \eqref{ass:strong} and is standard in the convergence of approximation schemes for SDE. The second part of Assumption \ref{ass:backward} is a technical condition: we need the fourth moment of the  $\nabla_x H_t^{x}$ to be bounded. Theorem \ref{thm:backward-H2} proves that the process $t\mapsto \nabla_x H_t^{x}$ satisfy a linear SDE with drift $\nu(t, H_t)$ linear in $H_t$, so we need finiteness of the $L^{8}$ norm of the exponential of the drift, see Lemma \ref{lem:linear-sde} for more details. In practice, $g_k^{(L)}$ and $h_k^{(L),\,x}$ stay bounded during training, so Assumption \eqref{ass:backward} is satisfied.

\begin{thm}[Backpropagation dynamics under Scaling regime~\ref{hypothesis.2}] \label{thm:backward-H2}
Let Assumptions~\ref{ass:activation}, \ref{ass:ito}, \ref{ass:strong}, and~\ref{ass:backward} hold  and let $\alpha=0$ and $\beta = 1$. Let $(H_t)_{t\in\left[0,1\right]}$ be a solution to the SDE \eqref{eq:limit1} and $(J_t)_{t\in\left[0,1\right]} \subset \R^{d\times d}$ be the unique solution to the linear matrix-valued SDE 
\begin{equation} \label{eq:def:limit-J}
\dd J_t = \big( \nu(t, H_t) \dd t + \dd W^A_t  \big) J_t , \quad J_0 = I_d,
\end{equation}
where $\nu$ is defined in \eqref{eq:def-nu}. Then, $\Prob-$a.s., $J_t$ is invertible for all $t\in\left[0,1\right]$
and
\begin{equation}
    \overbar{G}^{(L)} = \sum_{k=0}^{L-1} g_k^{(L)} \one_{\left[ t_k, t_{k+1}\right)}\mathop{\longrightarrow}^{L\to\infty} G_t \coloneqq J_1 J_t^{-1}\label{eq.Gt}
\end{equation}   uniformly in $L^1(\mathbb{P})$ in the sense of Def. \ref{def:unif-conv-l2}.
\end{thm}

The steps to prove Theorem \ref{thm:backward-H2}  are similar to those of Theorem~\ref{thm:H2} but the details are technically more involved. Indeed, terms that depend on $g_k^{(L),\,x}$ are not a priori adapted to the filtration generated by the Ito processes $W^A$ and $W^b$. To overcome this challenge, we denote
\begin{eqnarray}
   J_k^{(L),\, x}:= \nabla_x h_k^{(L),\,x},
\end{eqnarray}
and we can rewrite $g^{(L), \, x}_0 = g^{(L),\,x}_k J_k^{(L),\, x}$. This leads to a new perspective to understand $g^{(L),\, x}_k$ through two components $J_k^{(L),\,x}$ and $g^{(L),\,x}_0$. The first term $J_k^{(L),\,x}$ is adapted to the filtration generated by the Ito processes, and $g^{(L),\,x}_0$ is the Jacobian of the output with respect to the input, and does not depend on the layer. The complete proof is provided in Section \ref{sec:proof_backward-H2}.

\subsubsection{Connection with Neural SDE}
 
 In a recent work, \cite{li2020a} show that, when the hidden state $H$ satisfies a continuous-time 'neural SDE' dynamics,  
 the Jacobian of the output with respect to the hidden states satisfies a backward SDE: 
\begin{equation} \label{eq:backward-neural-sde}
    \dd \widehat{G}_t = \widehat{G}_t \left( -\nu(t, \widehat{H}_t) \dd t - \dd \widehat{W}^A_t \right), \quad \widehat{G}_1 = I_d,
\end{equation}
where $\widehat{W}^A$ is the time-reversed Brownian motion defined by $\widehat{W}^A_t \coloneqq W^A_t - W^A_1$, and $\widehat{H}_t$ is the solution of the backward flow of diffeomorphisms generated by the forward SDE \eqref{eq:limit1}.

It is clear that the limit $G_t$ in \eqref{eq.Gt} differs from the adjoint process \eqref{eq:backward-neural-sde}. Our limit $G_t = J_1 J_t^{-1}$ does not satisfy any forward nor backward SDE, as its solution is a function of $H_1$ which depends on weights across all layers i.e. the entire path of $W^A$. Indeed, Theorem 3.1 in \cite{Y2012} states that $t\mapsto J_t^{-1}$ solves the following SDE.
\begin{equation*}
\dd (J_t^{-1}) = J_t^{-1} \left( -\nu(t,H_t) \dd t - \dd W_t^A + \dd \left[W^A \right]_t \right), \quad J_0^{-1} = I_d.
\end{equation*}
Therefore, one can write 
\begin{align}
G_t = J_1 J_t^{-1} &= J_1 \left( J_1^{-1} + \int_{t}^1 J_s^{-1} \left( -\nu(s,H_s) \dd s - \dd W_s^A + \dd \left[W^A \right]_s \right) \right) \nonumber \\
&= I_d + \int_{t}^1 G_s \left( -\nu(s,H_s) \dd s - \dd W_s^A + \dd \left[W^A \right]_s \right). \label{eq:backward-sde-rewrite}
\end{align}
One can readily see that $G_t$ depends on $H_1$ for all $t\in[0,1]$. Note that the quadratic variation drift correction term stems from using Ito integrals instead of Stratonovitch integrals. \\
In contrast to \eqref{eq:backward-neural-sde},  
\eqref{eq:backward-sde-rewrite} is the exact large-depth limit of gradients computed by backpropagation in finite depth residual networks, as stated in Theorem \ref{thm:backward-H2}. 

\subsection{Proofs}

\subsubsection{Proof of Theorem~\ref{thm:backward-H1}} \label{proof:backward-H1}
The ideas of the proof follow closely those of Theorem~\ref{thm:H1}, and we will provide here the main arguments to the Neural ODE case. The other case is very similar. \\
Denote $t_k=k/L$, $k=0,1,\ldots,L$ as the uniform partition of the interval $[0,1]$. For $t\in (t_k,t_{k+1}]$, define
\begin{equation*}
\widetilde{G}^{(L)}_t \coloneqq g_{k+1}^{(L)} \left( I_d + (t_{k+1} - t) \,\diag\left( \sigma'_d \hspace{-1pt} \left( \overbar{A}_{t_k} H_{t_k} + \overline{b}_{t_k} \right) \right) \overbar{A}_{t_k} \right),
\end{equation*}
where $\overbar{A}$ and $\overline{b}$ are specified in Theorem \ref{thm:H1}.
Hence, we can directly deduce that
\begin{align*}
\sup_{t\in\left[0, 1\right]} \norm{ \widetilde{G}^{(L)}_t - \overbar{G}^{(L)}_t} &\leq L^{-1} \sup_{0\leq k < L} \sup_{t \in \left(t_k, t_{k+1}\right]} \norm{g_{k+1}^{(L)} \diag\left( \sigma'_d \hspace{-1pt} \left( \overbar{A}_{t_k} H_{t_k} + \overline{b}_{t_k} \right) \right) \overbar{A}_{t_k}  } \\
&\leq L^{-1} \sup_{0\leq k<L} \norm{ g_{k+1}^{(L)} } \,\, \sup_{t \in \left[0, 1\right]} \norm{\diag\left( \sigma'_d\hspace{-1pt}\left( \overbar{A}_t H_t + \overline{b}_t \right) \right) \overbar{A}_t}
\end{align*}
By continuity of $\overbar{A}$, $\overline{b}$ and $H$, the first supremum is finite and by a similar argument as in the proof of Theorem \ref{thm:H1}, the second supremum is also finite. Thus, there exists a constant $G_{\infty} > 0$ such that $ \sup_{t\in\left[0, 1\right]} \norm{ \widetilde{G}^{(L)}_t - \overbar{G}^{(L)}_t} \leq G_{\infty} L^{-1}$. Now, we also have, for $t\in (t_k,t_{k+1}]$,
\begin{align*} 
\widetilde{G}^{(L)}_t - G_t &= \widetilde{G}^{(L)}_{t_{k+1}} - G_{t_{k+1}} + (t_{k+1} - t) \, g_{k+1}^{(L)} \, \diag\left( \sigma'_d\hspace{-1pt} \left( \overbar{A}_{t_k} H_{t_k} + \overline{b}_{t_k} \right) \right) \overbar{A}_{t_k}  \\
&- \int_t^{t_{k+1}} G_s \diag\left( \sigma'_d\hspace{-1pt}\left( \overbar{A}_s H_s + \overline{b}_s \right) \right) \overbar{A}_s \dd s.
\end{align*}
Hence, for $e^{(L)}_k \coloneqq \sup_{t_k < t \leq t_{k+1}} \norm{ \widetilde{G}^{(L)}_t - G_t }$, we can estimate
\begin{align*}
e^{(L)}_k &\leq e^{(L)}_{k+1} + \int_t^{t_{k+1}} \norm{ G_s J_s - g_{k+1}^{(L)} J_{t_k} } \dd s. \\
&\leq e^{(L)}_{k+1} + \int_t^{t_{k+1}} \left( \norm{G_s - g_{k+1}^{(L)}} \norm{J_s} + \norm{g_{k+1}^{(L)}} \norm{ J_s - J_{t_k} }  \right) \dd s. \\
&\leq e^{(L)}_{k+1} + J_{\infty}  L^{-1} \left( e_k^{(L)} + G_{\infty} L^{-1} \right) + \int_t^{t_{k+1}}  \norm{g_{k+1}^{(L)}} \norm{ J_s - J_{t_k} }  \dd s,
\end{align*}
where $J_s \coloneqq \diag\left( \sigma'_d\hspace{-1pt}\left( \overbar{A}_s H_s + \overline{b}_s \right) \right) \overbar{A}_s$ and $J_{\infty} \coloneqq \sup_{s\in\left[0,1\right]} \norm{J_s} < \infty$. Now, recall that $\norm{g_{k+1}^{(L)}}$ is uniformly bounded in $k, L$, and we have $\overbar{A}, \overline{b} \in \HH^1$ and $H \in \CC^1$, so there exists a constant $J'_{\infty} < \infty$ such that $ \| g_{k+1}^{(L)} \| \int_t^{t_{k+1}} \norm{ J_s - J_{t_k} }  \dd s < J'_{\infty} L^{-2}$. Thus,
\[
\left( 1 - J_{\infty}  L^{-1} \right)  e^{(L)}_k \leq e^{(L)}_{k+1} + \left(J_{\infty} G_{\infty} + J'_{\infty} \right) L^{-2}.
\]
By Gronwall's lemma and the fact that $e_L^{(L)} = \OO(L^{-1})$, we deduce that $\max_k e_k^{(L)} = \OO(L^{-1})$ and conclude \[
\lim_{L\to\infty} \sup_{t\in\left[0,1\right]} \norm{ \widetilde{G}^{(L)}_t - G_t } \leq \lim_{L\to\infty} \left( \sup_{t\in\left[0,1\right]} \norm{ \widetilde{G}^{(L)}_t - \overbar{G}^{(L)}_t} + \max_k e_k^{(L)} \right) = 0.
\]

\subsubsection{Proof of Theorem~\ref{thm:backward-H2}}\label{sec:proof_backward-H2}
The ideas of the proof follow closely those of Theorem~\ref{thm:H2}, and we will provide here the main arguments for the case $\alpha=0$ and $\beta=1$. Other cases follow similarly. For the ease of notation exposition, we consider $U=0$ and we use $C$ to denote some generic constant (independent from $L$ and other parameters, such as $\varepsilon$, $\delta$, and $R$, to be defined later) that may vary from step to step. 

{The proof consists of 11 steps that can be summarized as follows. Step 1 decomposes the discrete gradient $g_k^{(L)}$ into two terms: the Jacobian of the output with respect to the input, and the Jacobian of the hidden state $h_k^{(L)}$ with respect to the input, which we denote by $J_k^{(L)}$. We then write a forward equation for $J_k^{(L)}$. Step 2 defines a continuous-time approximation $\widetilde{J}_k^{(L)}$ and a continuous-time interpolation $\overbar{J}_k^{(L)}$. Step 3 establishes a uniform bound $\OO(L^{-1})$ between $\widetilde{J}_k^{(L)}$ and $\overbar{J}_k^{(L)}$. Step 4 defines high-probability events under which the hidden states $h^{(L)}_k$ and the continuous-time limit $J_t$ are uniformly bounded. Step 5 decomposes the difference between $\widetilde{J}^{(L)}$ and $J$ with a drift term and an error term $D^{(L)}$, which can be further decomposed into a variance term $N^{(L)}$ and a Taylor remainder term $R^{(L)}$. Step 6 proves that $R^{(L)}$ uniformly vanishes as $\OO(L^{-1})$. Step 7 decomposes $N^{(L)}$ into three terms. Step 8 and 9 prove that these terms uniformly vanishes as $\OO(L^{-\min(1, \kappa)})$. Step 10 wraps everything together to show a uniform $L^2$ bound between $\widetilde{J}^{(L)}$ and $J^{(L)}$. Step 11 uses it to prove a uniform $L^1$ bound between the discrete gradients $g^{(L)}_k$ and their limit $G_t = J_1 J_t^{-1}$. \\ 
}

\noindent \underline{Step 0: Well-posedness of the statement.} The matrix-valued linear stochastic differential equation \eqref{eq:def:limit-J} has a continuous and adapted solution, and this solution is unique in the sense that almost all sample processes of any two solutions coincide, see for example \cite{E1979}. Furthermore, $\Prob-$a.s., $J_t$ is invertible for all $t\in\left[0,1\right]$, see Corollary 2.1 in \cite{D1991}. Also, Theorem 3.1 in \cite{Y2012} states that $K_t \coloneqq J_t^{-1}$ solves the following SDE.
\begin{equation*}
\dd K_t = K_t \left( -\nu(t,H_t) \dd t - \dd W_t^A + \dd \left[W^A \right]_t \right), \quad K_0 = I_d.
\end{equation*}
Recall from Assumption \ref{ass:ito} that the quadratic variation of $W^A$ is uniformly continuous with resepect to the Lesbegue measure. Therefore, by Lemma \ref{lem:linear-sde} and Assumption \ref{ass:backward}, we conclude that the fourth moments of the supremum of $J$ and $J^{-1}$ are finite.
\begin{equation} \label{eq:bound-j-j-inverse}
\E\left[ \sup_{t\in\left[0,T\right]} \max\left(\norm{J_t^{-1}}_F, \norm{J_t}_F\right)^4 \right] \leq C \,  \E\left[\exp\left( 8 \int_0^T \abs{ \tr \left( \nu(s, H_s) \right)} \dd s \right) \right]^{1/2} < \infty.
\end{equation}

\noindent \underline{Step 1: Rewrite the discrete backpropagation equation.} First, observe that multiplying \eqref{eq:backprop} together gives, for $k=0, \ldots, L$, 
\begin{equation} \label{eq:discrete-backprop}  
g_{0}^{(L),x} = g_{k}^{(L),x} \left[ \prod_{k'=k-1}^{0} \left( I_d + \diag\left( \sigma'_d \hspace{-1pt} \left( A_{k'}^{(L)} h_{k'}^{(L), x} + b_{k'}^{(L)} \right) \right) A^{(L)}_{k'} \right) \right]. 
\end{equation}
Define $J_0^{(L),x} \coloneqq I_d$ and for $k=0, \ldots, L-1$: 
\begin{equation} \label{eq:def-J}
    J_{k+1}^{(L), x} \coloneqq \left( I_d + \diag\left( \sigma'_d \hspace{-1pt} \left( A_{k}^{(L)} h_{k}^{(L), x} + b_{k}^{(L)} \right) \right) A^{(L)}_{k} \right) J_{k}^{(L), x} .
\end{equation}
Note that by the chain rule, we directly have $J_{k}^{(L), x} = \nabla_x h_k^{(L),x}$ and $g_{0}^{(L),x} = g_{k}^{(L),x} J_{k}^{(L), x} $. In the following, we omit the explicit dependence on the initial data $x$ when the context is clear. Recall now the definition $M_k^{(L)}(h) \coloneqq A_k^{(L)}h + b_k^{(L)}$ from \eqref{eq:def:MkL}. By Taylor's theorem on $\sigma'$, {as $\sigma'''$ is continuous}, for each $i=1, \ldots, d$, there exists $\abs{ \xi_{k, i}^{(L)}} < \abs{ M_k^{(L)}\Big(h_k^{(L)}\Big)_i }$ such that
\begin{align*}
    \sigma'\left( M_k^{(L)}\Big(h_k^{(L)}\Big)_i \right) &= \sigma'(0) + \sigma''(0) M_k^{(L)}\Big(h_k^{(L)}\Big)_i + \frac{1}{2} \sigma''' \hspace{-2pt} \left( \xi_{k, i}^{(L)} \right) M_k^{(L)}\Big(h_k^{(L)}\Big)_i^2.
\end{align*}
Hence, using $\sigma'(0)=1$, $\Delta_L = L^{-1}$, and \eqref{eq:hypothesis-2-recall}, we get 
\begin{align}
    J_{k+1}^{(L)} &= \left( I_d + \diag\left( \sigma'_d \hspace{-1pt} \left( M_k^{(L)}\Big(h_k^{(L)}\Big) \right) \right) A^{(L)}_{k} \right) J_{k}^{(L)} \nonumber \\
    &= \left( I_d + A^{(L)}_{k} + \sigma''(0) \diag\left( M_k^{(L)}\Big(h_k^{(L)}\Big) \right) A^{(L)}_{k} \right) J_{k}^{(L)}  \nonumber \\
    &\quad + \diag\left( \frac{1}{2} \sigma''' \hspace{-2pt} \left( \xi_{k}^{(L)} \right) \odot M_k^{(L)}\Big(h_k^{(L)}\Big)^{\odot, \, 2} \right) A_k^{(L)} J_{k}^{(L)}  \nonumber \\
    &= \Big( I_d + \Big( \underbrace{\overbar{A}_{t_k} + \frac{1}{2} \sigma''(0) \nabla_h Q\Big( t_k, h_k^{(L)}\Big)}_{\eqqcolon \, \nu\big(t_k, h_k^{(L)}\big)} \Big) \Delta_L + \Delta W^A_k \Big) J_{k}^{(L)} \nonumber \\
    &\quad +  \Big( \sigma''(0) \Big( \underbrace{ \diag\left( M_k^{(L)}\Big(h_k^{(L)}\Big) \right) A^{(L)}_{k} - \frac{1}{2} \nabla_h Q\Big( t_k, h_k^{(L)}\Big) \Delta_L}_{\eqqcolon \,\, N^{(L)}_{k}\big(J_k^{(L)},\, h_k^{(L)}\big)} \Big) \Big) J_{k}^{(L)} \nonumber \\
    &\quad + \underbrace{ \left( \diag\left( \frac{1}{2} \sigma''' \hspace{-2pt} \left( \xi_{k}^{(L)} \right) \odot M_k^{(L)}\Big(h_k^{(L)}\Big)^{\odot, \, 2} \right) \right) A_k^{(L)} J_{k}^{(L)} }_{\eqqcolon \,\, R^{(L)}_{k}\big(J_k^{(L)}, \, h_k^{(L)}\big)} \nonumber \\
    &= J_{k}^{(L)} + \nu\left( t_k, h_k^{(L)} \right) J_{k}^{(L)} + \Delta W^A_k J_{k}^{(L)} + D_k^{(L)}\left(J_k^{(L)}, h_k^{(L)}\right), \label{eq:discrete-time-J}
\end{align}
where we define the error term $D_k^{(L)} \coloneqq \sigma''(0) N^{(L)}_{k} J^{(L)}_{k} + R^{(L)}_{k}$. \\

\noindent \underline{Step 2: Continuous-time approximation.} Recall the (forward) SDE defined in the statement of the theorem 
\[
    \dd J_t = \nu(t, H_t) J_t \dd t + \dd W_t^A J_t, \quad J_0 = I_d.
\]
Recall the definition of $\overbar{H}^{(L)}$ in \eqref{eq:cte}, and define similarly the \textit{continuous-time extension} (CTE) of $\big\{ J_k^{(L)} \colon k = 0, \ldots, L \big\}$:
\begin{equation} \label{def:cte-J}
    \overbar{J}_t^{(L)} \coloneqq \sum_{k=0}^{L} J_k^{(L)} {\bf 1}_{t_{k}\leq t <t_{k+1}}.  
\end{equation}
Let $k_s$ the index for which $t_{k_s} \leq s < t_{k_s+1}$. Define the \textit{continuous-time approximation} (CTA) of $J_t$ as 
\begin{equation} \label{eq:cta-J}
    \widetilde{J}_t^{(L)} \coloneqq I_d + \int_0^t  \nu \hspace{-1pt} \left( t_{k_s}, \overbar{H}^{(L)}_s \right) \overbar{J}_s^{(L)} \dd s + \int_0^t \dd W_s^A \overbar{J}_s^{(L)}  + \sum_{k < Lt} D_k^{(L)}\hspace{-2pt} \left( J_k^{(L)}, h_k^{(L)} \right).
\end{equation}

\noindent \underline{Step 3: Uniform bound between $\overbar{J}^{(L)}$ and $\widetilde{J}^{(L)}$.} Using \eqref{eq:discrete-time-J} and \eqref{eq:cta-J}, we have, for $s\in\left[ 0, 1 \right]$, 
\begin{align*}
\norm{ \widetilde{J}^{(L)}_s - \overbar{J}^{(L)}_s }^2 &= \norm{ \left(  \nu\hspace{-1pt} \left(t_{k_s}, h_{k_s}^{(L)}\right) (s - t_{k_s}) + \left( W^A_{s} - W^A_{t_{k_s}} \right) \right) J_{k_s}^{(L)} }^2 \\
&\leq \left( C \left(1 + \sup_k \norm{h_k^{(L)}}^2 \right) (\Delta_L)^2 + 2 \norm{W^A_{s} - W^A_{t_{k_s}}}^2 \right) \norm{ J_{k_s}^{(L)}}^2 
\end{align*}
Hence, 
\begin{align*}
    &\E\left[ \sup_{0\leq s \leq 1} \norm{ \widetilde{J}^{(L)}_s - \overbar{J}^{(L)}_s }^2  \right] \\
    &\leq \left( C \left( 1 + \E\left[ \sup_k \norm{h_k^{(L)}}^4 \right]^{1/2}  \right) (\Delta_L)^2 + \E\left[ \sup_{0\leq s \leq 1} \norm{ W^A_{s} - W^A_{t_{k_s}} }^4 \right]^{1/2} \right) \E\left[ \sup_k \norm{J_k^{(L)}}^4 \right]^{1/2} .  
\end{align*}
By Assumptions~\ref{ass:strong} and \ref{ass:backward}, and equation~\eqref{eq:bound16}:
\begin{equation} \label{eq:unif-bound-bar-tilde}
\E\left[ \sup_{0\leq s \leq 1} \norm{ \widetilde{J}^{(L)}_s - \overbar{J}^{(L)}_s }^2  \right] < C L^{-1}.
\end{equation}

\noindent \underline{Step 4: Initial computations for a uniform $L^1$ bound between $G$ and $\overbar{G}^{(L)}$.} Fix $\epsilon > 0$, and let $\delta > 0$ (to be determined later) that only depends on $L$ and $\epsilon$. Define for $R>1$
\begin{equation} \label{def:event-thm5}
    E^{(L)}_R \coloneqq \left\{ \sup_{k\leq L} \norm{h_k^{(L)}} \leq R \right\} \cap \left\{ \sup_{t\in\left[0,1\right]} \norm{J_t} \leq R \right\}.
\end{equation}
Using Assumption~\ref{ass:strong} and \eqref{eq:bound-j-j-inverse}, we obtain similarly to \eqref{eq:bound4} that 
\begin{equation} \label{eq:bound-event-c}
\Prob\left( ( E_R^{(L)})^c \right) \leq \left( \E\left[ \sup_{t\in\left[0,1\right]} \norm{\overbar{H}_t}^4 \right] + \E\left[ \sup_{t\in\left[0,1\right]} \norm{J_t}^4 \right]  \right) R^{-4} \leq C R^{-4}.
\end{equation}
Now, by Cauchy-Schwarz inequality, we have
\begin{equation*}
ab = \delta^{1/2}a \cdot \delta^{-1/2}b \leq \frac{\delta}{2} a^{2} +\frac{1}{2\delta}b^2, \quad \forall a,b,\delta>0.
\end{equation*}
We use it to decompose the $L^1$ distance between $G$ and $\overbar{G}^{(L)}$:
\begin{align*}
\mathbb{E}\left[ \sup_{0 \leq t \leq 1} \norm{G_t - \overbar{G}^{(L)}_t} \right] &= \mathbb{E}\left[ \sup_{0 \leq t \leq 1} \norm{G_t - \overbar{G}^{(L)}_t} \one_{E^{(L)}_R} \right] + \mathbb{E}\left[ \sup_{0 \leq t \leq 1} \norm{G_t - \overbar{G}^{(L)}_t} \one_{(E^{(L)}_R)^c} \right] \nonumber \\
&\leq \mathbb{E}\left[ \sup_{0 \leq t \leq 1} \norm{G_t - \overbar{G}^{(L)}_t} \one_{E^{(L)}_R} \right] + \frac{\delta}{2} \mathbb{E}\left[ \sup_{0 \leq t \leq 1} \norm{G_t - \overbar{G}^{(L)}_t}^2 \right] \nonumber \\
&\quad+ \frac{1}{2\delta} \Prob\left( ( E_R^{(L)})^c \right).
\end{align*}
Now, we have the following estimate
\begin{equation*}
\mathbb{E}\left[ \sup_{0 \leq t \leq 1} \norm{G_t - \overbar{G}^{(L)}_t}^2 \right] \leq 2 \E\left[ \sup_{0 \leq t \leq 1} \norm{J_t^{-1}}^4 \right]^{1/2} \E\left[ \norm{J_1}^4 \right]^{1/2} + 2 \E\left[ \sup_{0 \leq t \leq 1} \norm{\overbar{G}^{(L)}_t}^4 \right]^{1/2}.
\end{equation*}
Therefore, by Assumption \ref{ass:backward} and \eqref{eq:bound-j-j-inverse},
\begin{equation} \label{eq:bound-eg}
\mathbb{E}\left[ \sup_{0 \leq t \leq 1} \norm{G_t - \overbar{G}^{(L)}_t} \right] \leq \mathbb{E}\left[ \sup_{0 \leq t \leq 1} \norm{G_t - \overbar{G}^{(L)}_t} \one_{E^{(L)}_R} \right] + C \delta + \frac{C}{\delta R^4}.
\end{equation}

\noindent \underline{Step 5: Initial computations for a uniform bound $L^2$ between $J$ and $\widetilde{J}^{(L)}$.} First we estimate, for $t\in\left[0,1\right]$,
\begin{align}
    \norm{J_t - \widetilde{J}^{(L)}_t}^2 &\leq 3 \int_0^t \norm{  \nu\left( t_{k_s}, \overbar{H}_s^{(L)} \right) \overbar{J}_s^{(L)} -  \nu\left( s, H_s \right) J_s }^2 \dd s + 3 \norm{ \int_0^t  \dd W^A_s \left(\overbar{J}_s^{(L)} - J_s \right) }^2 \nonumber \\
    &+ 3 \norm{ \sum_{k < Lt} D_k^{(L)} \hspace{-2pt} \left( J_{k}^{(L)}, h_{k}^{(L)} \right) }^2. \label{eq:unif-bound-tilde-1}
\end{align}
The goal is to find an upper bound of the first two terms, consisting of the sum of the $L^2$ distance between $J$ and $\overbar{J}^{(L)}$ and terms vanishing uniformly in $L$. We also want to show that the error term $D^{(L)}$ uniformly vanishes in $L$. To handle the term involving the drift $\nu$, we first observe that for $t_1, t_2 \in \left[0, T\right]$, $h_1, h_2 \in \R^{d}$, and $J_1, J_2 \in \R^{d \times d}$, we have
\begin{align*}
    \norm{ \nu(t_2, h_2) J_2 - \nu(t_1, h_1) J_1 }^2 &\leq 3 \norm{\nu(t_2, h_2)}^2 \norm{J_2 - J_1}^2 + 3  \norm{ \nu(t_2, h_2) - \nu(t_2, h_1)}^2 \norm{J_1}^2 \\
    &\quad + 3 \norm{\nu(t_2, h_1) - \nu(t_1, h_1)}^2 \norm{J_1}^2. \\
    &\leq C \left(1 + \norm{h_2}^2 \right) \norm{J_2 - J_1}^2 + C \norm{h_2 - h_1}^2 \norm{J_1}^2 \\
    &\quad + C \norm{h_1}^2 \norm{J_1}^2 \abs{t_2 - t_1}^{\kappa}.
\end{align*}
We used the fact that $\nu$ is linear in $h$ and $\Sigma^A_{\cdot}$ is $\kappa / 2-$H\"older continuous. We directly deduce that 
\begin{align}
&\E\left[  \one_{E^{(L)}_R} \int_0^T \norm{ \nu\left( t_{k_s}, \overbar{H}_s^{(L)} \right) \overbar{J}_s^{(L)} - \nu\left( s, H_s \right) J_s }^2 \dd s \right] \nonumber \\
&\leq C(1+R^2) \, \E\left[ \one_{E^{(L)}_R} \int_0^T \norm{\overbar{J}_s^{(L)} - J_s}^2 \dd s \right] + C R^2 \, \E\left[ \int_0^T \norm{\overbar{H}_s^{(L)} - H_s }^2 \dd s \right] \\
&\quad + CR^2 \E\left[ \int_0^T \norm{ H_s }^2 \dd s \right] L^{-\kappa} \nonumber \\ 
&\leq CR^2 \, \E\left[ \one_{E^{(L)}_R} \int_0^T \norm{\overbar{J}_s^{(L)} - J_s}^2 \dd s  \right] + CR^2 L^{-\min(1/2, \, \kappa)}. \label{eq:bound-mu} 
\end{align}
The last inequality holds by Theorem \ref{thm:H2}. Hence, we obtain from \eqref{eq:unif-bound-tilde-1} 
\begin{align}
\E\left[ \sup_{0\leq t \leq 1} \norm{J_t - \widetilde{J}^{(L)}_t}^2  \one_{E^{(L)}_R} \right] &\leq CR^2 \, \E\left[ \one_{E^{(L)}_R} \int_0^T \norm{\overbar{J}_s^{(L)} - J_s}^2 \dd s \right] + CR^2 L^{-\min(1/2, \, \kappa)} \nonumber \\
&+ 3 \, \E\left[ \sup_{0\leq t\leq 1} \norm{ \sum_{k < Lt} D_k^{(L)} \hspace{-2pt} \left(J_{k}^{(L)},  h_{k}^{(L)} \right) }^2  \one_{E^{(L)}_R} \right]. \label{eq:unif-bound-tilde-2}
\end{align}
We applied Doob's martingale inequality~\cite{revuz2013continuous} on the second term of \eqref{eq:unif-bound-tilde-1}, as $\overbar{J}$, $J$, and $E_R^{(L)}$ are adapted to the filtration generated by $W^A$. We now estimate the error term $D^{(L)}$ in \eqref{eq:unif-bound-tilde-2}. Recall that it decomposes into a variance term $N^{(L)}$ and a Taylor remainder term $R^{(L)}$.

\noindent \underline{Step 6: Prove that the remainder $R^{(L)}_k$ uniformly vanishes.} We proceed to show that
\begin{equation} \label{eq:remainder-taylor}
    \E\left[ \sup_{0\leq t\leq 1} \norm{ \sum_{k < Lt} R_k^{(L)} \hspace{-2pt} \left(J_{k}^{(L)},  h_{k}^{(L)} \right) }^2  \one_{E^{(L)}_R} \right] \leq C R^6 L^{-1},
\end{equation}
which is straightforward since:
\begin{align*}
    &\E\left[ \sup_{0\leq t\leq 1} \norm{ \sum_{k < Lt} R_k^{(L)} \hspace{-2pt} \left(J_{k}^{(L)},  h_{k}^{(L)} \right) }^2  \one_{E^{(L)}_R} \right] \leq \E\left[ \left( \sum_{k=0}^{L-1} \norm{ R_k^{(L)} \hspace{-2pt} \left(J_{k}^{(L)},  h_{k}^{(L)} \right) } \right)^2 \one_{E^{(L)}_R} \right] \\
    &\leq CR^2 \, \E\left[ \left( \sum_{k=0}^{L-1} \left( \norm{A_k^{(L)}} R + \norm{b_k^{(L)}} \right)^2  \norm{A_k^{(L)}} \right)^2 \right] \\
    &\leq CR^2 L \, \E\left[ \sum_{k=0}^{L-1} \left( \norm{A_k^{(L)}} R + \norm{b_k^{(L)}} \right)^4 \norm{A_k^{(L)}}^2 \right] \leq C R^6 L^{-1}.
\end{align*}
The last inequality holds for the same reasons as \eqref{eq:bound10}.

\noindent \underline{Step 7: Prove that the remainder $N^{(L)}_k$ uniformly vanishes.} First note that we can write $N^{(L)}_{k, ij} = N^{(L), 0}_{k, ij} + \sum_{m=1}^d N^{(L), 1}_{k, ij, m}$, where 
\begin{align}
    N^{(L), 0}_{k, ij} &\coloneqq b^{(L)}_{k,i} A^{(L)}_{k, ij} \nonumber \\
    N^{(L), 1}_{k, ij, m} &\coloneqq \left( A^{(L)}_{k, im} A^{(L)}_{k, ij} - \left(\Sigma^A_{t_k}\right)_{imij} \Delta_L \right) h^{(L)}_{k, m} \label{eq:N-kijl}
\end{align}
We assumed that the Ito processes $W^A$ and $W^b$ are driven by uncorrelated Brownian motions, hence $N^{(L), 0}_{k, ij}$ uniformly vanishes in $L^2$ at rate $\Delta_L$. 
Thus, we get
\begin{align} \label{eq:estim-N}
\E \left[ \sup_{0\leq t \leq T}  \abs{ \sum_{k < Lt} N_{k, ij}^{(L)} }^2 \one_{E^{(L)}_R} \right] &\leq CL^{-1} + d \sum_{m=1}^d \E\left[ \sup_{0\leq t \leq T}  \abs{ \sum_{k < Lt} N_{k, ij, m}^{(L), 1} }^2 \one_{E^{(L)}_R} \right]. \\
\end{align}
Using the discrete (forward) filtration $\big\{ \GG_k : k=-1, 0,\ldots, L-1 \big\}$ defined in \eqref{def:filtration-G}, we now expand \eqref{eq:N-kijl} using the definition of Scaling regime \ref{hypothesis.2}.
\begin{align*}
    N^{(L), 1}_{k, ij, m} &= \underbrace{ \left( \left( \Delta W^A_k \right)_{im} \left( \Delta W^A_k \right)_{ij} - \int_{t_k}^{t_{k+1}} \expect*{ \left(\Sigma^A_s\right)_{imij} | \GG_{k-1}}  \dd s \right) h^{(L)}_{k,m}}_{\tcircle{1}}  \\
    &+ \underbrace{ h^{(L)}_{k,m} \int_{t_k}^{t_{k+1}} \expect*{ \left(\Sigma^A_s - \Sigma^A_{t_k} \right)_{imij} | \GG_{k-1}}  \dd s }_{\tcircle{2}} \\
    &+ \underbrace{ \left[ \left( \overbar{A}_{t_k}\right)_{im} \left( \Delta W^A_k \right)_{ij} + \left( \overbar{A}_{t_k}\right)_{ij} \left( \Delta W^A_k \right)_{im}  \right] h^{(L)}_{k,m} \Delta_L + \left( \overbar{A}_{t_k}\right)_{im} \left( \overbar{A}_{t_k}\right)_{ij} h^{(L)}_{k,m} (\Delta_L)^2 }_{\tcircle{3}}. 
\end{align*}

\noindent \underline{Step 8: Prove that the term $\tcircle{1}$ uniformly vanishes.} Define 
\[
X^{(L)}_{k, ij, m} \coloneqq \left(\Delta W^A_k \right)_{im} \left( \Delta W^A_k \right)_{ij} - \int_{t_k}^{t_{k+1}} \expect*{ \left(\Sigma^A_s\right)_{imij} | \GG_{k-1}}  \dd s,
\] 
and $S^{(L)}_{k, ij, m} \coloneqq \sum_{k'=0}^k X^{(L)}_{k', ij, m}$. Observe that $\big\{ S^{(L)}_{k, ij, m} \colon k=0, \ldots, L \big\}$ is a $(\GG_k)-$martingale, where the filtration $\GG_{k}$ is defined in \eqref{def:filtration-G}. Hence, by Doob's martingale inequality, we have
\begin{equation} \label{eq:doob-x}
\E \left[ \sup_{0\leq t \leq T}  \abs{ \sum_{k < Lt} X_{k, ij, m}^{(L)} }^2 \right] = \E \left[ \sup_{0\leq t \leq T}  \abs{ S_{\floor{Lt}, ij, m}^{(L)} }^2 \right] \leq 4 \, \E \left[ \abs{ S_{\floor{LT}, ij, m}^{(L)} }^2 \right].
\end{equation}
Fix $k=0, \ldots, L-1$ and compute the following conditional expectation.
\begin{align}
\expect*{ \left( S^{(L)}_{k, ij, m} \right)^2 | \GG_{k-1}} &= \expect*{ \left( S^{(L)}_{k-1, ij, m} \right)^2 + 2 X^{(L)}_{k,ij,m} S^{(L)}_{k-1,ij,m} + \left( X^{(L)}_{k, ij, m} \right)^2 | \GG_{k-1}} \nonumber \\
&= \left( S^{(L)}_{k-1, ij, m} \right)^2 + \expect*{ \left( X^{(L)}_{k, ij, m} \right)^2 | \GG_{k-1}}.
\end{align}
The cross-term disappear as $\expect*{ X^{(L)}_{k, ij, m} | \GG_{k-1}} = 0$. Furthermore, conditionally on $\GG_{k-1}$, observe that $\big( X^{(L)}_{k, ij, m} \big)^2$ is the variance of a product of two normal random variable with $O(L^{-1})$ variance, uniformly in $k$ by \eqref{eq:bound_ito}, so \[
\sup_{0\leq k < L} \expect*{ \left( X^{(L)}_{k, ij, m} \right)^2 | \GG_{k-1}} \leq C L^{-2}.
\] 
Hence, plugging it back into \eqref{eq:doob-x}, we obtain 
\begin{equation} \label{eq:unif-bound-1}
\E \left[ \sup_{0\leq t \leq T}  \abs{ \sum_{k < Lt} \tcircle{1} \, }^2 \one_{E^{(L)}_R} \right] \leq R^2 \, \E \left[ \sup_{0\leq t \leq T}  \abs{ \sum_{k < Lt} X_{k, ij, m}^{(L)} }^2 \right] \leq C R^2 L^{-1}.
\end{equation}

\noindent \underline{Step 9: Prove that the terms $\tcircle{2}-\tcircle{3}$ uniformly vanishes.} The term $\tcircle{2}$ can be estimated directly using Cauchy-Schwarz, Tonelli, and \eqref{eq:continuity_ito}: 
\begin{align}
\E\left[ \sup_{0\leq t \leq T}  \abs{ \sum_{k < Lt} \tcircle{2} \, }^2 \one_{E^{(L)}_R} \right] &\leq R^2 \, \E\left[ \left( \sum_{k=0}^{L} \int_{t_k}^{t_{k+1}} \expect*{ \abs{ \left(\Sigma^A_s - \Sigma^A_{t_k} \right)_{imij}} | \GG_{k-1}} \dd s \right)^2  \right] \nonumber \\
&\leq C R^2 \left( \sum_{k=0}^{L} \int_{t_k}^{t_{k+1}} \abs{s - t_{k_s}}^{\kappa/2} \dd s \right)^2 \leq CR^2 L^{-\kappa}. \label{eq:unif-bound-2}
\end{align}
The estimation for term $\tcircle{3}$ is straightforward and similar to \eqref{eq:bound10}:
\begin{equation} \label{eq:unif-bound-3}
    \E\left[ \sup_{0\leq t \leq T}  \abs{ \sum_{k < Lt} \tcircle{3} \, }^2 \one_{E^{(L)}_R} \right] \leq C R^2 L^{-1}.
\end{equation}

\noindent \underline{Step 10: Uniform bound between $J$ and $\widetilde{J}^{(L)}$.} From equations \eqref{eq:remainder-taylor} \eqref{eq:estim-N}, \eqref{eq:unif-bound-1}, \eqref{eq:unif-bound-2}, and \eqref{eq:unif-bound-3}, we deduce that 
\begin{equation} \label{eq:remainder-bound}
\E\left[ \sup_{0\leq t\leq 1} \norm{ \sum_{k < Lt} D_k^{(L)} \hspace{-2pt} \left(J_{k}^{(L)},  h_{k}^{(L)} \right) }^2  \one_{E^{(L)}_R} \right] \leq C R^6 L^{-\min(1, \kappa)}.
\end{equation}
We then plug \eqref{eq:remainder-bound} into \eqref{eq:unif-bound-tilde-2}, together with Tonelli's theorem, to get
\begin{align*}
\E\left[ \sup_{0\leq t \leq 1} \norm{J_t - \widetilde{J}^{(L)}_t}^2  \one_{E^{(L)}_R} \right] &\leq C R^2 \, \E\left[ \one_{E^{(L)}_R} \int_0^T \norm{\overbar{J}_s^{(L)} - J_s}^2 \dd s \right] + C R^6 L^{-\min(1/2, \, \kappa)} \\ 
&\leq C R^2 \, \E\left[ \one_{E^{(L)}_R} \int_0^T \norm{\widetilde{J}_s^{(L)} - J_s}^2 \dd s \right] + C R^6 L^{-\min(1/2, \, \kappa)}.
\end{align*}
We use \eqref{eq:unif-bound-bar-tilde} for the last inequality. Hence, by Gronwall lemma, we deduce:
\begin{equation} \label{eq:bound-j-j-tilde}
\E\left[ \sup_{0\leq t \leq 1} \norm{J_t - \widetilde{J}^{(L)}_t}^2  \one_{E^{(L)}_R} \right] \leq C R^6 L^{-\min(1/2, \, \kappa)} \exp\big(CR^2\big).
\end{equation}

\noindent \underline{Step 11: Difference between $G$ and $g$.}
We first estimate the $L^1$ distance between the discrete gradients $g_k^{(L)}$ and the continuous-time limit $G_t$. For each $t\in\left[0,1\right]$, we have the identity
\begin{align*}
\overbar{G}_t^{(L)} - G_t &= \left( J_L^{(L)} - J_1 \right) J_t^{-1} + g_{k_t}^{(L)} - J_L^{(L)} J_t^{-1} \\ 
&= \left( J_L^{(L)} - J_1 \right) J_t^{-1} + g_{k_t}^{(L)} \left( J_t - J_{k_t}^{(L)} \right) J_t^{-1}.
\end{align*}
Hence, by Assumption \ref{ass:backward}, \eqref{eq:bound-j-j-inverse}, \eqref{eq:unif-bound-bar-tilde}, and \eqref{eq:bound-j-j-tilde}:
\begin{align*}
&\E\left[ \sup_{0\leq t \leq 1} \norm{G_t - \overbar{G}^{(L)}_t}  \one_{E^{(L)}_R} \right] \\
&\leq \E\left[ \norm{\overbar{J}^{(L)}_1 - J_1}^2 \one_{E^{(L)}_R} \right]^{1/2} \E\left[ \sup_{0\leq t \leq 1} \norm{(J_t)^{-1}}^2 \right]^{1/2}  \\
&\quad + \E\left[ \sup_{0\leq t \leq 1} \norm{\overbar{G}_t}^4 \right]^{1/4} \E\left[ \sup_{0\leq t \leq 1} \norm{J_t - \overbar{J}^{(L)}_t}^2 \one_{E^{(L)}_R} \right]^{1/2} \E\left[ \sup_{0\leq t \leq 1} \norm{(J_t)^{-1}}^4 \right]^{1/4}. \\
&\leq C R^3 L^{-\min(1/4, \, \kappa/2)} \exp\big(CR^2\big).
\end{align*}
We plug it in \eqref{eq:bound-eg} to obtain 
\begin{equation} \label{eq:final-bound-g}
\mathbb{E}\left[ \sup_{0 \leq t \leq 1} \norm{G_t - \overbar{G}^{(L)}_t} \right] \leq C_1 R^3 L^{-\min(1/4, \, \kappa/2)} \exp\big(C_2R^2\big) + C_3 \delta + \frac{C_4}{\delta R^4}.
\end{equation}
To conclude, given any $\epsilon > 0$, we can choose $\delta > 0$ such that $\delta < \frac{\epsilon}{3C_3}$, and then choose $R>1$ so that $R^4 > \frac{3C_4}{\delta \epsilon}$, and finally $L$ sufficiently large so that
\begin{equation*}
     C_1 R^3 L^{-\min(1/4, \, \kappa/2)} \exp\big(C_2R^2\big) < \frac{\epsilon}{3}.
\end{equation*}
Therefore, we have in \eqref{eq:final-bound-g}
\begin{equation*}
    \mathbb{E}\left[ \sup_{0 \leq t \leq 1} \norm{G_t - \overbar{G}^{(L)}_t} \right] \leq \epsilon.
\end{equation*}

\bibliographystyle{siam}
\bibliography{ref}

\newpage

\appendix
\section{Hyperparameters \label{sec:hyperparameters}}

We provide in Table~\ref{tab:hyperparameters} the training hyperparameters used in our numerical experiments. In Table~\ref{tab:description_hyperparameters}, we give a short description of each hyperparameter. For the convolutional architecture, we also use a momentum of 0.9, a weight decay of $0.0005$ and a cosine annealing learning rate scheduler~\cite{DBLP:journals/corr/LoshchilovH16a}.

{
    \vspace{0.2cm}
    \tabulinesep=1.3mm
        \captionof{table}{Training hyperparameters.}
    \begin{center}
    \begin{tabu}{c|c|c|c|c|c|c|c|c|c}
        \hline
        Dataset & Layers & $N$ & $B$ & $\lr$ & $L_{\min}$ & $L_{\max}$ & $T_{\max}$ & $N_\text{epochs}$ & $\epsilon$ \\ \hline \hline
        Synthetic & Fully-connected & 1,024 & 32 & 0.01 & 3 & 10,321 & 160 & 5 & 0.01 \\ \hline
        MNIST & Fully-connected & 60,000 & 50 & 0.01 & 3 & 942 & 12,000 & 10 & 0.01  \\ \hline
        CIFAR-10 & Convolutional & 60,000 & 128 & 0.1 & 8 & 121 & 93,800 & 200 & None \\ \hline
    \end{tabu}
    \end{center}
    \label{tab:hyperparameters}

    \vspace{0.2cm}
}

{
    \vspace{0.2cm}
    \tabulinesep=1.3mm
        \captionof{table}{Description of the values in Table~\ref{tab:hyperparameters}. Note that $T_{\max} = \ceil{ \frac{N}{B} } N_\text{epochs}$.}
    \begin{center}
        \begin{tabu}{c|l}
            \hline
            Parameter & Description \\  \hline
            $N$ & number of training samples \\
            $B$ & minibatch size \\
            $\lr$ & learning rate \\
            $L_{\min}$ & smallest network depth \\
            $L_{\max}$ & largest network depth \\
            $T_{\max}$ & max number of SGD updates \\
            $N_\text{epochs}$ & max number of epochs \\
            $\epsilon$ & early stopping value \\
            \hline
        \end{tabu}
    \end{center}

    \label{tab:description_hyperparameters}

    \vspace{0.2cm}
}

\section{Auxiliary lemma \label{sec:additional-lemma}}

\begin{lem} \label{lem:linear-sde}
Let $\left(Y_t\right)_{t\in\left[0,T\right]} \subset \R^{d\times d}$ be a continuous semimartingale that can be decomposed as $\dd Y_t = A_t \dd t + \dd M_t$, where $A$ is a square-integrable adapted process, $M$ is a continuous square-integrable martingale with quadratic variation $\dd \big[M, M^{\top}\big]_t = Q_t \, \dd t$, and $\sup_{0\leq t \leq T} \norm{Q_t}_F < Q_{\infty} < \infty$, where $Q_{\infty}$ is a deterministic constant. Let $\left(X_t\right)_{t\in\left[0,T\right]} \subset \R^{d\times d}$ be the unique solution to the linear matrix-valued SDE $\dd X_t = X_t \dd Y_t$, with $X_0$ being a deterministic non-zero matrix. Then, for each $p>1$, there exists a constant $C\equiv C(p, d, Q_{\infty}, X_0, T)$ such that \[
\E\left[ \sup_{t\in\left[0,T\right]} \norm{X_t}_F^p \right] \leq C \,  \E\left[\exp\left( 2p \int_0^T \abs{ \tr \left( A_s \right)} \dd s \right) \right]^{1/2}.
\]
\end{lem}

\begin{proof}
We apply the multidimensional Ito formula and linearity of the trace operator to first get
\begin{align*}
\dd \norm{X_t}_F^2 = \dd \, \tr(X_t^{\top} X_t) &= \tr \left( \dd X_t^{\top} X_t + X_t^{\top} \dd X_t + \dd \left[X^{\top}, X \right]_t \right) \\
&= \tr \left( X_t^{\top} X_t \left(\dd Y_t + \dd Y_t^{\top} \right) + \dd \left[X^{\top}, X \right]_t \right) \\
\end{align*}
Now, by cyclic permutation invariance of the trace, we have 
\begin{align*}
\tr\left( \dd \left[X^{\top}, X \right]_t \right) = \tr\left( \dd \left[X, X^{\top} \right]_t \right) &= \tr\left( X_t \, \dd \left[Y, Y^{\top} \right]_t X_t^{\top} \right) \\
&= \tr\left( X_t^{\top} X_t \, \dd \left[Y, Y^{\top} \right]_t \right) \\
&= \tr\left( X_t^{\top} X_t Q_t \, \dd t \right).   
\end{align*}
Therefore, 
\begin{equation} \label{eq:dx-frob}
\dd \norm{X_t}_F^2 = \tr \left( X_t^{\top} X_t \left(A_t + A_t^{\top} + Q_t \right) \right) \dd t + \norm{X_t}_F^2 \dd N_t,
\end{equation}
where
\begin{equation}
N_t \coloneqq \tr\left( \int_0^t \frac{X_s^{\top} X_s}{\norm{X_s}_F^2} \left(\dd M_s + \dd M_s^{\top} \right) \right)
\end{equation}
is a  martingale with quadratic variation  given by
\begin{equation*}
\left[ N \right]_t = \sum_{i_1,j_1,i_2,j_2} \int_0^t \norm{X_s}_F^{-4} \left( X^{\top}_s X_s \right)_{i_1 j_1} \left( X^{\top}_s X_s \right)_{i_2 j_2} \dd \left[ (M+M^{\top})_{i_1 j_1}, (M+M^{\top})_{i_2 j_2} \right]_s  \\
\end{equation*}
By  the Kunita-Watanabe inequality, 
\begin{align}
\left[ N \right]_t &\leq \sum_{i_1,j_1,i_2,j_2} \left( \int_0^t \norm{X_s}_F^{-4} \left( X^{\top}_s X_s \right)^2_{i_1 j_1} \dd \left[ (M+M^{\top})_{i_1 j_1} \right]_s \right)^{1/2} \cdot \nonumber \\
&\qquad \left( \int_0^t \norm{X_s}_F^{-4} \left( X^{\top}_s X_s \right)^2_{i_2 j_2} \dd \left[ (M+M^{\top})_{i_2 j_2} \right]_s \right)^{1/2} \nonumber \\
&=  \left( \sum_{i,j} \left( \int_0^t \norm{X_s}_F^{-4} \left( X^{\top}_s X_s \right)^2_{ij} \dd \left[ (M+M^{\top})_{ij} \right]_s \right)^{1/2} \right)^2 \nonumber \\
&\leq d^2 \sum_{i,j} \int_0^t \norm{X_s}_F^{-4} \left( X^{\top}_s X_s \right)^2_{ij} \dd \left[ (M+M^{\top})_{ij} \right]_s \nonumber \\
&\leq 4 d^2 Q_{\infty} \int_0^t \norm{X_s}_F^{-4} \norm{X_s^{\top} X_s}_F^2 \dd s \leq 4 d^2 Q_{\infty} t. \label{eq:qv-N-bound}
\end{align}
The second inequality follows from Cauchy-Schwarz. Now, by conditioning on $\norm{X_t}_F > 0$ if necessary, we have by  the Ito's formula and \eqref{eq:dx-frob}
\begin{align*}
\dd \log \norm{X_t}_F^2 &= \norm{X_t}_F^{-2} \dd \norm{X_t}_F^2 - \frac{1}{2} \norm{X_t}_F^{-4} \dd \left[ \norm{X}_F^2 \right]_t \\
&= \norm{X_t}_F^{-2} \tr \left( X_t^{\top} X_t \left(A_t + A_t^{\top} + Q_t \right) \right) \dd t + \dd N_t  - \frac{1}{2} \dd \left[ N \right]_t.
\end{align*}
Hence, by integrating and taking the exponential, we get
\begin{align*}
\norm{X_t}_F^2 &= \norm{X_0}_F^2 \exp\left( \int_0^t \norm{X_s}_F^{-2} \tr \left( X_s^{\top} X_s \left(A_s + A_s^{\top} + Q_s \right) \right) \dd s \right) \exp\left( N_t - \frac{1}{2} \left[ N \right]_t \right) \\
&\leq \norm{X_0}_F^2 \exp\left( \int_0^t \abs{ \tr \left( 2A_s + Q_s \right)} \dd s \right) \EE( N )_t \\
&\leq \norm{X_0}_F^2 \exp\left(d Q_{\infty} T \right) \exp\left( 2 \int_0^t \abs{ \tr \left( A_s \right)} \dd s \right) \EE( N )_t, 
\end{align*}
where $\EE(N)_t \coloneqq \exp\left( N_t - \frac{1}{2} \left[ N \right]_t \right)$ denotes the stochastic exponential of $N$. The first inequality follows from $\tr\left(AB\right) \leq \abs{\tr(A)} \abs{\tr(B)}$. Therefore, for $p>1$, by Cauchy-Schwarz,
\begin{equation} \label{eq:sup-x-p} 
    \E\left[\sup_{t\in\left[0,T\right]} \norm{X_t}_F^p \right] \leq \norm{X_0}_F^p \exp\left(\frac{p}{2} d Q_{\infty} T \right) \E\left[\exp\left( 2p \int_0^T \abs{ \tr \left( A_s\right)} \dd s \right) \right]^{1/2} \E\left[ \sup_{t\in\left[0,T\right]} \abs{ \EE(N)_t}^{p} \right]^{1/2}. 
\end{equation}
Now, as $\E\left[ \exp\left( 1/2 \left[N\right]_T \right) \right] \leq \exp\left( 2d^2 Q_{\infty} T \right) < \infty$, Novikov condition implies that $\left(\EE(N)_t\right)_{t\in\left[0,T\right]}$ is a (continuous) martingale. Therefore, by Doob's inequality, we get
\begin{equation*}
\E\left[ \sup_{t\in\left[0,T\right]} \abs{ \EE(N)_t}^{p} \right] \leq \left(\frac{p}{p-1}\right)^p \E\left[ \abs{ \EE(N)_T}^{p} \right]
\end{equation*}
Finally, we use the definition of the stochastic exponential and Cauchy-Schwarz to obtain 
\begin{align*}
\E\left[ \abs{ \EE(N)_T}^{p} \right] &= \E\left[ \exp\left( pN_T - \frac{p}{2} \left[N\right]_T \right) \right] \\
&\leq \E\left[ \exp\left( pN_T - p^2 \left[N\right]_T \right) \exp\left( \frac{2p^2-p}{2} \left[N\right]_T \right) \right] \\
&\leq \E\left[ \EE(2pN)_T \right]^{1/2} \E\left[\exp\left( (2p^2-p) \left[N\right]_T \right) \right]^{1/2} \\
&\leq \exp\left( 2(2p^2-p)d^2 Q_{\infty}  T \right).
\end{align*}
We plug this last inequality into \eqref{eq:sup-x-p} to conclude the proof, with 
\begin{equation*}
C(p, d, Q_{\infty}, X_0, T) \coloneqq \left(\frac{p \norm{X_0}_F^2}{p-1}\right)^{p/2}  \exp\left( 2p^2 d^2 Q_{\infty} T \right).
\end{equation*}

\end{proof}

\end{document}